\documentclass{article}
\usepackage{mylatexstyle}
\usepackage{amsmath,amssymb,amsthm,mathtools}
\usepackage{bm}
\usepackage{booktabs,array,tabularx}
\usepackage{graphicx}
\graphicspath{{./}}
\usepackage[caption=false,font=normalsize]{subfig}
\usepackage{placeins}
\usepackage{wrapfig}

\usepackage[colorlinks=true,linkcolor=red,citecolor=blue,urlcolor=green]{hyperref}

\usepackage{xurl}
\usepackage[nameinlink,capitalize,noabbrev]{cleveref}
\crefname{equation}{Eq.}{Eqs.}
\Crefname{equation}{Eq.}{Eqs.}

\title{Muon Sublates the Edge of Stability in LLM Pretraining: What Persists and What Changes}

\author{
  Yanzhe Chen\thanks{School of Mathematical Sciences,
  Shanghai Jiao Tong University, China.
  Email: \texttt{cyzebra\_z@sjtu.edu.cn}}
  ~~~
  Qifang Zhao\thanks{Alibaba Inc., Hangzhou, China.
  Emails: \texttt{james.zqf@alibaba-inc.com} (Qifang Zhao);
  \texttt{xiaoxiao.xuxx@alibaba-inc.com} (Xiaoxiao Xu).}
  ~~~
  Xiaoxiao Xu\footnotemark[2]
  ~~~
  Fanghui Liu\thanks{School of Mathematical Sciences,
  Institute of Natural Sciences and MOE-LSC,
  Shanghai Jiao Tong University, China.
  Email: \texttt{fanghui.liu@sjtu.edu.cn} (Corresponding author)}
}
\date{}

\newtheorem{theorem}{Theorem}[section]
\newtheorem{proposition}[theorem]{Proposition}

\newtheorem{corollary}[theorem]{Corollary}
\theoremstyle{definition}
\newtheorem{definition}{Definition}[section]

\theoremstyle{remark}

\newcommand{\R}{\mathbb{R}}
\newcommand{\E}{\mathbb{E}}
\newcommand{\bW}{\bm{W}}
\newcommand{\bG}{\bm{G}}
\newcommand{\bP}{\bm{P}}
\newcommand{\bU}{\bm{U}}
\newcommand{\bH}{\bm{H}}

\newcommand{\bZ}{\bm{Z}}
\newcommand{\bDelta}{\bm{\Delta}}
\newcommand{\Pol}{\operatorname{Pol}}

\newcommand{\rank}{\operatorname{rank}}
\newcommand{\opnorm}[1]{\left\lVert #1\right\rVert_{\mathrm{op}}}
\newcommand{\nucnorm}[1]{\left\lVert #1\right\rVert_{*}}
\newcommand{\normF}[1]{\left\lVert #1\right\rVert_{\mathrm{F}}}
\newcommand{\ipF}[2]{\left\langle #1,#2\right\rangle_{\mathrm{F}}}
\newcommand{\dd}{\,\mathrm{d}}

\definecolor{fhcolor}{rgb}{0.523, 0.235, 0.625}

\begin{document}
\maketitle
\begin{abstract}

Muon is increasingly used for language-model pretraining, yet its large-step dynamics are not captured by the classical edge-of-stability (EoS) picture of gradient descent (GD). We study stochastic (no-momentum) Muon through theoretical analysis, controlled networks, and language models from 22M to 1B parameters, distinguishing three diagnostics: generalized sharpness, loss balance, and update alignment. We theoretically and empirically demonstrate that loss balance differs from its generalized sharpness and cannot determine update alignment across various learning rates in 130M pre-training as well as 1B runs. Besides, our experiments also show that increasing the batch size encourages update alignment towards to total reversal. Furthermore, checkpoint interventions in the 22M and 130M models show that increasing the learning rate can initially raise the loss but a lower validation loss at the end of the prescribed warmup–stable–decay continuation. These results support a split EoS picture: generalized-sharpness saturation persists in the measured runs, but loss balance and update geometry require separate characterization, allowing for learning-rate schedules design. The source code for reproducing the experiments can be found in \url{https://github.com/cyzebra/Muon-Sublates-the-Edge-of-Stability-in-LLM-Pretraining}.  

\end{abstract}

\section{Introduction}
\label{sec:introduction}

Muon is increasingly used for large language model (LLM) pretraining, where it replaces each matrix gradient by an approximately semi-orthogonal polar
direction~\citep{jordan2024muon,liu2025muonscalablellmtraining}. This
transformation changes the geometry of the update, not merely its scale.
Consequently, Muon's large-step dynamics cannot be inferred directly from the
familiar gradient-descent picture.

For gradient descent (GD) with learning rate $\eta$, large learning-rate neural-network training is commonly described through the \emph{edge of stability} (EoS): Hessian sharpness (also called curvature) approaches a learning-rate-dependent stability boundary $2/\eta$, loss becomes non-monotone over short horizons, but training progresses over longer
horizons~\citep{NEURIPS2018_6651526b,Jastrzebski2020The,
lewkowycz2020catapult,cohen2021gradient,damian2023self}.
The distinction between short-term loss fluctuations and longer-term progress is central to this phenomenon. 
To be specific, for {a scalar quadratic model, several quantities, including one-step loss balance (zero loss change), equal-magnitude update reversal (i.e., an exact two-step return), and marginal linear stability (i.e., a multiplier of unit magnitude) coincide with each other~\citep{litman2026originedgestability}.
Table~\ref{tab:gd_muon_signatures} summarizes this scalar GD picture and
contrasts it with no-momentum Muon.}

\citet{islamov2026noneuclidean} extend empirical study of EoS to non-Euclidean gradient methods. For full-batch normalized spectral GD, the idealized no-momentum of deterministic Muon, they empirically observe normalized generalized sharpness near $2/\eta$ in small-scale models. This provides a geometry-aware diagnostic, but does not establish whether the same behavior occurs in stochastic LLM pretraining. For example, loss balance depends on effective curvature along the current polar direction, whereas update alignment/reversal depends on how polar directions change. Stochasticity introduces a further distinction: an unbiased minibatch gradient need not induce an unbiased polar direction. LLM pretraining additionally uses approximate polar updates and an auxiliary optimizer.
To our knowledge, there is no study to check whether LLM pre-training via Muon occur at EoS or not, leading to the following question:
\begin{center}
\emph{Which signatures of EoS persist or change under Muon in LLM pretraining?}
\end{center}

We investigate this question for stochastic no-momentum Muon, isolating the role of its polar transformation in theoretical study; and conduct comprehensive experiments from controlled small-scale models (MLP, CNN, Transformers) to language models (at the scale of 22M, 130M, and 1B parameters). 
We find that, Muon neither simply preserves nor eliminates the classical EoS, but \emph{\bf sublates}\footnote{The term ``sublate'' is borrowed from Hegelian philosophy, translating the German \emph{aufheben}: to negate a prior form while preserving its essential content in a new, more encompassing form.} EoS: It preserves a loss-balance edge, but it breaks the classical coupling between loss balance and update reversal in \cref{tab:gd_muon_signatures}. We call this a \emph{split edge of stability}: effective curvature separates from normalized generalized sharpness, while loss balance does not uniquely determine update alignment, according to our 130M runs in \cref{fig:introduction-muon}.


\begin{figure}[!t]
\centering
\subfloat[{Loss ($\eta=0.02$)}\label{fig:introduction_muon_130M:b}]{\includegraphics[width=.325\linewidth]{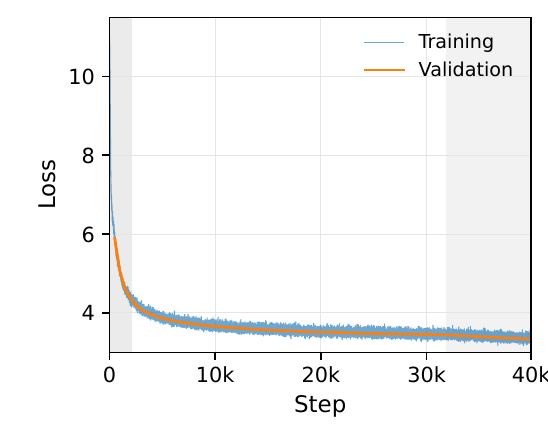}}
\hfill
\subfloat[{Curvature and boundary}\label{fig:introduction_muon_130M:d}]{\includegraphics[width=.325\linewidth]{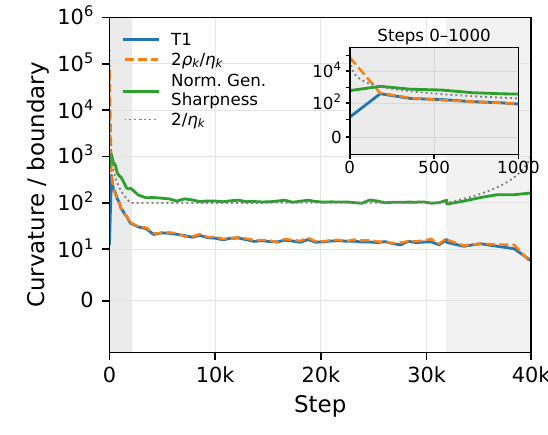}}
\hfill
\subfloat[{Update alignment (T2)}\label{fig:introduction_muon_130M:c}]{\includegraphics[width=.325\linewidth]{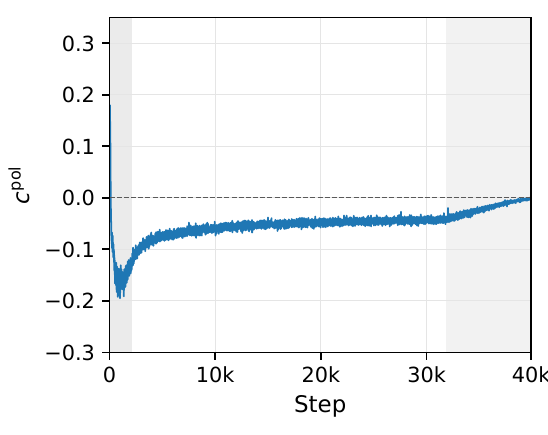}}
\par\vspace{6pt}
\subfloat[{Loss ($\eta=0.12$)}\label{fig:introduction-012-loss}]{\includegraphics[width=.325\linewidth]{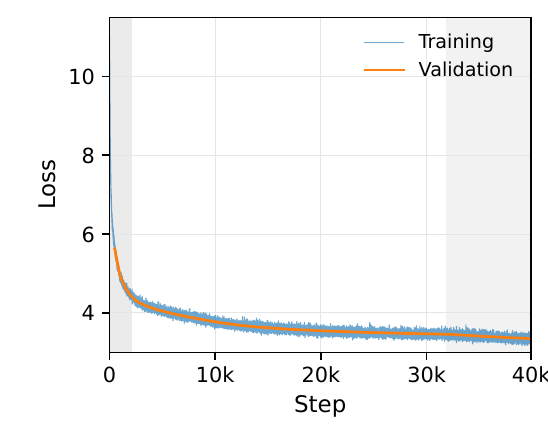}}
\hfill
\subfloat[{Curvature and boundary}\label{fig:llm-lr-T1-012}]{\includegraphics[width=.325\linewidth]{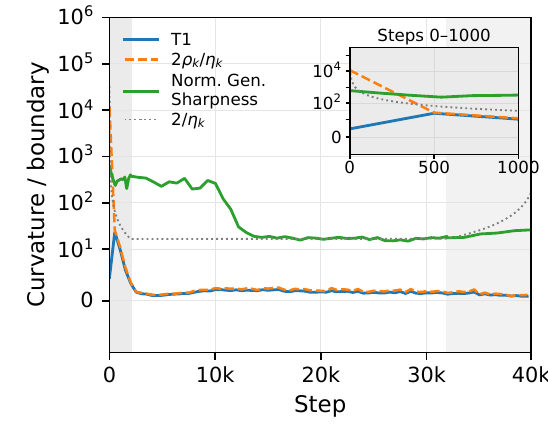}}
\hfill
\subfloat[{Update alignment (T2)}\label{fig:llm-lr-cosine-012}]{\includegraphics[width=.325\linewidth]{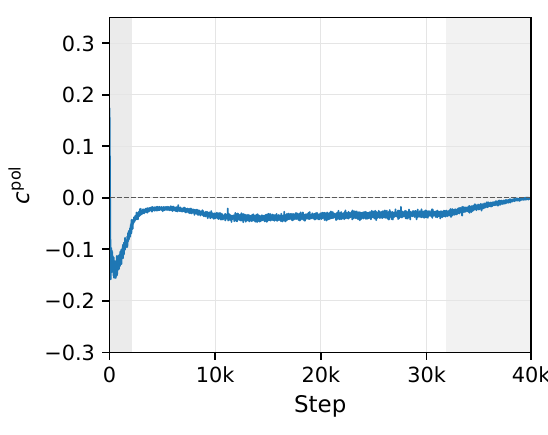}}
\caption{{\textbf{Loss, curvature, and update alignment over 130M runs.} Rows use peak learning rates $0.02$ and $0.12$, with NS-5 Muon, batch 32, and WSD. Columns show loss, curvature diagnostics, and global update alignment, respectively. Middle column compares effective curvature $s^{\mathrm M}_{k,b}$ (our T1 diagnostic) and its corresponding threshold $2\rho_{k,b}/\eta_k$, and normalized generalized sharpness (short for {\tt Norm.~Gen.~Sharpness} in \textcolor[RGB]{44,160,44}{green} color), and its corresponding threshold $2/\eta_k$; its inset enlarges steps 0--1,000. Curvature axes use a symmetric-log scale; grey shading marks warmup and decay. Results on 1B runs are deferred to \cref{subsec:llm-1b}.}}
\label{fig:introduction-muon}
\vspace{4pt} 
\end{figure}

\begin{table}[!t]
\caption{{Classical scalar-quadratic GD signatures versus no-momentum Muon.
The GD column assumes $L(x)=\lambda x^2/2$, $\lambda>0$;
the Muon column concerns general full-batch exact-polar updates ($\bP_k=\Pol(\bG_k)$). Thresholds for Muon are defined in \cref{eq:intro-diagnostics}. For multidimensional GD, see Table~\ref{tab:gd-general-signatures} in Appendix~\ref{app:gd-general-signatures}.}}
\label{tab:gd_muon_signatures}
\small
\setlength{\tabcolsep}{4pt}
\renewcommand{\arraystretch}{1.25}
\begin{tabularx}{\linewidth}{@{}>{\raggedright\arraybackslash}p{0.32\linewidth}>{\centering\arraybackslash}p{0.19\linewidth}>{\centering\arraybackslash}X@{}}
\toprule
Signature & Scalar quadratic GD & No-momentum Muon \\
\midrule
One-step loss balance (T1) & $\eta\lambda=2$ & $s_k^{\mathrm M}=2/\eta$ \\
Equal-magnitude reversal / exact two-step return ($c^{\mathrm{pol}}=-1$) & $\eta\lambda=2$ & $\bP_{k+1}=-\bP_k$ \\
Marginal linear stability & $\eta\lambda=2$ & No general equivalence with T1 or negative alignment \\
\bottomrule
\end{tabularx}
\end{table}

\subsection{Contributions and findings}

We distinguish three diagnostics. For a single matrix parameter $\bm W_k$ at time-step $k$, apart from classical normalized generalized sharpness $S_k$ \citep{islamov2026noneuclidean}, we also define the following two diagnostics motivated by \cref{tab:gd_muon_signatures}
\begin{equation}
\underbrace{s_{k,b}^{\mathrm M}=\frac{2\rho_{k,b}}{\eta_k}}
  _{\substack{\text{T1:}\\{\text{loss balance}}}}
\qquad
\underbrace{c_{k,b}^{\mathrm{pol}}=
\frac{\langle\bP_k,\bP_{k+1}\rangle_{\mathrm F}}
{\|\bP_k\|_{\mathrm F}\|\bP_{k+1}\|_{\mathrm F}}}
  _{\substack{\text{T2:}\\{\text{update alignment}}}}
\qquad
{\underbrace{S_k=
\frac{\displaystyle\max_{\|\bm D\|_{\mathrm{op}}=1}
\langle \bm D,\nabla^2L(\bm W_k)[\bm D]\rangle_{\mathrm F}}
{\|\nabla L(\bm W_k)\|_{*}}}
_{\substack{\text{normalized}\\\text{generalized sharpness}}}}\,,
\label{eq:intro-diagnostics}
\end{equation}
{where $s_{k,b}^{\mathrm M}$ (denoted by T1) specifies loss balance related to the retained first-order descent $\rho_{k,b}$ at time-step $k$ and batch size $b$; $c_{k,b}^{\mathrm{pol}}$ measures consecutive update alignment, with $c_{k,b}^{\mathrm{pol}}=0$ (denoted by T2) separating its signs. $S_k$ is the normalized generalized sharpness uses the block norm $\| \cdot \|_{\mathrm{op}}$ and dual norm $\| \cdot \|_{*}$.} \cref{tab:gd_muon_signatures} provides a scalar-quadratic reference; the coincidences in that special case are not general equivalences along high-dimensional training trajectories. Our main contribution and findings are summarized as below:


\begin{itemize}
\item {\textbf{Loss balance edge (T1) differs from its generalized sharpness.}
We derive an exact conditional loss identity with boundary $2\rho_{k,b}/\eta_k$ and prove that effective curvature cannot remain a fixed distance below $2\rho_{k,b}/\eta$ (Proposition~\ref{prop:curvature-visits}).
In 130M pretraining, we find such separation on these two thresholds. Besides, T1 approaches its boundary before normalized generalized sharpness approaches $2/\eta_k$, see \cref{fig:introduction-muon,fig:llm-lr}) for details.}
\item {\textbf{Loss curvature (T1) does not determine update alignment (T2).}}
We show this, theoretically and empirically. Our quadratic model theoretically implies that effective curvature does not determine update alignment (Proposition~\ref{prop:curvature-alignment}).
Controlled experiments show T1 and T2 crossings in either order (\cref{fig:small-network-onsets}); checkpoint learning-rate interventions change the sign of the conditional loss increment while update alignment remains negative (\cref{fig:perturb22m-lr}).
Across 130M and 1B pretraining, loss balance coexists with update alignment but far from exact reversal with $c^{\mathrm{pol}}=-1$ (\cref{fig:introduction-muon,fig:llm-1b}). But increasing the batch size encourages $c^{\mathrm{pol}}$ towards to $-1$.
\end{itemize}

\noindent {\bf Implications for theory and practice.}
Together, these results demonstrate a separation of EoS for Muon in terms of loss balance, generalized sharpness, and update alignment. This separation would identify concrete theory questions on how curvature and update geometry evolve together. Note that, our paper is not to explain the theoretic mechanism of EoS for Muon but to empirically characterize the signatures of EoS under LLM pre-training, which would inspire several practical use. For example, our results may inform future design of learning-rate in LLM pre-training. Our experiments in \cref{fig:perturb22m-lr} have indicated that the T1 diagnostic before loss balance leaves sufficient room for increasing learning rates, which is possible to improve the WSD learning-rate schedule.

\subsection{Related Work}
\label{sec:related}

\noindent {\bf Origins and mechanisms of EoS.}
Early work connected large learning rates to dynamical stability, sharp directions, break-even points, and catapult dynamics~\citep{NEURIPS2018_6651526b,jast2019relation,Jastrzebski2020The,lewkowycz2020catapult}. \citet{cohen2021gradient} documented full-batch GD sharpening toward $2/\eta$. Subsequent explanations include unstable convergence~\citep{pmlr-v162-ahn22a}, progressive sharpening~\citep{NEURIPS2022_40bb79c0}, self-stabilization~\citep{damian2023self}, and two-step dynamics and bifurcation~\citep{pmlr-v202-chen23b,NEURIPS2023_e2a9256b}. \citet{litman2026originedgestability} derives an edge coupling underlying curvature visits. 

\noindent {\bf Stochastic, adaptive, and LLM edges.}
Batch-aware curvature diagnostics extend EoS to SGD~\citep{lee2023a,andreyev2024edge}; adaptive methods instead involve a preconditioned Hessian and optimizer state~\citep{cohen2022adaptive}. In LLM pretraining, \citet{cai2026does} distinguishes an edge of convexity from EoS, reporting EoS signatures across SGD and Adam settings as prevalent but not universal. \citet{kalra2026scalable} reports progressive sharpening and EoS-like behavior in AdamW language-model training up to 7B parameters. These results motivate studying stochastic Muon, where nonlinear polar normalization changes batch-to-population alignment.

\noindent {\bf Muon and matrix-aware optimization.}
Muon combines momentum with approximate matrix orthogonalization~\citep{jordan2024muon}; its no-momentum core has a spectral-norm steepest-descent interpretation~\citep{bernstein2024neuripsw-old}. Large-scale implementations add scaling, weight decay, and distributed execution~\citep{liu2025muonscalablellmtraining}; theory also relates Muon with decoupled weight decay to spectral-norm constraints~\citep{chen2026tmlr-muon}. \citet{wang2026muon} study Muon's advantage over Adam through directional curvature, decomposing its curvature penalty into within-layer and cross-layer contributions. 

\noindent {\bf Non-Euclidean and Muon EoS.}
Directional smoothness measures curvature along an optimization chord~\citep{NEURIPS2024_1ac83203}. \citet{islamov2026noneuclidean} extends EoS diagnostics to arbitrary norms, including Spectral GD and normalized Spectral GD. Its full-batch experiments include normalized Spectral GD on a CNN and unnormalized Spectral GD on a Transformer; but stochastic extensions remain open. Instead, we derive a coherence-corrected stochastic loss boundary, prove its separation from {update-alignment} sign, and examine them in LLM pre-training.

\section{Analysis of Stochastic Muon with T1 and T2}
\label{sec:stochastic}

In this section, we present our analysis on stochastic Muon on loss balance and update alignment, {using the {loss-balance boundary} T1 and update-alignment diagnostic T2,} and show their separation in controlled networks and a matrix quadratic setting. 
Before introducing our results, we give a brief introduction of Muon as below for self-completeness.

\noindent {\bf Brief introduction of Muon.} Let $L:\R^{m\times n}\to\R$ be a continuously differentiable population or fixed full-data objective and $L_B$ a batch loss. For a fresh batch $B_k\sim\mathcal D_b$ of size $b$, 
assume $L=\E_bL_B$, $\nabla L=\E_b\nabla L_B$, with fixed $\eta>0$, the gradient and polar map are given by
\begin{equation}
\bG_k=\nabla L(\bW_k),\quad
\widehat{\bG}_k=\nabla L_{B_k}(\bW_k),\quad
\bP_k=\Pol(\widehat{\bG}_k),\quad
\bW_{k+1}=\bW_k-\eta\bP_k\,,
\label{eq:stoch-update}
\end{equation}
where we use a compact SVD $\bG=\bm U_r\bm\Sigma_r\bm V_r^\top$ containing positive singular values, leading to $\Pol(\bG)=\bm U_r\bm V_r^\top$ and $\Pol(\bm0)=\bm0$. Write $\ipF{\cdot}{\cdot}$ for the Frobenius inner product and $\normF{\cdot}$, $\opnorm{\cdot}$, $\nucnorm{\cdot}$ for the Frobenius, operator, and nuclear norms. Active directions have $\opnorm{\bP_k}=1$, and operator--nuclear duality gives
\begin{equation}
\max_{\opnorm{\bZ}\leq1}\ipF{\bG}{\bZ}
=\ipF{\bG}{\Pol(\bG)}=\nucnorm{\bG}.
\label{eq:operator-nuclear-duality}
\end{equation}
Thus $\eta\nucnorm{\bG_k}$ is the available first-order descent scale; unbiased batch gradients need not yield unbiased polar directions.

\subsection{T1: stochastic loss balance}\label{subsec:t1}

{To identify a {loss-balance edge} for Muon, we ask when its expected one-step loss change is zero from the spirit of classical EoS. This change has two parts: first-order descent and a finite-step remainder. Minibatch sampling changes the alignment of the polar direction with the population gradient. We use $\rho_b$ to measure the retained descent and $s_b^{\mathrm M}$ to measure curvature along the update segment. Their balance defines T1.}

\begin{definition}[Coherence and effective curvature]
\label{def:batch-coherence}
For the update in \cref{eq:stoch-update} with $\bG_k\neq\bm{0}$, define
\begin{equation}
\rho_b(\bW_k):=\frac{\E_b\ipF{\bG_k}{\bP_k}}{\nucnorm{\bG_k}}\in[-1,1]\,.
\label{eq:rho}
\end{equation}

The effective curvature is the finite-step loss remainder below. If $L$ is also $C^2$ near $\bW_k$, it has the stated small-step limit, with $\bH_k=\nabla^2L(\bW_k)$:
\begin{equation}
\begin{aligned}
s_b^{\mathrm M}(\bW_k;\eta)
&:=\frac{2\E_b\!\left[L(\bW_k-\eta\bP_k)-L(\bW_k)+\eta\ipF{\bG_k}{\bP_k}\right]}{\eta^2\nucnorm{\bG_k}} \xrightarrow{\eta\to0}
\frac{\E_b\ipF{\bP_k}{\bH_k[\bP_k]}}{\nucnorm{\bG_k}}\,.
\end{aligned}
\label{eq:stoch-curvature}
\end{equation}
\end{definition}
It is the expected directional Hessian curvature, normalized by $\nucnorm{\bG_k}$. The finite-step definition requires only the original $C^1$ assumption and finite expectations. Coherence measures the retained first-order descent. Operator--nuclear duality in \cref{eq:operator-nuclear-duality} bounds $\rho_b$ by $[-1,1]$. Together, $\rho_b$ and $s_b^{\mathrm M}$ give the exact one-step loss change. The proofs are deferred to Appendix~\ref{app:det-details}.

\begin{figure}[!t]
\centering
\subfloat[\small MLP: loss balance\label{fig:controlled-mlp-t1}]{\includegraphics[width=.245\linewidth]{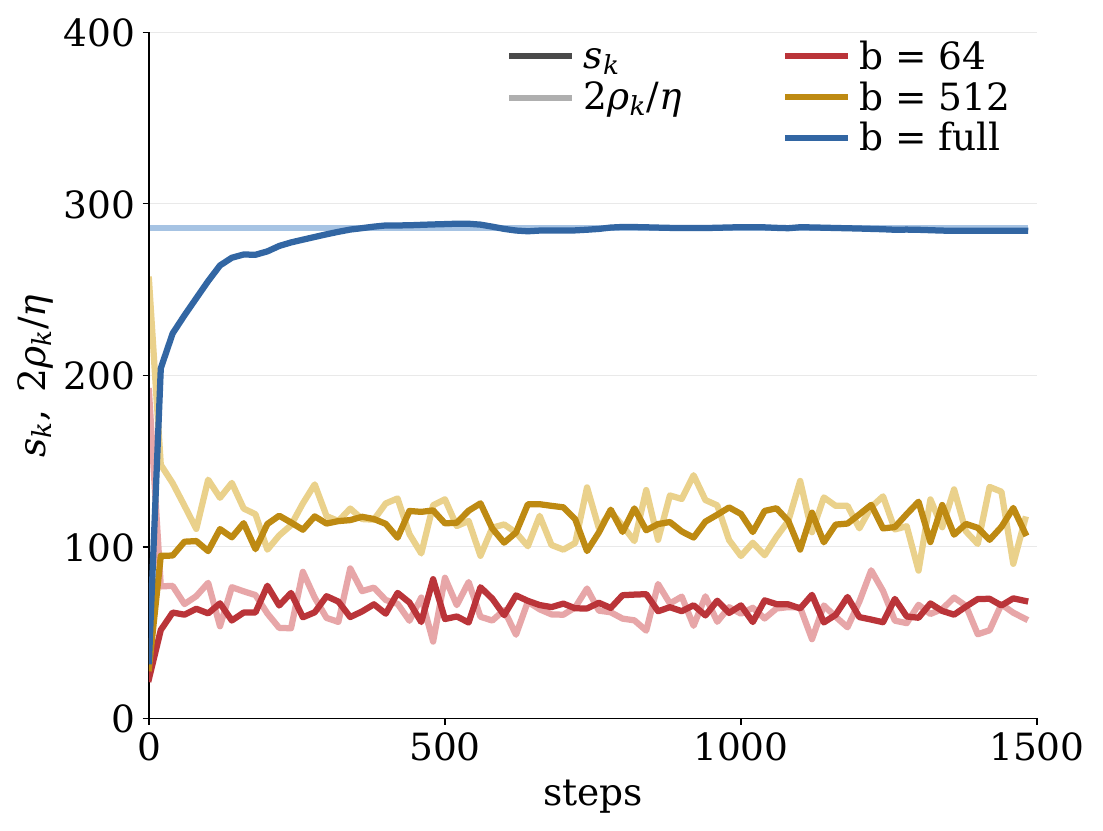}}
\hfill
\subfloat[\small CNN: loss balance\label{fig:controlled-cnn-t1}]{\includegraphics[width=.245\linewidth]{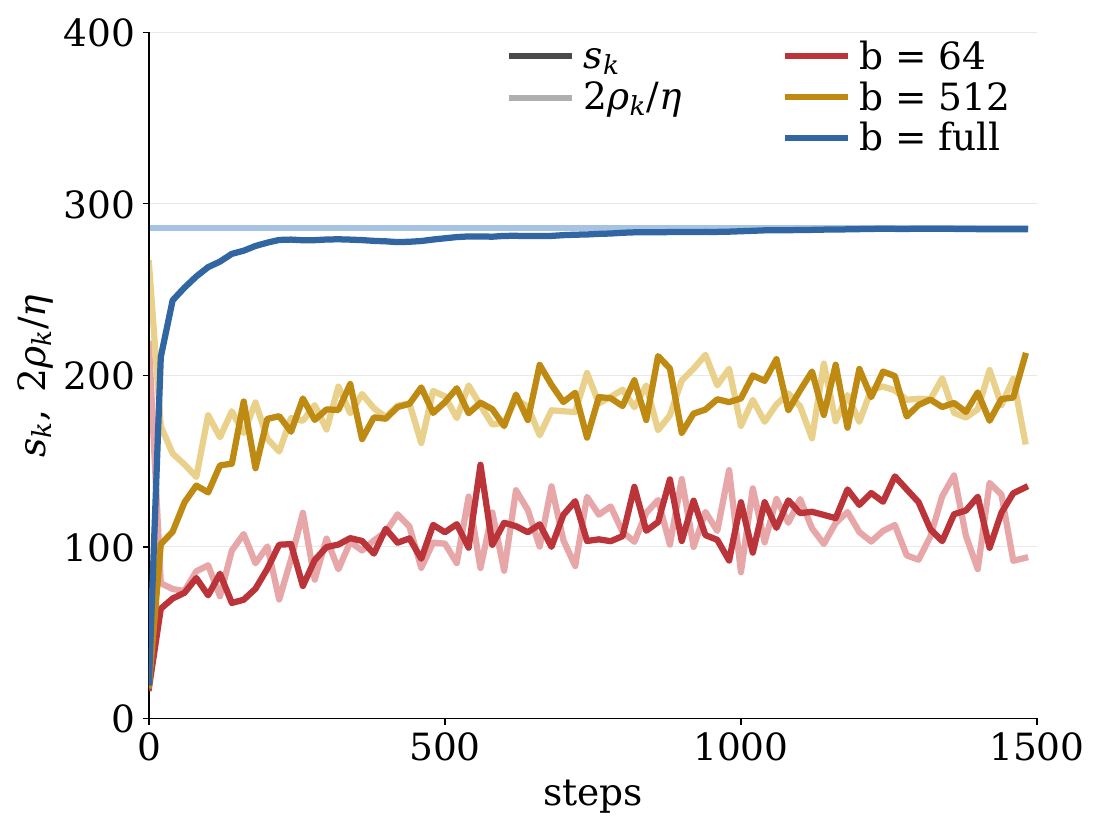}}
\hfill
\subfloat[\small MLP: $c_{k,b}^{\mathrm{pol}}$ directions\label{fig:controlled-mlp-t2}]{\includegraphics[width=.245\linewidth]{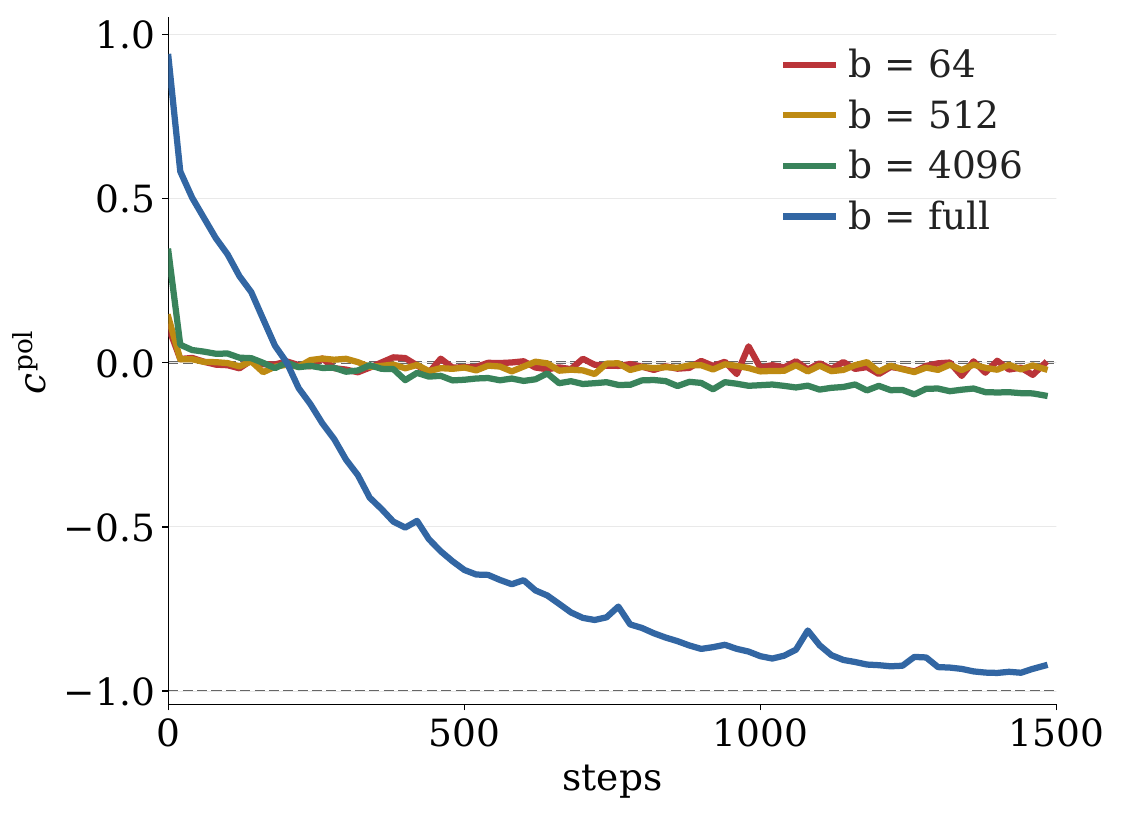}}
\hfill
\subfloat[\small CNN: $c_{k,b}^{\mathrm{pol}}$ directions\label{fig:controlled-cnn-t2}]{\includegraphics[width=.245\linewidth]{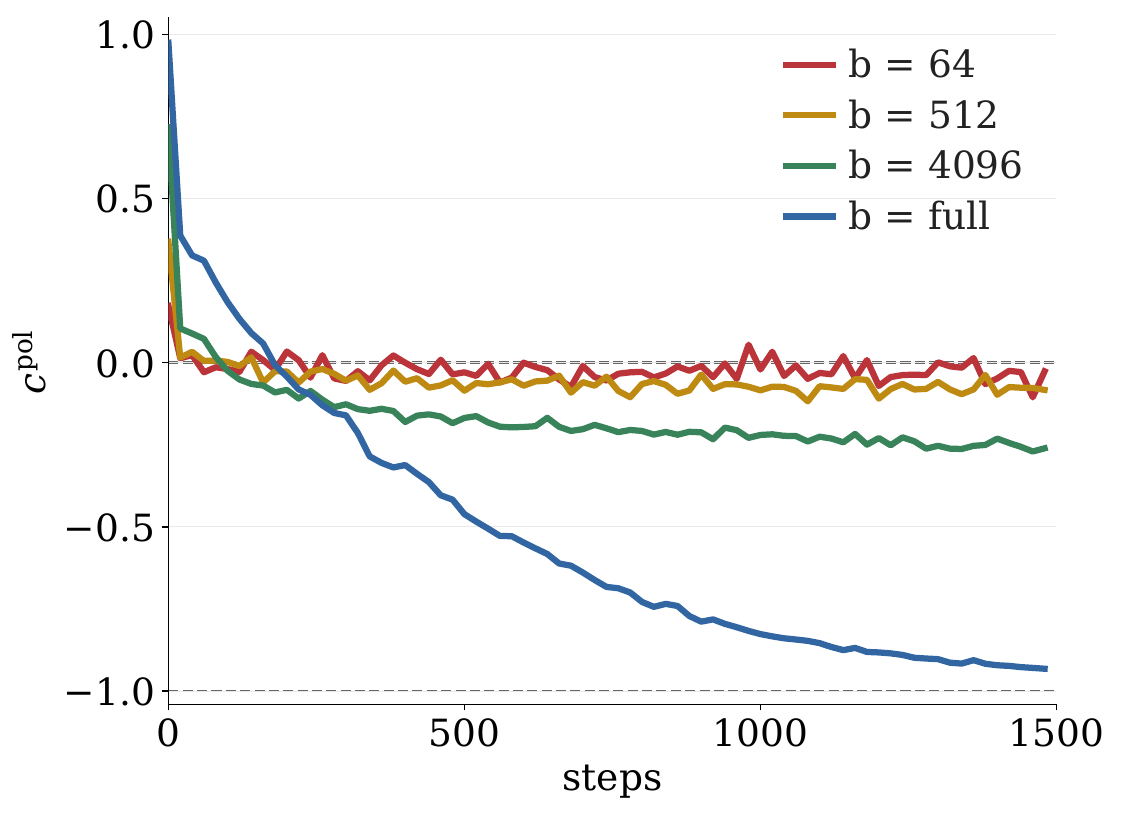}}
\caption{\textbf{Loss balance and update alignment across batch sizes.}
Results are shown for MLP and CNN experiments with SVD polar updates at $\eta=0.007$.
(a) and (b) compare conditional curvature $\hat s_b^{\mathrm M}$ with $2\hat\rho_b/\eta$ on MLP and CNN, respectively;
(c) and (d) show global consecutive-update cosines on MLP and CNN, respectively.
Dark and light curves in (a) and (b) denote curvature and boundary, respectively.}
\label{fig:controlled-diagnostics}
\vspace{3pt} 
\end{figure}

\begin{figure}[!t]
\centering
\subfloat[\small Transformer: loss\label{fig:intro-transformer-loss}]{\includegraphics[width=.325\linewidth]{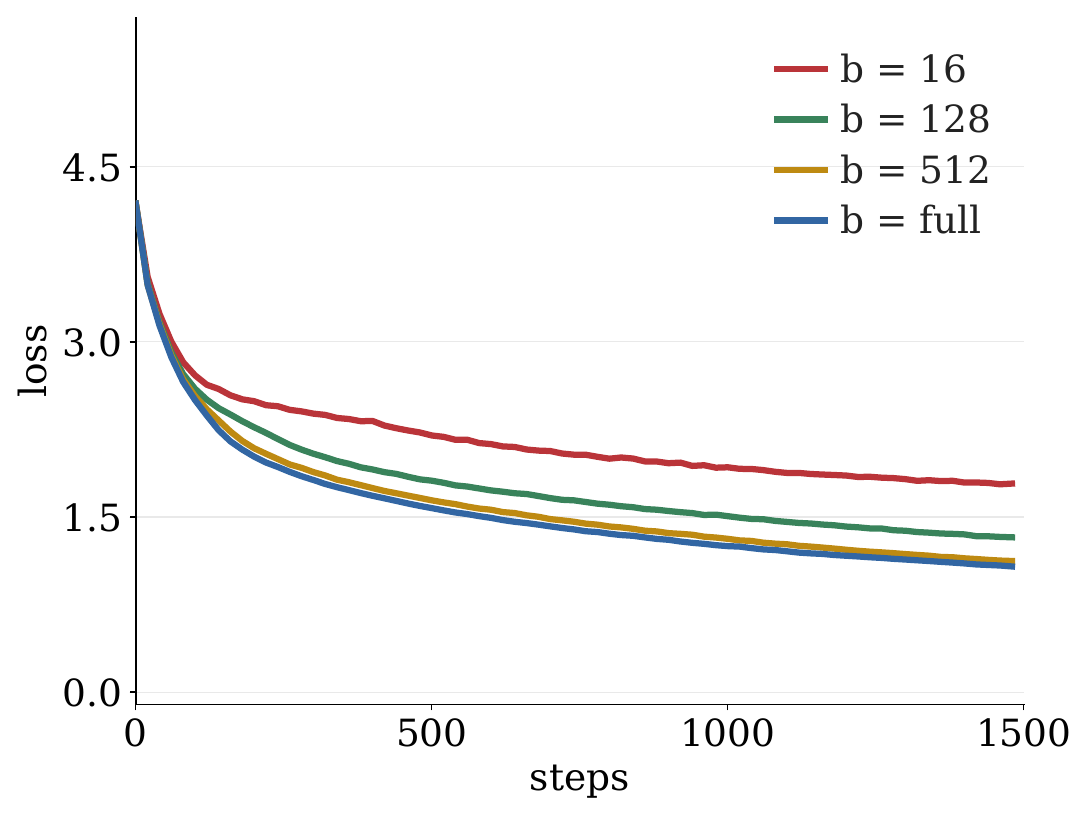}}
\hfill
\subfloat[\small Transformer: loss balance\label{fig:intro-transformer-balance}]{\includegraphics[width=.325\linewidth]{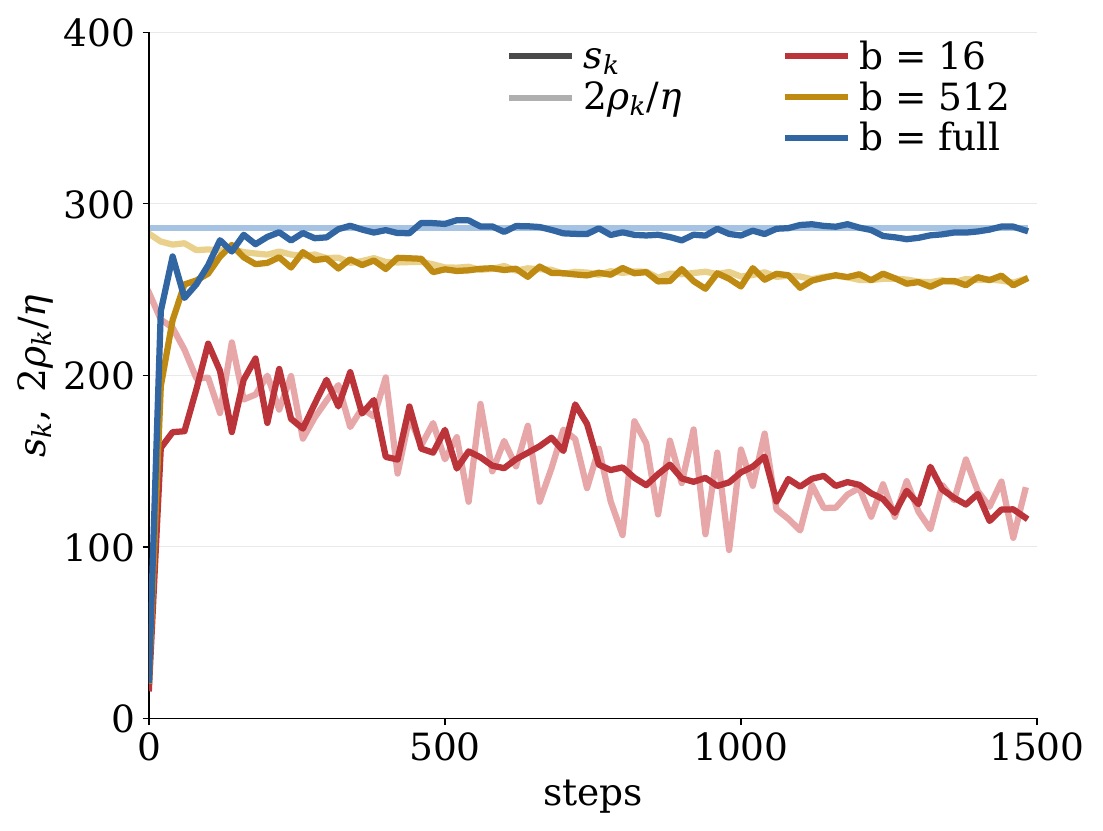}}
\hfill
\subfloat[{\small Transformer: update alignment}\label{fig:introduction_muon_130M:a}]{\includegraphics[width=.325\linewidth]{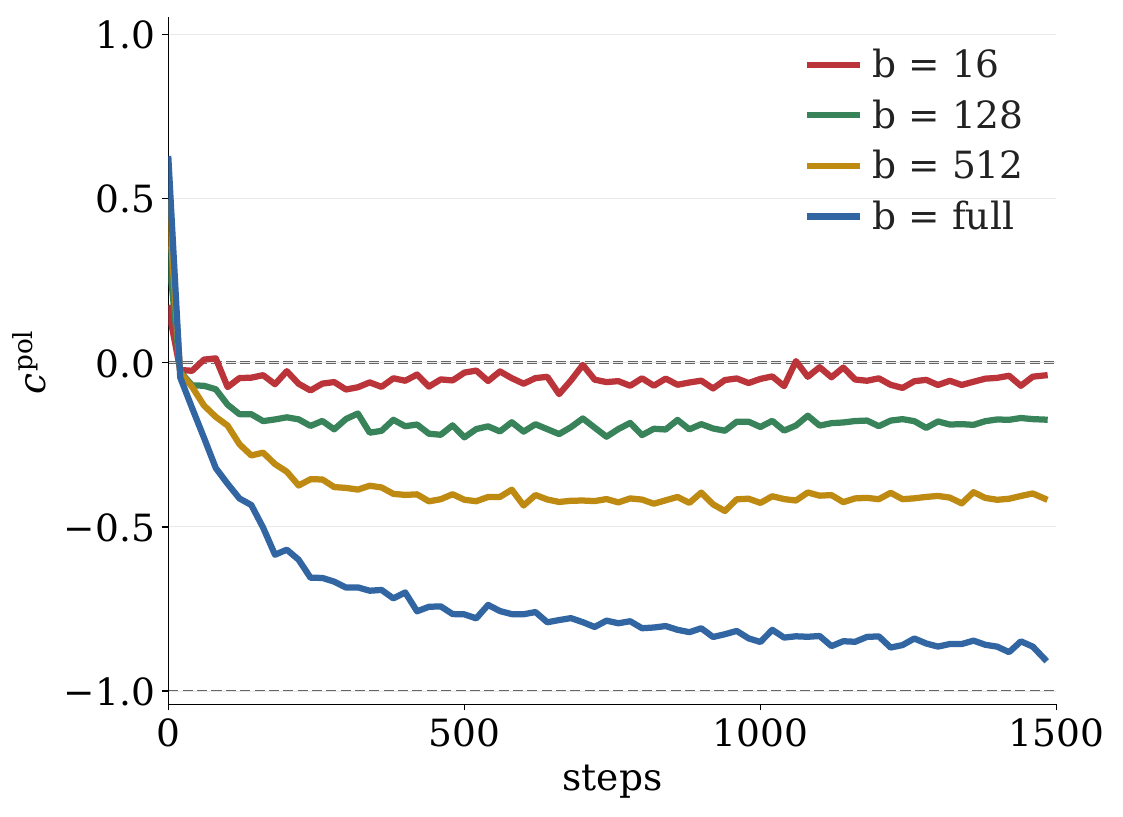}}
\caption{{\textbf{Loss, curvature, and update alignment across batch sizes.} Results are shown on controlled Transformer with SVD-polar Muon at $\eta=0.007$ and varying batches. Columns show loss, curvature diagnostics, and global update alignment.}}
\label{fig:controlled-transformer}
\vspace{3pt} 
\end{figure}

\noindent{\textbf{Conditional loss identity.}}
{For the update in \cref{eq:stoch-update}, the definitions above give, whenever $\bG_k\neq\bm0$,}
\begin{equation}
\E_b\!\left[L(\bW_{k+1})-L(\bW_k)\mid\bW_k\right]
=\frac{\eta^2\nucnorm{\bG_k}}2
\left(s_b^{\mathrm M}(\bW_k)-\frac{2\rho_b(\bW_k)}\eta\right).
\label{eq:stoch-loss}
\end{equation}
The conditional expected loss is nonincreasing if and only if $s_b^{\mathrm M}\leq2\rho_b/\eta$. Its boundary is
\begin{equation}
 s_b^{\mathrm M}=2\rho_b/\eta\,.
\label{eq:moving-edge}
\end{equation}

{At full batch, $\rho_{\mathrm{full}}=1$, so T1 recovers $2/\eta$.
The effective curvature equals the tight operator-norm directional smoothness
of \citet{islamov2026noneuclidean} divided by $\nucnorm{\bG_k}$.
Beyond this one-step identity, the following result constrains effective
curvature over time.}

\begingroup
\begin{proposition}[Effective curvature cannot stay below the boundary]
\label{prop:curvature-visits}
Consider stochastic exact-polar no-momentum Muon in
\cref{eq:stoch-update}, with fixed $\eta>0$, and fresh independent batches of size $b$.
Suppose $L$ is continuously differentiable and bounded below,
$\nucnorm{\bG_k}>0$ for every $k$ almost surely, and
$
\sum_{k=0}^{\infty}\eta^2\nucnorm{\bG_k}=\infty
$ almost surely.
Then
\[
\limsup_{k\to\infty}
\left(s_{k,b}^{\mathrm M}-\frac{2\rho_{k,b}}{\eta}\right)
\geq0
.
\]
Thus, effective curvature cannot eventually remain a fixed distance
below its loss-balance boundary.
\end{proposition}
The proof stops the process when a fixed negative boundary gap is
lost and uses the conditional loss identity and the lower bound on
$L$ to rule out infinite cumulative weight before stopping
(Appendix~\ref{app:det-details}).
\endgroup

\noindent {\bf Controlled-network empirical validation.}
Our experiments in \cref{fig:intro-transformer-balance} on Transformer and 
\cref{fig:controlled-mlp-t1} on MLP and \cref{fig:controlled-cnn-t1} on CNN study how batch size changes both effective curvature and the coherence-dependent loss boundary. One can see that, the curvature approaches and oscillates around this boundary, marking near-zero conditional loss increments and changes in their sign. Larger batch size (ranging from 64 to 512 and full batch) encourages larger $\rho_b$ and induces less fluctuation. Full-batch exact-polar updates make $\rho_{\mathrm{full}}=1$, so the boundary reduces to $2/\eta$. These experiments use numerical SVD-polar updates; more experimental settings about architectures, probes, and normalization are given in Appendix~\ref{app:experiments}.

\subsection{{T2: update alignment}}
\label{subsec:t2}
For scalar quadratic GD, directions reverse when $\eta\lambda>1$; {equal-magnitude update reversal} and loss balance coincide at $\eta\lambda=2$. Muon loss curvature does not determine the alignment of successive polar directions. We therefore measure their cosine as a separate diagnostic.

\begin{definition}[{{Update alignment}}]
\label{def:polar-cos}
For two consecutive nonzero directions, define
\begin{equation}
c_{k,b}^{\mathrm{pol}}:=\frac{\ipF{\bP_k}{\bP_{k+1}}}{\normF{\bP_k}\normF{\bP_{k+1}}}\in[-1,1]\,.
\label{eq:stoch-cos}
\end{equation}
\end{definition}
{T2 measures update alignment; its sign boundary is $c_{k,b}^{\mathrm{pol}}=0$.} It separates acute from obtuse alignment. {Exact direction reversal} $c_{k,b}^{\mathrm{pol}}=-1$ {gives an exact two-step return with equal direction norms:}
\begin{equation}
c_{k,b}^{\mathrm{pol}}=-1
\iff \bP_{k+1}=-\bP_k
\iff \bW_{k+2}=\bW_k.
\label{eq:exact-update-reversal}
\end{equation}
A negative cosine does not imply this two-step return or identify a stability edge by itself. 

\noindent {\bf Controlled-network empirical validation.}
As shown in \cref{fig:introduction_muon_130M:a} on Transformer, \cref{fig:controlled-mlp-t2} on MLP, and \cref{fig:controlled-cnn-t2} on CNN, larger batches give more negative alignment in the displayed runs, while smaller batches remain closer to orthogonality. Small-batch trajectories can remain weakly negative over long intervals. No displayed global trajectory {reaches exact direction reversal}. 

Besides, we also study {the order in which loss balance and the {update-alignment sign} crossing occur}. Appendix~\ref{app:secant-details} makes the distinction precise through secant relations between consecutive updates: For GD, {{update-alignment sign} boundary occurs} at the directional secant response $\overline s_k^{\mathrm{GD}}=1/\eta$; Full-batch exact-polar Muon instead leads to {the {update-alignment sign} condition} $\overline s_k^{\mathrm M}=\vartheta_k/\eta$, with negative alignment when $\overline s_k^{\mathrm M}>\vartheta_k/\eta$. Here $\vartheta_k$ depends on the singular-value geometry of consecutive gradients and need not equal one. Thus, Muon's {{update-alignment sign}} boundary is geometry-dependent rather than a universal $1/\eta$ threshold. \cref{fig:small-network-onsets-panel-a,fig:small-network-onsets-panel-b} empirically shows no consistent ordering: at $b=512$, T2 precedes T1 in the MLP and CNN, whereas at $b=4096$, T1 precedes T2. 

\noindent {\bf Is weak alignment from minibatch sampling or learning dynamics?}
The small-batch trajectories in
\cref{fig:introduction_muon_130M:a,fig:controlled-mlp-t2,fig:controlled-cnn-t2}
and the 130M trajectory in \cref{fig:introduction_muon_130M:c}
have $c_{k,b}^{\mathrm{pol}}$ close to zero.
This observation alone cannot tell us whether the weak alignment comes from Muon's training dynamics or simply from minibatch randomness, since random high-dimensional vectors are nearly orthogonal.

To distinguish this, we compare frozen-state and consecutive-step directions in a Tiny Transformer with no-momentum NS-5 Muon, batch size eight, and two seeds. Frozen-state mean cosines are small and positive, while stable-stage {update-direction cosine means} are negative. This supports minibatch dispersion as a source of weak alignment and state changes as a source of negative bias. \cref{fig:fixedcp-noise} in the appendix reports the results, see Appendix~\ref{app:fixedcp-protocol} for details.

\begin{figure}[!t]
\centering
\subfloat[MLP\label{fig:small-network-onsets-panel-a}]{\includegraphics[width=.325\linewidth]{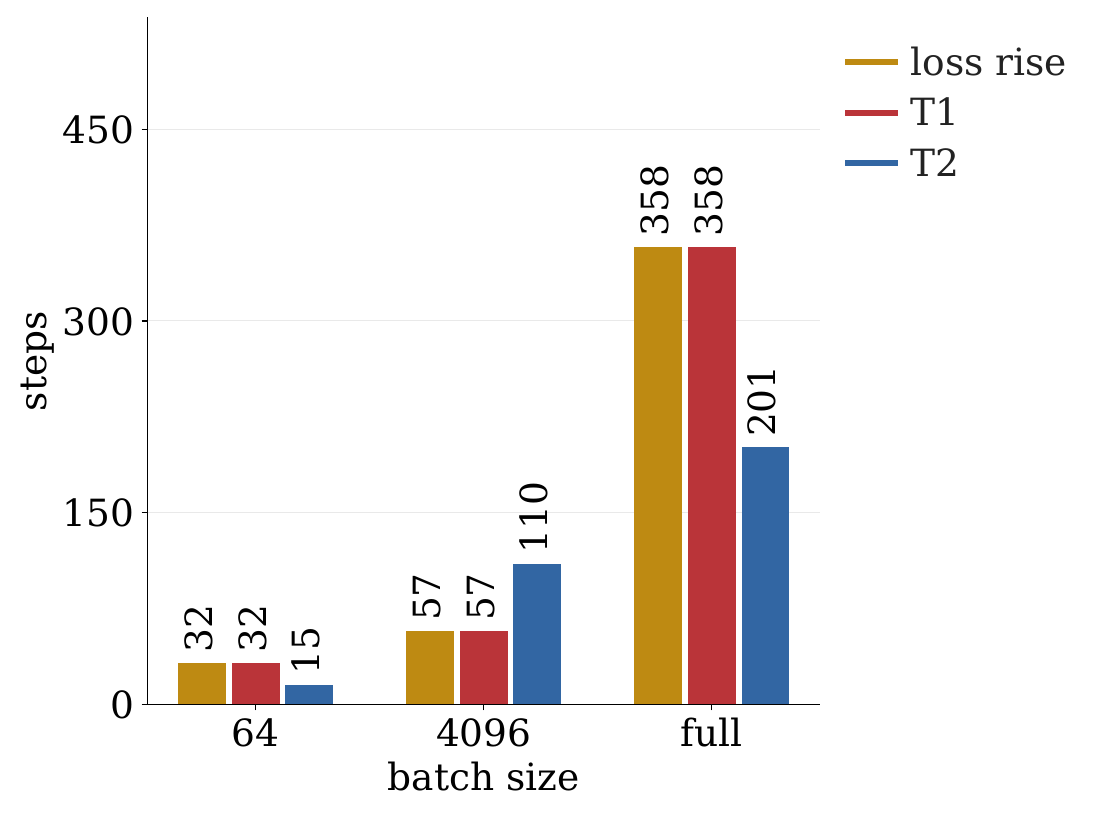}}\hfill
\subfloat[CNN\label{fig:small-network-onsets-panel-b}]{\includegraphics[width=.325\linewidth]{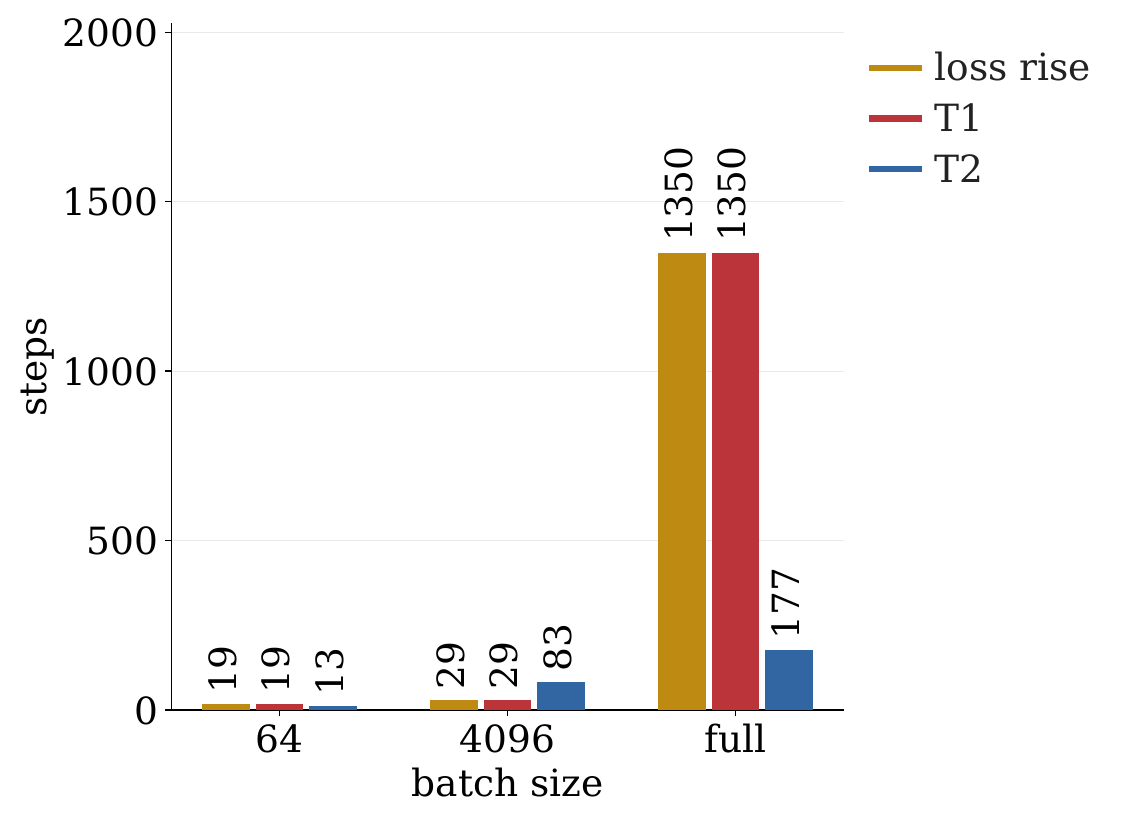}}\hfill
\subfloat[Transformer\label{fig:small-network-onsets-panel-c}]{\includegraphics[width=.325\linewidth]{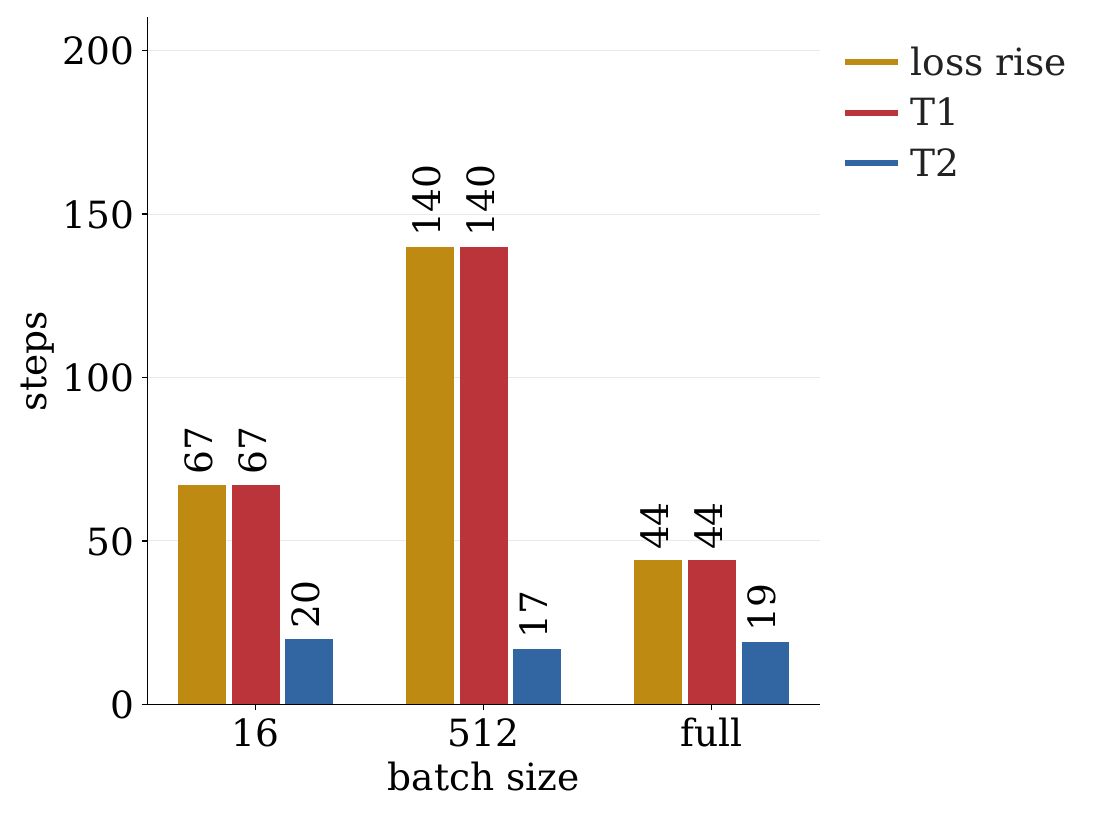}}
\caption{\textbf{Diagnostic events in controlled networks.} These experiments use the MLP, CNN, and Transformer at $\eta=0.007$ with the listed batch sizes; they measure the first loss-rise, T1, and T2 events. Each bar reports the first qualifying step under the recorded tolerances.}
\label{fig:small-network-onsets}
\vspace{3pt} 
\end{figure}

\begin{figure}[!t]
\centering
\begingroup
\subfloat[Curvature, $\eta_{\rm peak}=0.08$\label{fig:llm-lr-T1-008}]{\includegraphics[width=.245\linewidth]{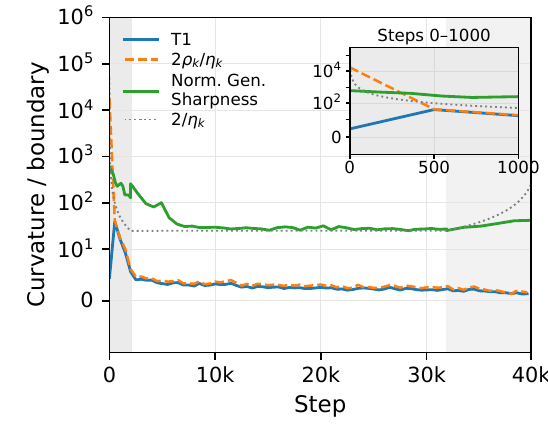}}
\hfill
\subfloat[Update alignment, $\eta_{\rm peak}=0.08$\label{fig:llm-lr-cosine-008}]{\includegraphics[width=.245\linewidth]{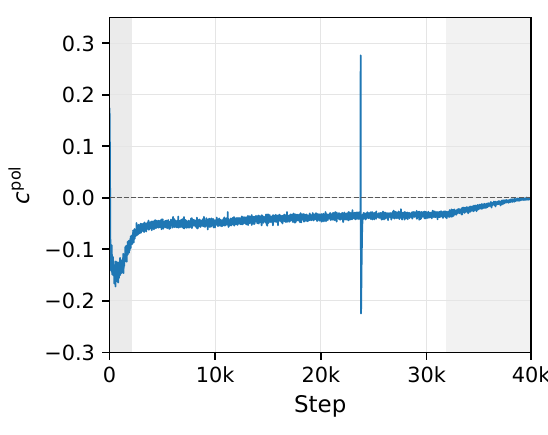}}
\hfill
\subfloat[Curvature, $\eta_{\rm peak}=0.04$\label{fig:llm-lr-T1-004}]{\includegraphics[width=.245\linewidth]{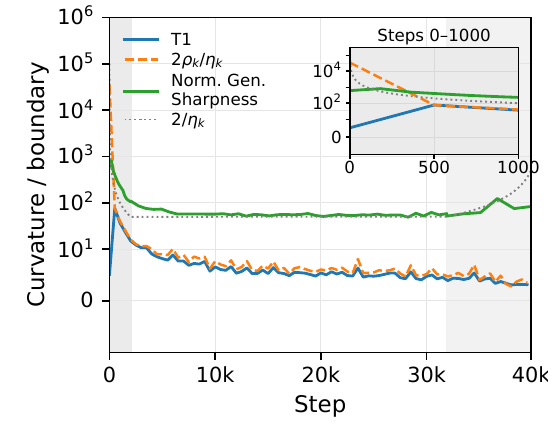}}
\hfill
\subfloat[Update alignment, $\eta_{\rm peak}=0.04$\label{fig:llm-lr-cosine-004}]{\includegraphics[width=.245\linewidth]{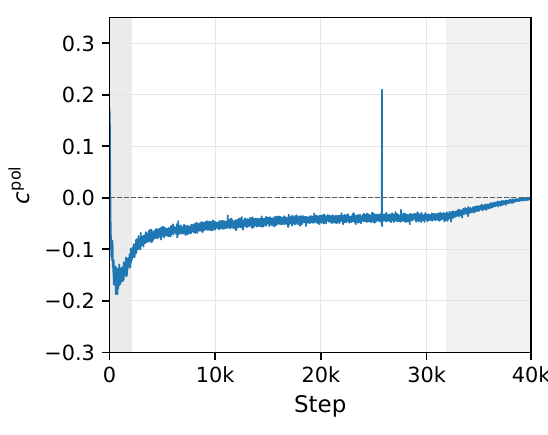}}
\caption{\textbf{130M curvature and update alignment.} Panels (a,b) use peak rate $0.08$; (c,d) use $0.04$. Curvature panels show $\hat s^{\mathrm M}_{k,b}$, $2\hat\rho_{k,b}/\eta_k$, normalized generalized sharpness $\widehat S_{R,k}$, and $2/\eta_k$; the other panels show global update alignment. Insets enlarge steps 0--1,000; curvature axes use a symmetric-log scale. Grey shading marks warmup and decay.}
\label{fig:llm-lr}
\vspace{4pt} 
\endgroup
\end{figure}

\subsection{Effective curvature does not determine update alignment}
\label{sec:noncommuting-quadratic}\label{sec:noisy-quadratic}
A matrix quadratic~\citep{meterez2026defensequadraticmodel} makes the distinction exact.
\begin{proposition}[Identical curvature, opposite update alignment]
\label{prop:curvature-alignment}
Let $L(\bW)=\tfrac12\|\bW\bm X-\bm Y\|_{\mathrm F}^2$,
where $\bW\in\mathbb R^{m\times d}$, $m\ge d$, and
$\bm A=\bm X\bm X^\top\succ0$.
At a full-column-rank gradient $\bG_k$, write
$\bm B_k=(\bG_k^\top\bG_k)^{1/2}$ and assume
$\bm B_k-\eta\bm A$ is nonsingular.
For full-batch exact-polar Muon,
\begin{equation}
s_k^{\mathrm M}
=\frac{\operatorname{tr}(\bm A)}
{\operatorname{tr}(\bm B_k)}\,,
\qquad
c_k^{\mathrm{pol}}
=1-\frac{2n_-(\bm B_k-\eta\bm A)}{d}\,,
\label{eq:quadratic-summary}
\end{equation}
where $n_-$ counts negative eigenvalues.
Thus identical effective curvature at the same learning rate
can yield opposite {update-alignment signs}.
\end{proposition}

\begin{proof}[Proof sketch]
Writing $\bG_k=\bP_k\bm B_k$ gives
$\bG_{k+1}=\bP_k(\bm B_k-\eta\bm A)$.
Since $\bP_k^\top\bP_k=\bm I$, the quadratic loss remainder
depends on $\operatorname{tr}(\bm A)$, whereas the next polar
direction depends on the signs of $\bm B_k-\eta\bm A$.
For $\bm A=\bm I_3$, take
$\bm B_k=\eta\operatorname{diag}(0.8,0.8,0.8)$
or $\eta\operatorname{diag}(0.1,1.1,1.2)$.
Both give $s_k^{\mathrm M}=5/(4\eta)$, but
$c_k^{\mathrm{pol}}=-1$ and $1/3$, respectively.
Appendix~\ref{app:matrix-quadratic} gives the full derivation
and realization on the same objective.
\end{proof}

\section{Language model training at the scale}
\label{sec:practice}

In this section, we study language-model pretraining at the scale of 22M, 130M, and 1B. The used training strategy is with fresh minibatches, a warmup--stable--decay (WSD) schedule~\citep{hu2024minicpm}, approximate polar updates, and AdamW on auxiliary parameters. The 22M and 130M experiments change the learning rate at a shared checkpoint within each model. The 22M experiments also change training batch size. Separate 130M runs track loss balance and update alignment throughout training. {The 1B experiments also record loss balance and update alignment with approximately 40 tokens per parameter.}
We present our main results based on the 130M Llama-like LLM pre-training.

\subsection{{130M: loss balance, generalized sharpness, and update alignment}}
\label{subsec:llm-alignment}

We train a 130.7M Llama-like LLM on FineWeb~\citep{penedo2024fineweb} with batch size 32 and sequence length 4,096, processing 5.23B tokens (see Appendix~\ref{app:llm-settings} for details). Four runs all use batch size 32, with Muon peak learning rates $0.02$, $0.04$, $0.08$, and $0.12$, respectively. We find that

\noindent{\textbf{Loss balance and generalized sharpness.}}
{Across all four learning rates, effective curvature approaches and remains near $2\hat\rho_{k,b}/\eta_k$ (T1), while estimated normalized generalized sharpness is substantially larger and approaches $2/\eta_k$ during the constant-rate stage (\cref{fig:introduction_muon_130M:d,fig:llm-lr}). Thus loss balance occurs at an effective curvature distinct from generalized sharpness.}

\begin{figure}[!t]
\centering
\subfloat[{Loss}\label{fig:130m-momentum-loss}]{\includegraphics[width=.325\linewidth]{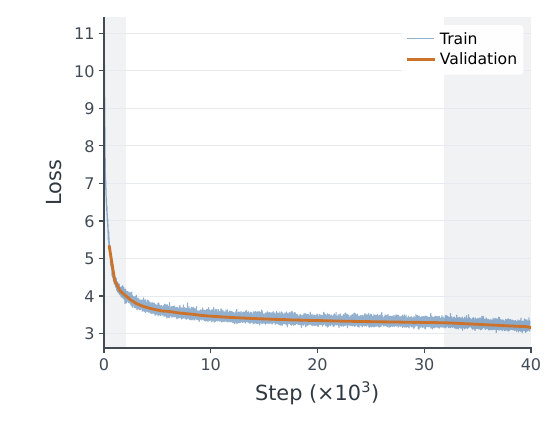}}
\hfill
\subfloat[{Conditional loss balance}\label{fig:130m-momentum-t1}]{\includegraphics[width=.325\linewidth]{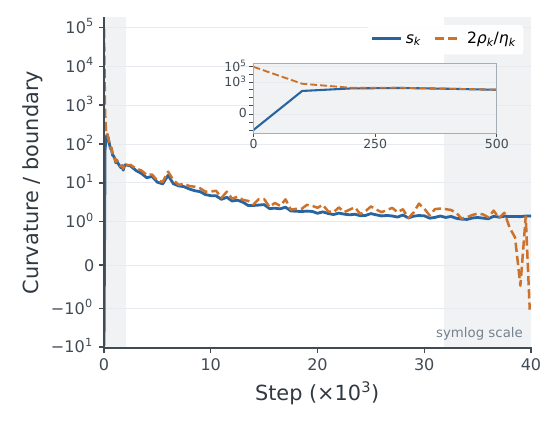}}
\hfill
\subfloat[{{Update alignment}}\label{fig:130m-momentum-t2}]{\includegraphics[width=.325\linewidth]{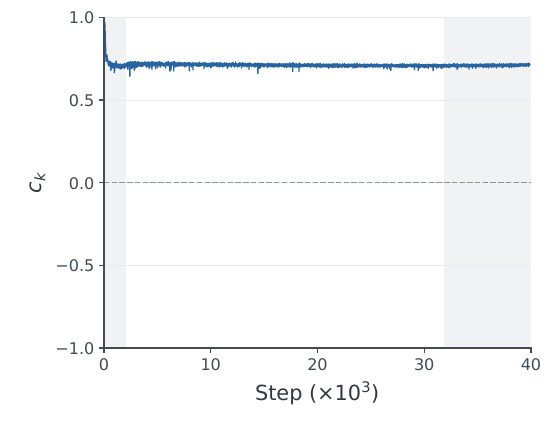}}
\caption{{\textbf{130M Muon with momentum $0.95$.} (a) Training and validation loss. (b) Conditional effective curvature and its {loss-balance boundary}. (c) Global cosine between consecutive actual Muon-matrix updates. }}
\label{fig:130m-momentum}
\end{figure}

\begin{wrapfigure}{r}{0.42\textwidth}
\centering
\includegraphics[width=\linewidth]{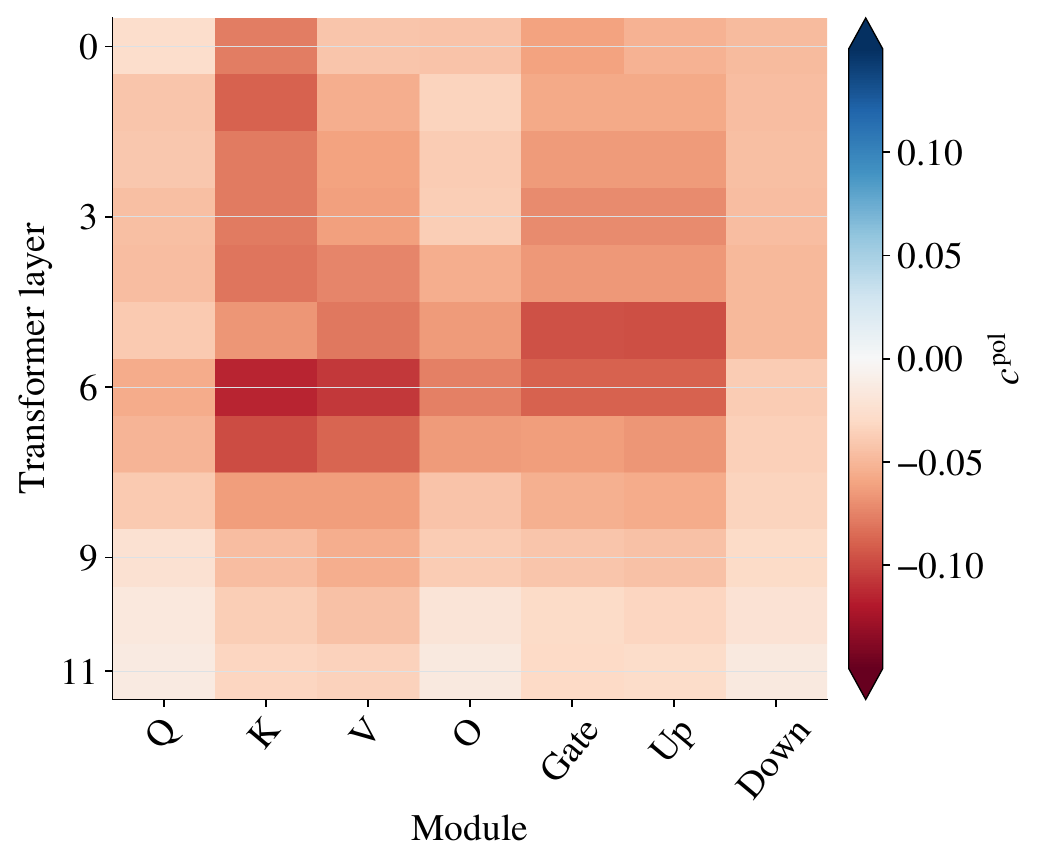}
\caption{\textbf{Modulewise alignment.} Median $c^{\mathrm{pol}}$ during the constant-learning-rate stage of the 130M run with peak rate $0.02$. }
\label{fig:llm-module-medians}
\end{wrapfigure}

\noindent {\bf Loss balance coexists with continued learning.}
Conditional curvature repeatedly lies near $2\hat\rho_{k,32}/\eta_k$
(\cref{fig:introduction_muon_130M:d,fig:llm-lr}).
By the loss identity, this means that the conditional mean reference-loss increment is near zero.
Paired Muon-only probes also have positive mean increments at several checkpoints.
Yet validation loss of the complete Muon-plus-AdamW runs improves over longer horizons
(\cref{fig:llm-lr-loss-appendix} in the appendix).
{Additional checks at peak rate $0.02$ retain near-zero relative conditional loss increments across independently sampled reference sets of 32--128 sequences and complete Muon-plus-AdamW updates (Appendix~\ref{app:toc-7}).}

\noindent {\bf Global directions stay near orthogonality.}
Across all four runs, sampled consecutive NS-5 directions are predominantly negatively aligned (\cref{fig:introduction_muon_130M:c,fig:llm-lr-cosine-004,fig:llm-lr-cosine-008,fig:llm-lr-cosine-012}). Full-horizon cosine medians range from $-0.048$ to $-0.032$, far from {exact direction reversal at} $-1$. Thus, loss-boundary tracking coexists with weak negative alignment and continued learning. {SGD controls on the small Transformer and 130M model also show boundary proximity and negative alignment, but visibly less stable update alignment (Appendix~\ref{app:sgd-comparison}).}

\noindent{\textbf{Muon with momentum exhibits positive update alignment}}
{We additionally train the 130M model with momentum $0.95$. Figure~\ref{fig:130m-momentum} further illustrates why conditional loss balance and update alignment should be examined separately. During the the entire training, consecutive-update cosines are positive, with median $0.712$, while conditional effective curvature still generally hovers around the {loss-balance boundary} . Thus, momentum changes the update alignment while loss balance remains a separate diagnostic.}

\noindent {\bf Negative alignment varies across layers.}
\label{subsec:llm-layers}
At peak learning rate $0.02$, middle-layer key and value matrices have more negative cosine medians than many query matrices (\cref{fig:llm-module-medians}). Weak global alignment therefore hides differences across layers and matrix types. These {update-alignment} measurements complement the within-layer and cross-layer curvature analysis of \citet{wang2026muon}, providing a separate view of Muon's behavior across parameter blocks.  \cref{fig:llm-layer-panel-a} in the appendix shows the block trajectories.

\noindent {\bf Probe batch size changes stochastic coherence.}
\label{subsec:llm-coherence}
\cref{fig:llm-batch} shows that larger probe batches generally strengthen reference-gradient coherence and shift $2\hat\rho_{k,b}/\eta_k$ trajectories.

\subsection{Learning-rate and batch-size perturbations}
\label{subsec:llm-perturbations}

Here we conduct a perturbation experiment by setting different peak learning rates in WSD to monitor how the validation loss and $c_{k,b}^{\mathrm{pol}}$ will change and build their connections. 

In our experiments, we change the Muon learning rate at step 4,000 for the 22M model and step 7,000 for the 130M model. Each experiment branches from rate $0.04$ to $0.02$, $0.04$, or $0.08$. Training batch size stays at 16 for 22M and 32 for 130M. The auxiliary AdamW schedule is unchanged across branches. \cref{fig:perturb22m-lr} shows validation loss and consecutive-direction cosines. Appendix~\ref{app:perturbations} gives the protocols. We have the following findings:

\begin{figure}[!t]
\centering
\subfloat[\small 22M: loss]{\includegraphics[width=.2485\linewidth]{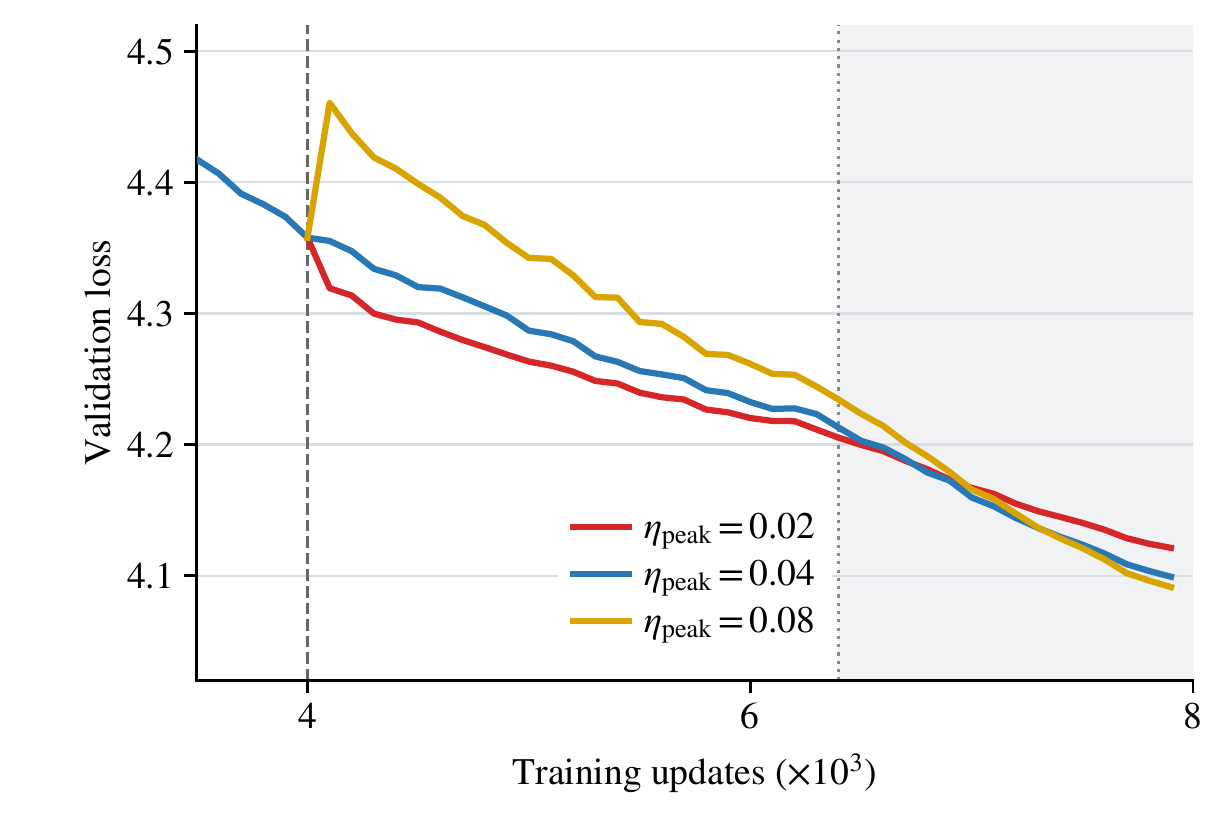}\label{fig:perturb22m-lr-panel-a}}
\hfill
\subfloat[\small 22M: $c_{k,b}^{\mathrm{pol}}$ directions]{\includegraphics[width=.2485\linewidth]{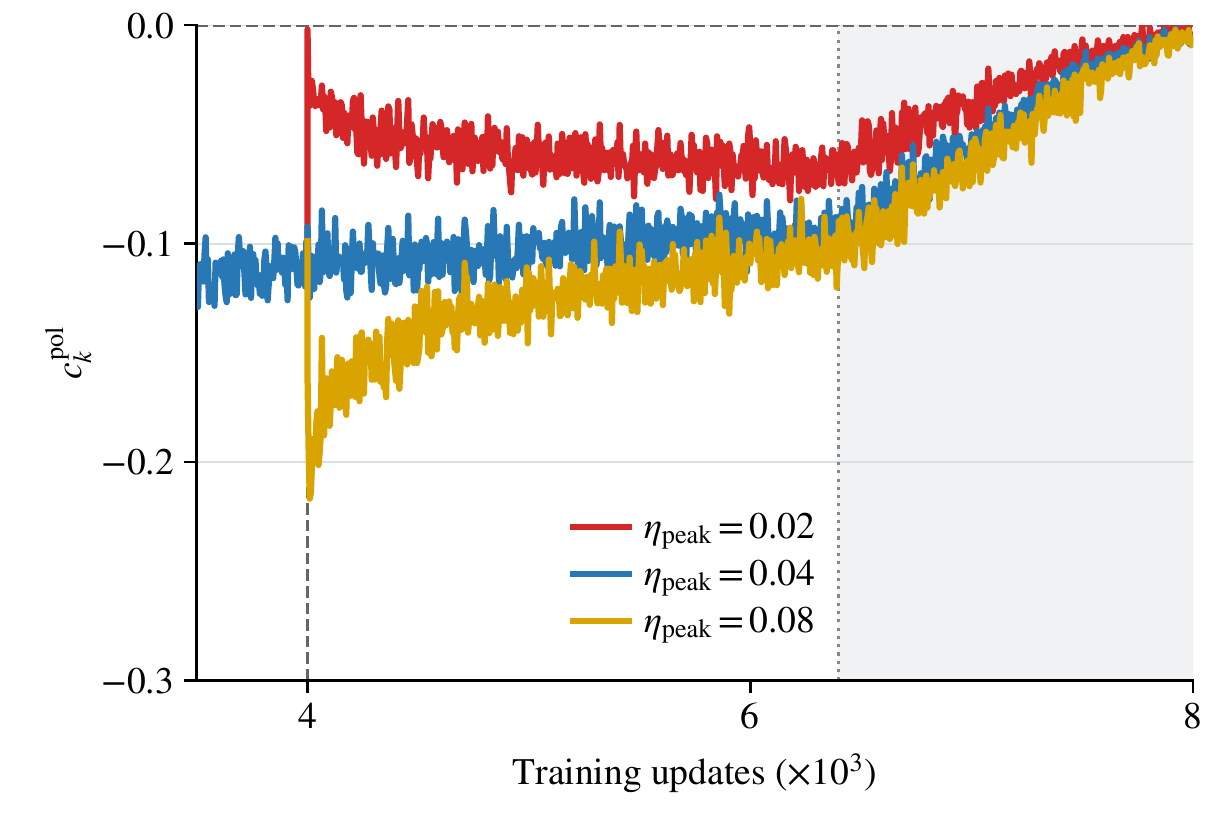}\label{fig:perturb22m-lr-panel-b}}
\subfloat[\small 130M: loss]{\includegraphics[width=.2485\linewidth]{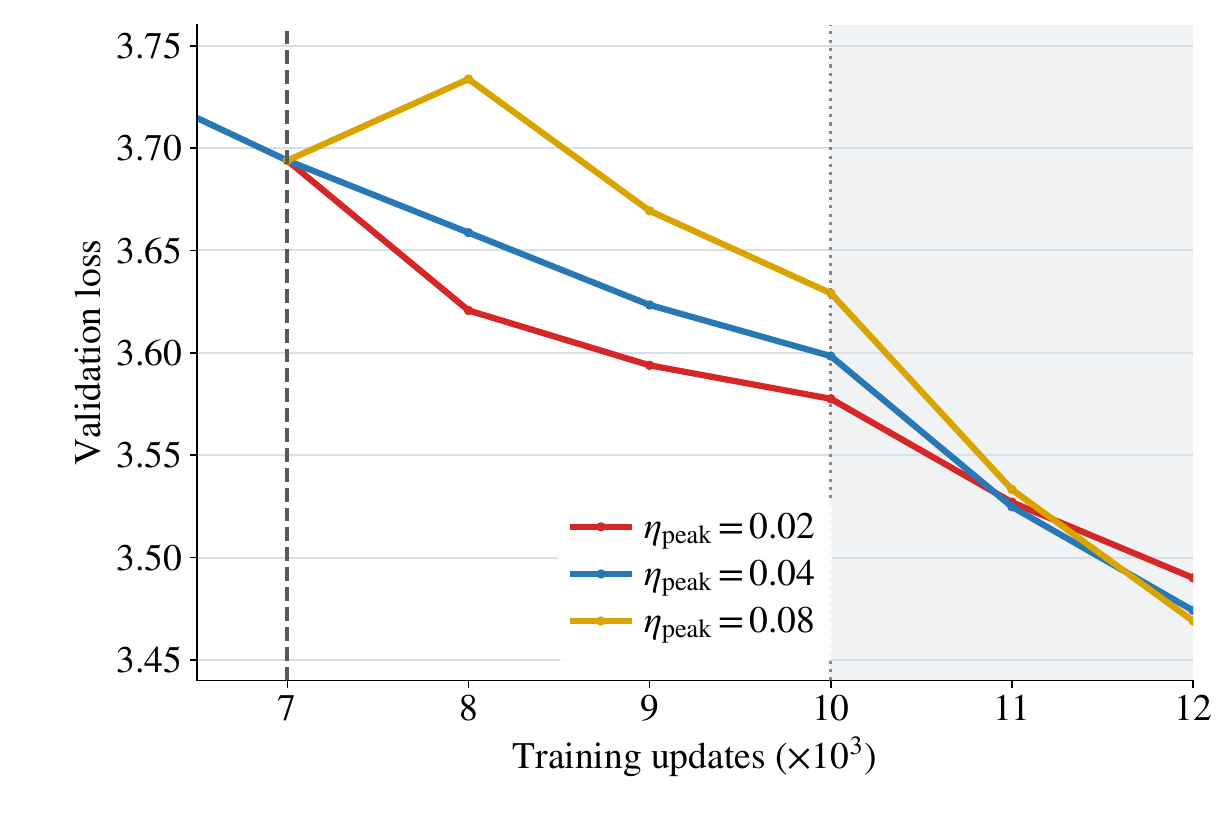}\label{fig:perturb130m-lr-panel-a}}
\hfill
\subfloat[\small 130M: $c_{k,b}^{\mathrm{pol}}$ directions]{\includegraphics[width=.2485\linewidth]{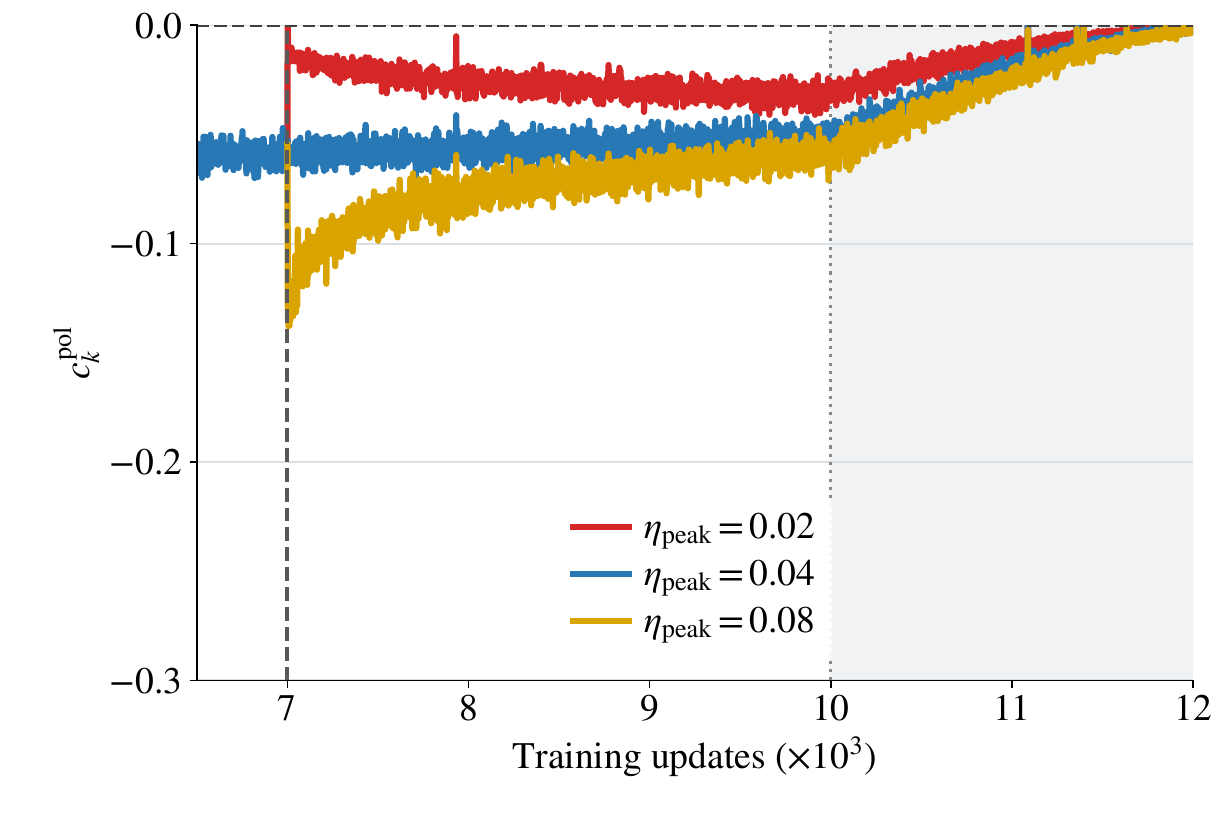}\label{fig:perturb130m-lr-panel-b}}
\caption{\textbf{Loss and update alignment after learning-rate changes.} The left subfigures (a) and (b) show a 22M model on validation loss and and global Muon direction cosines $c_{k,b}^{\mathrm{pol}}$. the right subfigures (c) and (d) show the related results on a 130M Llama-like model. Shaded grey regions mark learning-rate decay.}
\label{fig:perturb22m-lr}
\end{figure}

\begin{wrapfigure}{r}{0.44\textwidth}
  \centering
  \includegraphics[width=\linewidth]{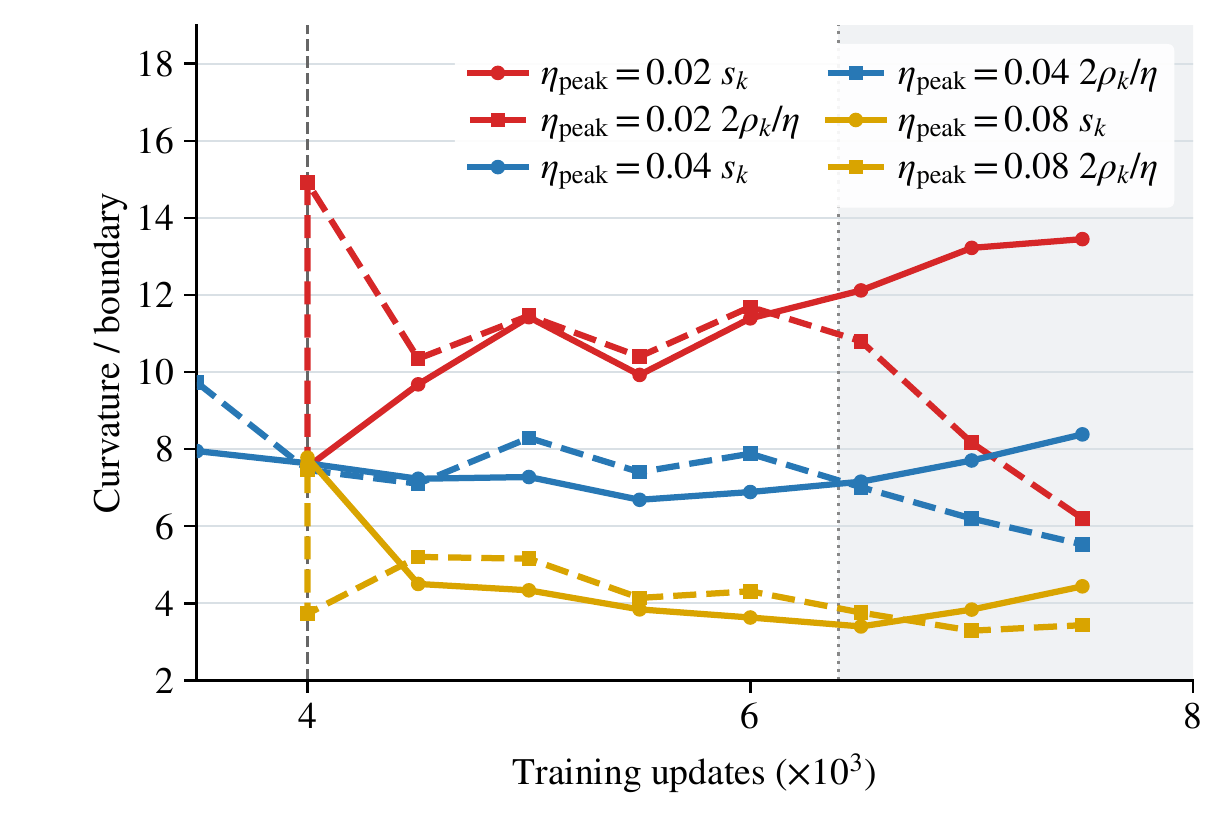}
  \caption{\textbf{22M loss balance after learning-rate changes.}
    Muon rates are $0.02$, $0.04$, and $0.08$.
    Curves compare conditional curvature with $2\hat\rho_k/\eta_k$.}
  \label{fig:perturb22m-lr-balance}
\end{wrapfigure}%

\noindent {\bf Effective curvature and alignment tends to return toward its pre-perturbation value.}
Increasing the learning rate initially makes $c^{\mathrm{pol}}$ more negative; decreasing it makes $c^{\mathrm{pol}}$ less negative. After either change, the cosine tends to return toward its pre-perturbation value in both models (\cref{fig:perturb22m-lr-panel-b,fig:perturb130m-lr-panel-b}). {Later probes show a return toward loss balance (\cref{fig:perturb22m-lr-balance,fig:perturb130m-t1}). In the 130M experiment, the relative discrepancies after 1,000 constant-rate updates are $-0.090$ and $0.003$ for the halved- and doubled-rate branches, compared with $-0.503$ and $0.883$ at the intervention. This is empirical recovery toward loss balance.}

\noindent {\bf Loss balance and alignment respond differently.}
{At the shared checkpoint within each model, halving the learning rate puts the conditional Muon-only step on the loss-decrease side of T1; doubling it puts the step on the loss-increase side (Table~\ref{tab:perturb22m-initial}). } Yet sampled {update-direction cosines} remain negative in all three branches before decay (\cref{fig:perturb22m-lr-panel-b}).
Negative alignment therefore does not determine the sign of the conditional loss change.

{Interestingly, in both models, the larger-learning-rate branch achieves lower validation loss after the prescribed WSD continuation (\cref{fig:perturb22m-lr-panel-a,fig:perturb130m-lr-panel-a}), despite an initially positive conditional loss increment followed by recovery toward loss balance. Thus, loss balance is not a ceiling on useful learning rates, and negative update alignment alone does not determine whether a rate increase is beneficial. Using these diagnostics to design learning-rate schedules remains future work.}


\subsection{1B: large-scale pre-training}
\label{subsec:llm-1b}

{We train a 1B Llama-like model with batch sizes 128, 256, and 512 and sequence length 4,096, using approximately 40B tokens per run. Figure~\ref{fig:llm-1b} compares all three batch sizes. Appendix~\ref{app:1b-protocol} gives the experimental settings and individual batch-128 and batch-256 curves.}

\begin{figure}[!t]
\centering
\subfloat[{\small 1B: loss}\label{fig:llm-1b-loss}]{\includegraphics[width=.325\linewidth]{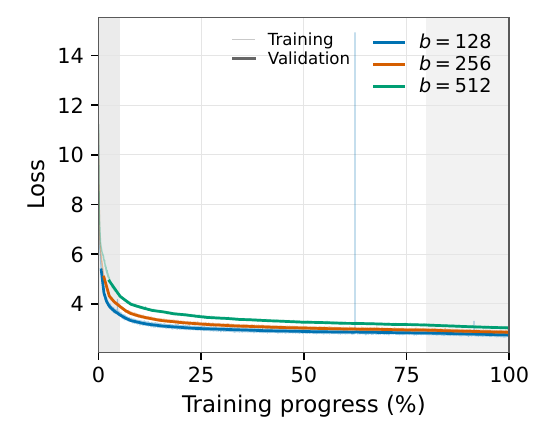}}
\hfill
\subfloat[{\small 1B: loss balance}\label{fig:llm-1b-balance}]{\includegraphics[width=.325\linewidth]{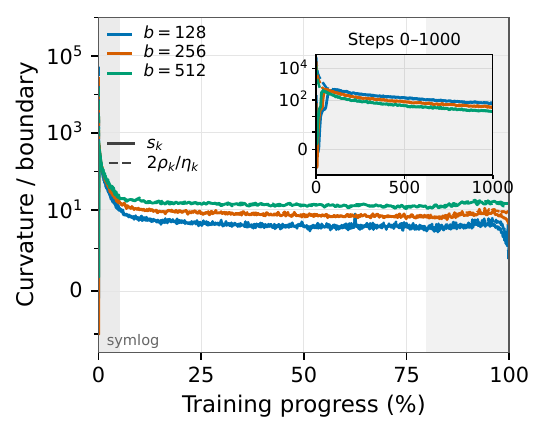}}
\hfill
\subfloat[{\small 1B: update alignment}\label{fig:llm-1b-global}]{\includegraphics[width=.325\linewidth]{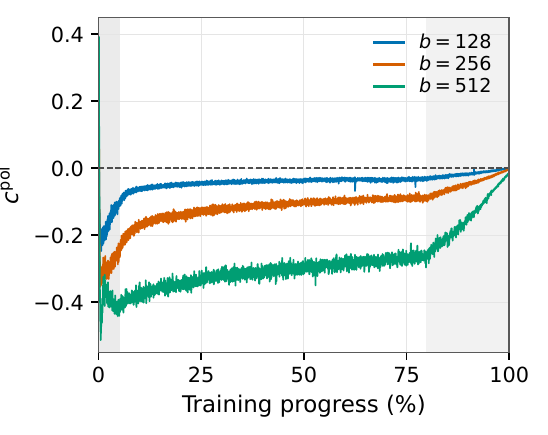}}
\caption{{\textbf{1B pretraining across batch sizes 128, 256, and 512.} Panels show loss, T1, and global update alignment. Colors denote batches; the horizontal axis is training progress. The curvature inset shows the first 1,000 updates; curvature axes use a symmetric-log scale. Grey shading marks warmup and decay.}}
\label{fig:llm-1b}
\vspace{4pt} 
\end{figure}

\noindent{\textbf{Loss balance with partial directional cancellation.}}
{In Figure~\ref{fig:llm-1b}, loss fluctuates over short windows but decreases over longer horizons while conditional effective curvature tracks the boundary. Global $c_k^{\mathrm{pol}}$ is negative but remains far from exact direction reversal at $-1$. The batch-128 and batch-256 runs exhibit the same qualitative loss-boundary tracking, with less negative global alignment than the batch-512 run. }

\section{Conclusion and discussion}
\label{sec:conclusion}

{Muon exhibits a split EoS: effective curvature separates from normalized generalized sharpness, and loss curvature does not determine update alignment. We derive an exact conditional loss identity with boundary $2\rho_{k,b}/\eta_k$ and, under the assumptions of Proposition~\ref{prop:curvature-visits}, prove that stochastic effective curvature almost surely cannot remain a fixed distance below this boundary when the cumulative gradient norm diverges. Across four learning rates in 130M pretraining, effective curvature approaches its loss-balance boundary before normalized generalized sharpness approaches $2/\eta_k$, and the two curvature scales remain separated. Our matrix quadratic establishes that identical effective curvature can yield opposite update-alignment signs; controlled and LLM experiments further show loss balance without universal update reversal. Further directions include certified generalized-sharpness estimation, trajectory-level uncertainty quantification, and validation across independent training seeds.}

{Checkpoint interventions show that a larger learning rate can initially increase conditional loss, recover toward loss balance, and improve validation loss after WSD continuation. These findings call for a theory connecting loss balance, generalized sharpness, and update alignment to multi-step progress. They also provide possibility of designing better learning-rate selection rules. Our theory concerns exact-polar no-momentum Muon, whereas LLM experiments use NS-5 updates with auxiliary AdamW. Extending the analysis to these practical updates and momentum remains open.}

\subsection*{AI use statement}
Generative AI tools assisted language editing, LaTeX organization, algebraic consistency checks, and plotting code in this draft. Experimental measurements were supplied separately and were not generated by AI. AI assistance also supported the diagnostic summaries. The latter measurements were produced by running the code on the authors' hardware. The authors are responsible for the manuscript, experimental results, and all AI-assisted material.

\subsection*{Reproducibility statement}
Section~\ref{sec:stochastic} defines the updates and sampling conventions. Appendix~\ref{app:theory-all} gives the proofs and matrix analysis. Appendix~\ref{app:experiments} documents the controlled experiments, including the frozen-state probes. Appendix~\ref{app:perturbations} gives the 22M and 130M perturbation protocols and results. Appendix~\ref{app:llm-settings} lists the 130M and 1B pretraining settings and diagnostics. The source code can be found in the supplementary material.

\FloatBarrier
\bibliographystyle{reference}
\bibliography{references}

\clearpage
\appendix
\hypersetup{linkcolor=blue}
\begin{center}
\noindent\rule{\linewidth}{1.2pt}\par
\vspace{0.4em}
{\Large\bfseries Appendix}\par
\vspace{0.4em}
\noindent\rule{\linewidth}{0.4pt}
\end{center}
\vspace{0.5em}
\section*{Contents}
\makeatletter
\newcommand{\appsectionentry}[2]{
  \par\addvspace{0.5em}\noindent
  \hyperref[#1]{\makebox[2.2em][l]{\ref*{#1}}#2}\hfill\pageref*{#1}\par}
\newcommand{\appsubsectionentry}[2]{
  \@dottedtocline{1}{1.5em}{3em}{\hyperref[#1]{\numberline{\ref*{#1}}#2}}{\pageref*{#1}}}
\makeatother
\begingroup
\small
\appsectionentry{app:theory-all}{Theory in Section~\ref*{sec:stochastic}}
\appsubsectionentry{app:det-details}{Conditional loss balance and stochastic curvature visits}
\appsubsectionentry{app:matrix-quadratic}{Matrix recursion and separation of T1 and T2}
\appsubsectionentry{app:theory}{Secant response and the direction-sign boundary}
\appsectionentry{app:experiments}{Experiments on small models}
\appsubsectionentry{app:toc-1}{Data and objectives}
\appsubsectionentry{app:toc-2}{Architectures and initialization}
\appsubsectionentry{app:toc-3}{Optimizer and sampling}
\appsubsectionentry{app:toc-4}{Conditional probes and numerical normalization}
\appsubsectionentry{app:events}{Event rules and uncertainty}
\appsubsectionentry{app:toc-5}{Additional trajectory results}
\appsubsectionentry{app:fixedcp-protocol}{Frozen-state and {update alignment}}
\appsectionentry{app:perturbations}{Perturbation experiments}
\appsubsectionentry{app:perturb22m-protocol}{22M setup and probes}
\appsubsectionentry{app:perturb130m-protocol}{130M learning-rate perturbation setup}
\appsubsectionentry{app:perturb22m-evidence}{Loss and direction responses}
\appsectionentry{app:llm-settings}{Results on 130M and 1B pretraining}
\appsubsectionentry{app:130m-protocol}{130M pre-training protocol}
\appsubsectionentry{app:llm-estimators}{Conditional probes and update scope}
\appsubsectionentry{app:generalized-sharpness}{{Normalized generalized sharpness}}
\appsubsectionentry{app:toc-7}{{130M loss balance: reference sets and complete updates}}
\appsubsectionentry{app:momentum}{{130M Muon with momentum}}
\appsubsectionentry{app:llm-structure}{130M layer structure and batch-dependent coherence}
\appsubsectionentry{app:1b-protocol}{1B LLM pretraining}
\begingroup
\appsectionentry{app:sgd-comparison}{SGD comparison and spectral weighting}
\appsubsectionentry{app:gd-general-signatures}{{GD beyond the scalar quadratic}}
\appsubsectionentry{app:sgd-130m}{{SGD in 130M pretraining}}
\appsubsectionentry{app:sgd-spectral-cosine}{How singular values enter the cosine}
\appsubsectionentry{app:sgd-quadratic}{GD on the same matrix quadratic as Muon}
\endgroup
\endgroup
\vspace{1em}

\section{Theory in Section~\ref{sec:stochastic}}
\label{app:theory-all}
\subsection{Conditional loss balance and stochastic curvature visits}
\label{app:det-details}
\noindent{\textbf{Derivation of the conditional loss identity.}}
Rearranging \cref{eq:stoch-curvature} and substituting \cref{eq:rho} gives \cref{eq:stoch-loss}. Since $\eta^2\|\bG_k\|_*/2>0$, the expected loss change has the sign of $s_b^{\mathrm M}-2\rho_b/\eta$. This proves the {loss-balance boundary}. For full-batch updates, it reduces as follows.

For full-batch exact-polar Muon, \cref{eq:operator-nuclear-duality} gives
$\ipF{\bG_k}{\bP_k}=\nucnorm{\bG_k}$, so $\rho_n=1$.
Write $s_k^{\mathrm M}$ for the full-batch specialization of
\cref{eq:stoch-curvature}. On an active step, \cref{eq:stoch-loss} becomes
\begin{equation}
L(\bW_{k+1})-L(\bW_k)
=\frac{\eta^2\nucnorm{\bG_k}}2
\left(s_k^{\mathrm M}-\frac2\eta\right).
\label{eq:det-loss}
\end{equation}
Thus the {loss-balance boundary} is $2/\eta$.

If $L$ is $C^2$ near the update segment, Taylor's integral remainder gives
\begin{equation*}
s_k^{\mathrm M}
=\frac{2}{\nucnorm{\bG_k}}
\int_0^1(1-t)
\ipF{\bP_k}{\nabla^2L(\bW_k-t\eta\bP_k)[\bP_k]}\dd t.
\end{equation*}
Here the Hessian acts on a matrix direction by
\[
\nabla^2L(\bW)[\bP]
=\left.\frac{\mathrm d}{\mathrm du}\nabla L(\bW+u\bP)\right|_{u=0}.
\]
The limit $\eta\to0$ holds at fixed $\bW_k$ and $\bP_k$; it is a
small-step limit, not a continuous-time trajectory.

The tight operator-norm directional smoothness of
\citet{islamov2026noneuclidean}, specialized to this update, is
\begin{equation*}
\begin{aligned}
D_{\mathrm{op}}(\bW_k,\bW_{k+1})
&=\frac{2\bigl[L(\bW_{k+1})-L(\bW_k)
-\ipF{\bG_k}{\bW_{k+1}-\bW_k}\bigr]}
{\opnorm{\bW_{k+1}-\bW_k}^2},\\
s_k^{\mathrm M}
&=\frac{D_{\mathrm{op}}(\bW_k,\bW_{k+1})}{\nucnorm{\bG_k}}.
\end{aligned}
\end{equation*}
This follows from $\opnorm{\bP_k}=1$ and operator--nuclear duality.
It concerns the realized update segment, rather than a maximum over
directions. Positive rescaling of $L$ leaves $s_k^{\mathrm M}$ unchanged.

\begingroup
\noindent\textbf{Stochastic curvature visits.}
We prove the stronger event-wise statement
\begin{equation}
\mathbb P\!\left(
\sum_{k=0}^{\infty}\eta^2\nucnorm{\bG_k}=\infty,\quad
\limsup_{k\to\infty}
\left(s_{k,b}^{\mathrm M}-\frac{2\rho_{k,b}}{\eta}\right)<0
\right)=0.
\label{eq:stochastic-curvature-visits}
\end{equation}
This only requires the other assumptions of
Proposition~\ref{prop:curvature-visits}; the cumulative sum need
not diverge almost surely.

\begin{proof}[Proof of Proposition~\ref{prop:curvature-visits}]
Let $\mathcal F_k=\sigma(B_0,\ldots,B_{k-1})$ denote the history
before sampling $B_k$, and let $L_{\mathrm{lb}}$ be a finite lower
bound for $L$.

\noindent\textbf{Finite-horizon integrability.}
Exact polar factors satisfy
$\normF{\bP_k}\leq\sqrt{\min(m,n)}$, so
\[
\normF{\bW_k-\bW_0}
\leq k\eta\sqrt{\min(m,n)}.
\]
At each finite horizon, continuity of $L$ and $\nabla L$ therefore
bounds the loss and gradient norm on a deterministic compact ball.
All finite-horizon expectations below are consequently well defined;
no boundedness of the infinite trajectory is assumed.

\noindent\textbf{Conditional drift.}
Fresh independent sampling makes \cref{eq:stoch-loss} valid
conditional on $\mathcal F_k$:
\begin{equation}
\E\!\left[
L(\bW_{k+1})-L(\bW_k)\mid\mathcal F_k
\right]
=
\frac{\eta^2\nucnorm{\bG_k}}{2}
\left(s_{k,b}^{\mathrm M}-\frac{2\rho_{k,b}}{\eta}\right).
\label{eq:curvature-filtration-drift}
\end{equation}
Both the gradient norm and the boundary gap are
$\mathcal F_k$-measurable.

\noindent\textbf{Stopping at a fixed margin.}
Fix a deterministic integer $t\geq0$ and $\delta>0$, and define
\[
\tau=\inf\left\{
k\geq t:
s_{k,b}^{\mathrm M}-\frac{2\rho_{k,b}}{\eta}>-\delta
\right\},
\qquad \inf\varnothing=\infty.
\]
This is a stopping time. For $k\geq t$,
\[
L(\bW_{(k+1)\wedge\tau})-L(\bW_{k\wedge\tau})
=
\mathbf 1_{\{k<\tau\}}
\bigl[L(\bW_{k+1})-L(\bW_k)\bigr].
\]
Using \cref{eq:curvature-filtration-drift} and the definition of
$\tau$ gives
\begin{align*}
&\E\!\left[
L(\bW_{(k+1)\wedge\tau})-L(\bW_{k\wedge\tau})
\mid\mathcal F_k
\right]\\
&\qquad\leq
-\frac{\delta}{2}\eta^2\nucnorm{\bG_k}
\mathbf 1_{\{k<\tau\}}.
\end{align*}

\noindent\textbf{Bounding the stopped cumulative weight.}
Taking expectations and summing over $k=t,\ldots,K-1$ yields
\[
\frac{\delta}{2}
\E\!\left[
\sum_{k=t}^{K-1}
\eta^2\nucnorm{\bG_k}\mathbf 1_{\{k<\tau\}}
\right]
\leq
\E[L(\bW_t)]-\E[L(\bW_{K\wedge\tau})]
\leq
\E[L(\bW_t)]-L_{\mathrm{lb}}.
\]
By monotone convergence,
\[
\E\!\left[
\sum_{k=t}^{\infty}
\eta^2\nucnorm{\bG_k}\mathbf 1_{\{k<\tau\}}
\right]
\leq
\frac{2\bigl(\E[L(\bW_t)]-L_{\mathrm{lb}}\bigr)}{\delta}
<\infty.
\]
Thus the stopped sum is finite almost surely. Since every finite
prefix of the cumulative sum is finite,
\[
\mathbb P\!\left(
\tau=\infty,\quad
\sum_{k=0}^{\infty}\eta^2\nucnorm{\bG_k}=\infty
\right)=0.
\]

\noindent\textbf{Countable union.}
If the limsup of the boundary gap is negative, then some integers
$t\geq0$ and $j\geq1$ satisfy
\[
s_{k,b}^{\mathrm M}-\frac{2\rho_{k,b}}{\eta}
\leq-\frac1j
\qquad\text{for every }k\geq t.
\]
For this choice of $t$ and $\delta=1/j$, we have $\tau=\infty$.
Taking the countable union over $(t,j)$ proves
\cref{eq:stochastic-curvature-visits}.
The almost-sure divergence assumption in the proposition then
gives its stated conclusion.
\end{proof}

\noindent\textbf{Full-batch specialization and scope.}
At full batch, $\rho_{k,\mathrm{full}}=1$, so the conclusion
reduces to $\limsup_k s_k^{\mathrm M}\geq2/\eta$ whenever the
cumulative weight diverges. The stochastic result rules out a
persistent fixed gap below the moving boundary; it does not
establish convergence, persistent boundary tracking, or
perturbation stability. The nonzero-gradient assumption is
essential to this formulation: stochastic updates can move even
when the population gradient vanishes, so such steps cannot
simply be discarded.
\endgroup

\subsection{Matrix recursion and separation of T1 and T2}
\label{app:matrix-quadratic}
\label{app:quadratic-recursion}
\label{app:quadratic-diagnostics}
We analyze the matrix quadratic from \cref{sec:noncommuting-quadratic}. Its Hessian is the right-multiplication
map $\bDelta\mapsto\bDelta\bm X\bm X^\top$.
The formulas below use this structure; they need not hold for an arbitrary
quadratic on matrix space.

Assume, as in the main text, that $m\geq d$, $\bG_k$ has full column rank,
and ${\bm B}_k-\eta\bm X\bm X^\top$ is nonsingular.
Then
\begin{align}
\bG_{k+1}&=\bP_k({\bm B}_k-\eta\bm X\bm X^\top),\notag
\\
{\bm B}_{k+1}&=|{\bm B}_k-\eta\bm X\bm X^\top|,\notag
\\
\bP_{k+1}&=\bP_k\operatorname{sign}({\bm B}_k-\eta\bm X\bm X^\top).
\label{eq:polar-matrix-recursion}
\end{align}
Absolute value and sign act eigenvaluewise on the symmetric matrix.
No commutativity assumption is needed.

The affine gradient gives
$\bG_{k+1}=\bG_k-\eta\bP_k\bm X\bm X^\top$.
Substitute $\bG_k=\bP_k{\bm B}_k$ and use
$\bP_k^\top\bP_k=\bm I_d$ to obtain
\[
\bG_{k+1}^\top\bG_{k+1}
=({\bm B}_k-\eta\bm X\bm X^\top)^2.
\]
Taking its positive square root proves the second formula.
The identity
$\bP_{k+1}=\bG_{k+1}{\bm B}_{k+1}^{-1}$ proves the third.

\begin{corollary}[T1 and T2 for the matrix quadratic]
\label{cor:noncommuting-thresholds}
Under the assumptions above,
\begin{align*}
L_{\bm X,\bm E}(\bW_{k+1})-L_{\bm X,\bm E}(\bW_k)
&=-\eta\operatorname{tr}({\bm B}_k)
+\frac{\eta^2}{2}\normF{\bm X}^2,
\\
s_k^{\mathrm M}
&=\frac{\normF{\bm X}^2}{\operatorname{tr}({\bm B}_k)},
\\
c_k^{\mathrm{pol}}
&=1-\frac{2n_-({\bm B}_k-\eta\bm X\bm X^\top)}d.
\end{align*}
Here $n_-$ counts negative eigenvalues.
\end{corollary}
\begin{proof}
The exact quadratic expansion gives
\[
L_{\bm X,\bm E}(\bW_k-\eta\bP_k)-L_{\bm X,\bm E}(\bW_k)
=-\eta\ipF{\bG_k}{\bP_k}
+\frac{\eta^2}{2}\normF{\bP_k\bm X}^2.
\]
Full column rank implies
$\ipF{\bG_k}{\bP_k}=\operatorname{tr}({\bm B}_k)$ and
$\normF{\bP_k\bm X}^2=\normF{\bm X}^2$.
These identities prove the loss and curvature formulas.
By \cref{eq:polar-matrix-recursion},
\[
\ipF{\bP_k}{\bP_{k+1}}
=\operatorname{tr}\operatorname{sign}({\bm B}_k-\eta\bm X\bm X^\top)
=d-2n_-({\bm B}_k-\eta\bm X\bm X^\top).
\]
Both polar factors have squared Frobenius norm $d$, giving the cosine.
\end{proof}

{T1 is the trace condition
$\operatorname{tr}({\bm B}_k)=\eta\normF{\bm X}^2/2$.
{Zero update alignment corresponds to $n_-=d/2$.}
For odd $d$, the cosine can change sign without attaining zero.}

\noindent {\bf {Exact direction reversal} and two-step displacement.}
Let $\bDelta_k=-\eta_k\bP_k$, with $\eta_k>0$.
For consecutive nonzero directions,
\begin{equation}
\begin{aligned}
\normF{\bW_{k+2}-\bW_k}^2
={}&\normF{\bDelta_k}^2+\normF{\bDelta_{k+1}}^2 +2\normF{\bDelta_k}\normF{\bDelta_{k+1}}c_{k,b}^{\mathrm{pol}}.
\end{aligned}
\label{eq:two-step-displacement}
\end{equation}
At cosine $-1$, Cauchy--Schwarz gives $\bP_{k+1}=-a\bP_k$ for $a>0$.
Nonzero canonical polar factors have operator norm one, so $a=1$.
At fixed learning rate this proves \cref{eq:exact-update-reversal}.
For unequal learning rates or general implemented directions, cosine
$-1$ alone does not imply equal displacement magnitudes.

{Take the two stretch matrices in Section~\ref{sec:noncommuting-quadratic}.
They are realized by $m=d=n=3$, $\bm X=\bm I_3$, and
$\bW_k=\bW_\star+\bm E+{\bm B}_k$, giving
$\bG_k={\bm B}_k$ and $\bP_k=\bm I_3$.
Both stretches have trace $2.4\eta$, hence curvature $5/(4\eta)$.
Their shifted matrices have three and one negative eigenvalues,
respectively, so their cosines are $-1$ and $1/3$.
Thus equal loss curvature can accompany opposite alignment signs.
This is a one-step separation result, not a claim of long-time attraction.}

\subsection{Secant response and the direction-sign boundary}
\label{app:theory}
\label{app:secant-details}
Scalar quadratic GD reverses direction above $1/\eta$, before loss balance at $2/\eta$. For Muon, the direction-sign boundary also depends on gradient geometry.

\subsubsection{The GD reference value}
For GD, define the directional secant response on a nonzero gradient by
\[
\overline s_k^{\mathrm{GD}}
:=-\frac{\ipF{\bG_k}{\bG_{k+1}-\bG_k}}
{\eta\normF{\bG_k}^2}.
\]
Then
$\ipF{\bG_k}{\bG_{k+1}}=
\normF{\bG_k}^2(1-\eta\overline s_k^{\mathrm{GD}})$.
For nonzero consecutive gradients, their cosine changes sign across
$\overline s_k^{\mathrm{GD}}=1/\eta$.
For a scalar quadratic, $\overline s_k^{\mathrm{GD}}=\lambda$:
the next gradient vanishes at $\eta\lambda=1$ and reverses sign above it.
The cosine at that zero-gradient point is undefined.

\subsubsection{The full-batch Muon condition}
Assume consecutive full gradients are nonzero and use exact polar factors.
Define
\begin{equation*}
\begin{aligned}
\overline s_k^{\mathrm M}
&:=-\frac{\ipF{\bP_k}{\bG_{k+1}-\bG_k}}
{\eta\nucnorm{\bG_k}},\\
a_k&:=\frac{\nucnorm{\bG_{k+1}}}{\nucnorm{\bG_k}},
\qquad
\chi_k:=-\frac{\ipF{\bP_k}{\bG_{k+1}}}{\nucnorm{\bG_{k+1}}}.
\end{aligned}
\end{equation*}
Operator--nuclear duality gives
\begin{equation}
\overline s_k^{\mathrm M}=\frac{1+a_k\chi_k}{\eta}.
\label{eq:det-secant}
\end{equation}
Thus $\overline s_k^{\mathrm M}=1/\eta$ means
$\ipF{\bP_k}{\bG_{k+1}}=0$.
{Zero update alignment instead requires $\ipF{\bP_k}{\bP_{k+1}}=0$.}
The first pairing weights the next singular directions by their singular
values; the second does not.

{To express the {update-alignment sign} boundary in secant form}, let $r_k=\rank(\bG_k)$ and write
\begin{equation*}
\bm R_{k+1}:=\bG_{k+1}
-\frac{\nucnorm{\bG_{k+1}}}{r_{k+1}}\bP_{k+1},
\qquad
\vartheta_k:=1-\frac{\ipF{\bP_k}{\bm R_{k+1}}}{\nucnorm{\bG_k}}.
\end{equation*}
Since $\normF{\bP_k}^2=r_k$, substitution gives the exact relation
\begin{equation*}
\overline s_k^{\mathrm M}
=\frac{\vartheta_k-a_k\sqrt{r_k/r_{k+1}}\,c_k^{\mathrm{pol}}}{\eta}.
\end{equation*}
Its positive cosine coefficient implies
\begin{equation}
\begin{aligned}
c_k^{\mathrm{pol}}=0
&\iff \overline s_k^{\mathrm M}=\frac{\vartheta_k}{\eta},\\
c_k^{\mathrm{pol}}<0
&\iff \overline s_k^{\mathrm M}>\frac{\vartheta_k}{\eta}.
\end{aligned}
\label{eq:muon-t2-secant}
\end{equation}
Unlike GD, Muon has a geometry-dependent coefficient $\vartheta_k$,
which need not equal one. This is an exact relation for a given step,
not a universal threshold or a prediction from secant curvature alone.
If the next nonzero singular values are equal, $\bm R_{k+1}=\bm0$
and the value reduces to $1/\eta$.

{For example, take the quadratic model with $\bm X=\bm I_2$ and
$\bG_k={\bm B}_k=\eta\operatorname{diag}(1/2,5/2)$.
Then $\bP_k=\bm I_2$ and
$\bP_{k+1}=\operatorname{diag}(-1,1)$, so $c_k^{\mathrm{pol}}=0$.
Yet $\overline s_k^{\mathrm M}=2/(3\eta)$ and $\vartheta_k=2/3$.}

\subsubsection{Relation to loss curvature}
For $L\in C^2$ near the update segment, let
$q_k(t)=\ipF{\bP_k}{\nabla^2L(\bW_k-t\eta\bP_k)[\bP_k]}$.
The gradient difference and Taylor remainder give
\begin{equation*}
\overline s_k^{\mathrm M}
=\frac{\int_0^1q_k(t)\dd t}{\nucnorm{\bG_k}},
\qquad
s_k^{\mathrm M}
=\frac{2\int_0^1(1-t)q_k(t)\dd t}{\nucnorm{\bG_k}}.
\end{equation*}
These quantities agree for a constant Hessian, including the quadratic
model in Appendix~\ref{app:matrix-quadratic}.
More generally, for
$\omega_k=\sup_{t,t'\in[0,1]}|q_k(t)-q_k(t')|$,
\[
|s_k^{\mathrm M}-\overline s_k^{\mathrm M}|
\leq\frac{\omega_k}{\nucnorm{\bG_k}}.
\]
At exact polar reversal, $\chi_k=1$, hence
$\overline s_k^{\mathrm M}=(1+a_k)/\eta$.
This reaches $2/\eta$ when $a_k=1$; it differs from the T2 sign boundary.
For an exact two-step return, summing \cref{eq:det-loss} instead gives
\[
\nucnorm{\bG_k}\left(s_k^{\mathrm M}-\frac2\eta\right)
+\nucnorm{\bG_{k+1}}\left(s_{k+1}^{\mathrm M}-\frac2\eta\right)=0.
\]
The two deviations cancel after weighting. Neither step must be individually
{at loss balance}.

\subsubsection{Stochastic secants}
Replacing full gradients in \cref{eq:det-secant} by consecutive nonzero
batch gradients gives
\begin{equation*}
\overline s_{k,b}^{\mathrm M}
:=-\frac{\ipF{\bP_k}{\widehat\bG_{k+1}-\widehat\bG_k}}
{\eta\nucnorm{\widehat\bG_k}}
=\frac{1+a_{k,b}\chi_{k,b}}{\eta},
\end{equation*}
where $a_{k,b}=\nucnorm{\widehat\bG_{k+1}}/\nucnorm{\widehat\bG_k}$
and $\chi_{k,b}=-\ipF{\bP_k}{\widehat\bG_{k+1}}/
\nucnorm{\widehat\bG_{k+1}}$.
This pathwise identity uses exact polar factors and a fixed step size.
Batch resampling contributes to the gradient difference, so the stochastic
secant is not a pure Hessian average or the conditional T1 curvature.

\section{Experiments on small models}
\label{app:experiments}

In this section, we present the experimental settings and results of MLP, CNN, and Transformers.
The settings are summarized in \cref{tab:small}.

\subsection{Data and objectives}
\label{app:toc-1}
The MLP and CNN use all 50,000 CIFAR-10 training images for training. The loss is the mean of $\frac12\|\bm f(\bm x)-\bm y\|_2^2$, summed over ten one-hot outputs. Each pixel/channel coordinate is standardized using training-set statistics. No data augmentation is considered.

The Transformer uses 2,048 fixed length-64 sequences from the first 90\% of TinyShakespeare, with vocabulary size 65. Sequence starts use data seed 0 and are saved in the configuration. Its loss is mean per-token cross entropy. All reported objectives are full training losses.

\subsection{Architectures and initialization}
\label{app:toc-2}
\begin{table}[t]
\caption{Small-network architectures and parameter counts.}
\label{tab:small}
\centering
\begin{tabularx}{\linewidth}{@{}lXr@{}}
\toprule
Model & Architecture & Parameters\\
\midrule
MLP & Flattened input, $3072\to128\to128\to10$; tanh after the two hidden layers; biases enabled & 411,146\\
CNN & $3\times3$ padded convolutions, $3\to32\to64$ channels; tanh and $2\times2$ average pooling after each; $4,096\to128\to10$ dense layers with hidden tanh; biases enabled & 545,098\\
Transformer & Two pre-LayerNorm blocks, width 64, four heads, FFN width 256 with GELU; separate Q/K/V; learned token and position embeddings; final LayerNorm and untied bias-free output head & 112,512\\
\bottomrule
\end{tabularx}
\end{table}
The MLP and CNN use PyTorch's default Linear/Conv2d initialization after setting training seed 0. Transformer Linear and Embedding weights use independent normal initialization with standard deviation 0.02, Linear biases are zero, and LayerNorm uses its default initialization. There is no dropout or BatchNorm. All models run in evaluation mode while retaining gradients.

\subsection{Optimizer and sampling}
\label{app:toc-3}
All trainable tensors receive SVD polar updates at $\eta=0.007$. Convolution kernels are reshaped to output channels by remaining coordinates; vectors become column matrices. There is no momentum, weight decay, auxiliary optimizer, or learning-rate schedule.

The suite contains 17 runs of 1,500 updates, with training seed 0:
\begin{itemize}
\item MLP: batch sizes $1,64,128,256,512,4096,\mathrm{full}$.
\item CNN: batch sizes $1,64,512,4096,\mathrm{full}$.
\item Transformer: batch sizes $1,16,128,512,\mathrm{full}$.
\end{itemize}
Full batch contains 50,000 images or 2,048 sequences. Finite batches are sampled without replacement within a step and independently across steps. Batch-size runs share prefixes of sampled index lists. Data, training-batch, and diagnostic seeds are 0, 17, and 29. Microbatch sizes are 1,000, 256, and 64 for MLP, CNN, and Transformer, respectively.

\subsection{Conditional probes and numerical normalization}
\label{app:toc-4}
At each iterate, independent probe batches are evaluated on a separate model copy initialized at the same $\bW_k$. Each candidate update is scored on the fixed full objective. Batch size one uses 32 probes; other finite batches use 16. Full batch uses its actual deterministic increment. The recorded environment is Python 3.12.3, PyTorch 2.7.0+cu128, CUDA 12.8, and an NVIDIA GeForce RTX 4090; model computations use float32 with strict determinism enabled. Inner products multiply and accumulate in float64.

The compact SVD retains singular values above the tolerance
$\max(m,n)\epsilon_{\mathrm{fp32}}\sigma_{\max}$.
The resulting direction may differ from the exact polar factor.
Write $\widetilde{\bm P}_j(\bm G_j)$ for this numerical direction.
The diagnostics use the measured full-objective amplitude
\begin{equation*}
A_k=\sum_j\langle \bm{G}_{j,k},\widetilde{\bm{P}}_j(\bm{G}_{j,k})\rangle_{\mathrm{F}}
\end{equation*}
in place of the exact dual norm $N_k=\sum_j\|\bm{G}_{j,k}\|_*$. For probe $i$, write $d_i=L(\bm{W}_k-\eta\widetilde{\bm{P}}_{B_i})-L(\bm{W}_k)$ and $a_i=\langle \bm{G}_k,\widetilde{\bm{P}}_{B_i}\rangle$. The plotted quantities are
\begin{equation*}
\hat\rho_b=\frac{\overline a}{A_k},\qquad
\hat s_b^{\mathrm M}=\frac{2(\overline d+\eta\overline a)}{\eta^2A_k}.
\end{equation*}
Hence full-batch $\hat\rho_n=1$ holds by the recorded normalization. The discrepancy $|A_k-N_k|/N_k$ is logged and reaches approximately $0.29\%$ in current CIFAR runs. The hats therefore denote numerical estimators under this convention, not exact evaluations of the ideal single-matrix theory. The general loss identity remains exact for the implemented direction and its measured alignment.

For several parameter matrices, the exact theory extends using the product norm $\max_j\|\bm{P}_j\|_{\mathrm{op}}$ and dual norm $\sum_j\|\bm{G}_j\|_*$. Loss evaluations use the complete simultaneous update and therefore include cross-block interactions. The recorded global cosine is
\begin{equation*}
c^{\mathrm{global}}_k=
\frac{\sum_j\langle \bm{P}_{j,k},\bm{P}_{j,k+1}\rangle_{\mathrm{F}}}
{\sqrt{\sum_j\|\bm{P}_{j,k}\|_F^2}\sqrt{\sum_j\|\bm{P}_{j,k+1}\|_F^2}},
\end{equation*}
including all trainable tensors. The full-batch layer curves show selected weight matrices; Transformer block aggregates include all tensors in the corresponding block. A cosine paired with update index $k$ compares the directions used at $k$ and $k+1$; the code also evaluates the next direction at the terminal state, so all 1500 rows contain a cosine.

\subsection{Event rules and uncertainty}\label{app:events}
We report the first detected T1 and T2 events, without persistence filtering.

T1 uses the conditional loss-increment estimate for finite batches and the realized fixed-objective increment at full batch. A qualifying rise exceeds $10^{-8}+10^{-7}|L(\bm{W}_k)|$. T2 requires $c_k^{\mathrm{pol}}<-10^{-7}$. All displayed bars use the first qualifying step in unthinned records (the saved window-one rule), without persistence filtering or interpolation. Thus the CNN full-batch loss/T1 event at step 1350 is a first detected event, not a claim about sustained onset.

The 95\% Monte Carlo intervals use the Student-$t$ quantile and paired probe increments. They describe pointwise finite-probe uncertainty, not across-seed variability or simultaneous confidence for a selected first event. Current and historical runs are separate protocols; historical sparse diagnostics are not filled in or joined to the current trajectories.

\noindent {\bf Event comparison.}
\cref{fig:small-network-onsets} compares the detected events.
Their times vary across diagnostics and batch sizes.
The matrix example in Section~\ref{sec:noncommuting-quadratic} proves
the theoretical separation. Unequal onset times alone do not prove it:
a scalar GD mode also reverses before reaching loss balance.

\subsection{Additional trajectory results}
\label{app:toc-5}
\label{app:figure-evidence}
\label{app:toc-6}

\noindent {\bf Batch-size comparisons.}
\cref{fig:network-batch-grid} combines fixed-objective loss, loss balance, and global {update alignment}. Each row compares the same four batches within one model. 
\begin{figure}[!t]
\centering
\subfloat[MLP: loss\label{fig:network-batch-grid-panel-a}]{\includegraphics[width=.32\linewidth]{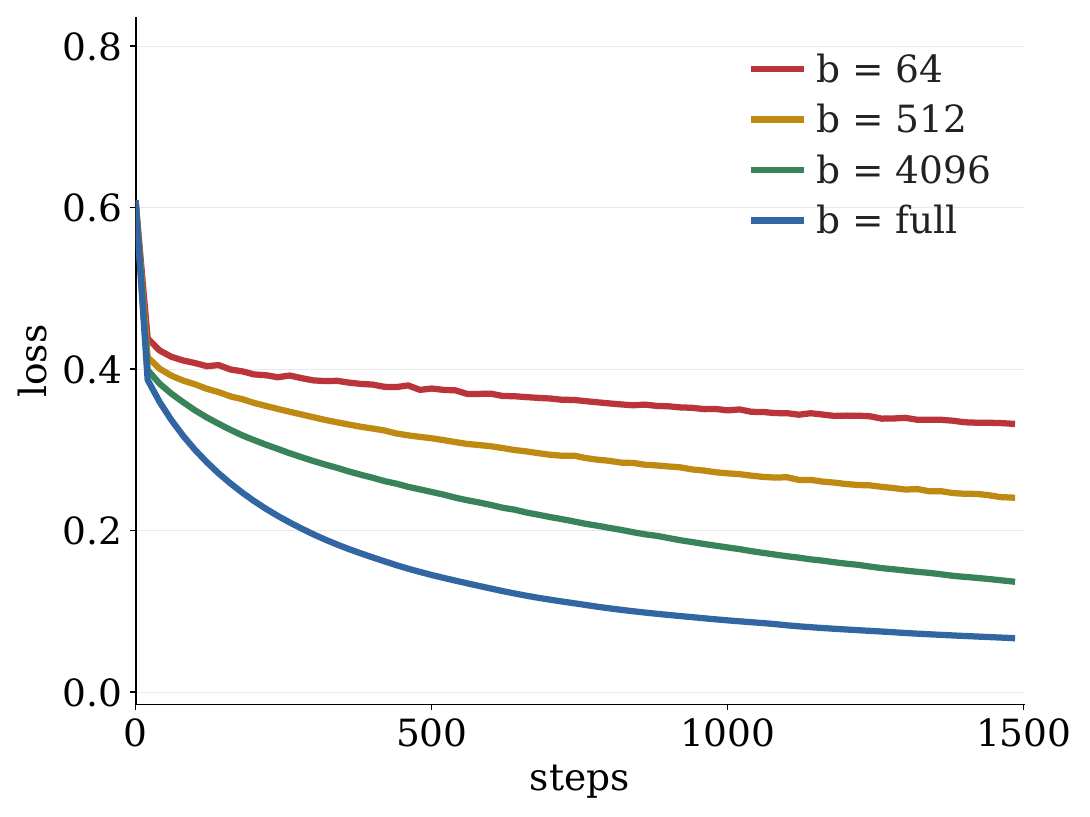}}\hfill
\subfloat[MLP: T1\label{fig:network-batch-grid-panel-b}]{\includegraphics[width=.32\linewidth]{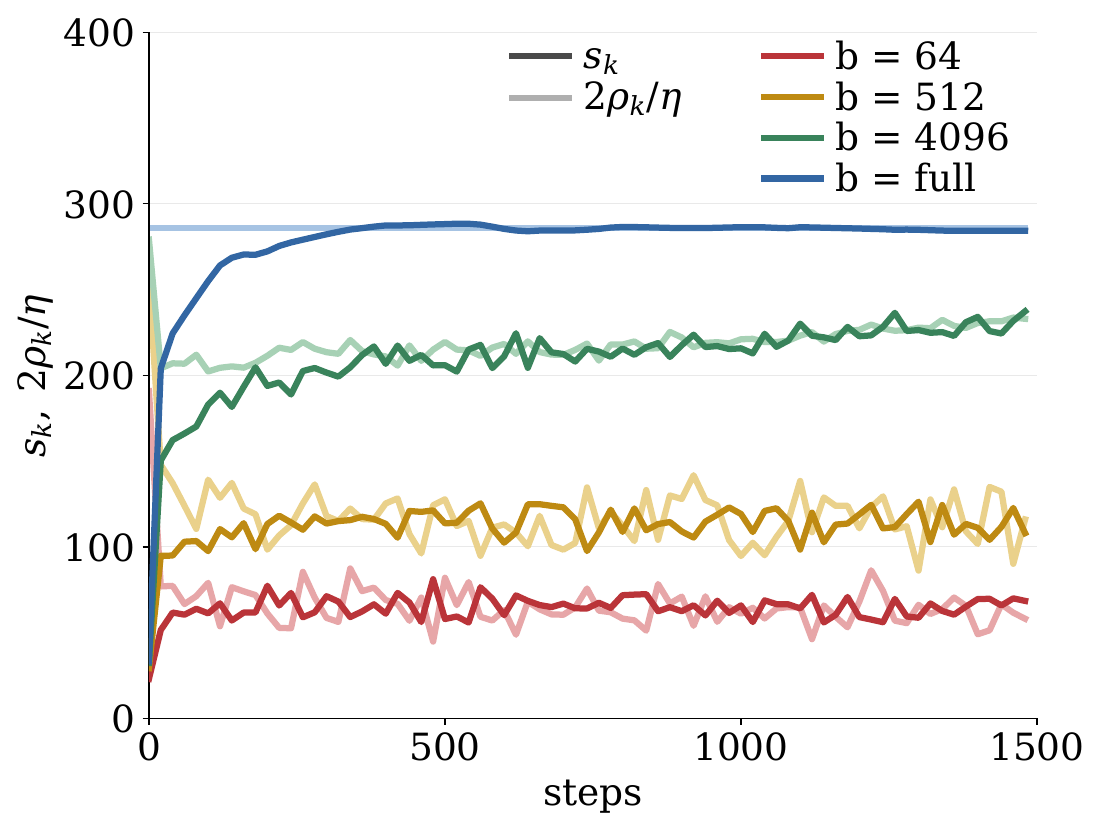}}\hfill
\subfloat[MLP: T2\label{fig:network-batch-grid-panel-c}]{\includegraphics[width=.32\linewidth]{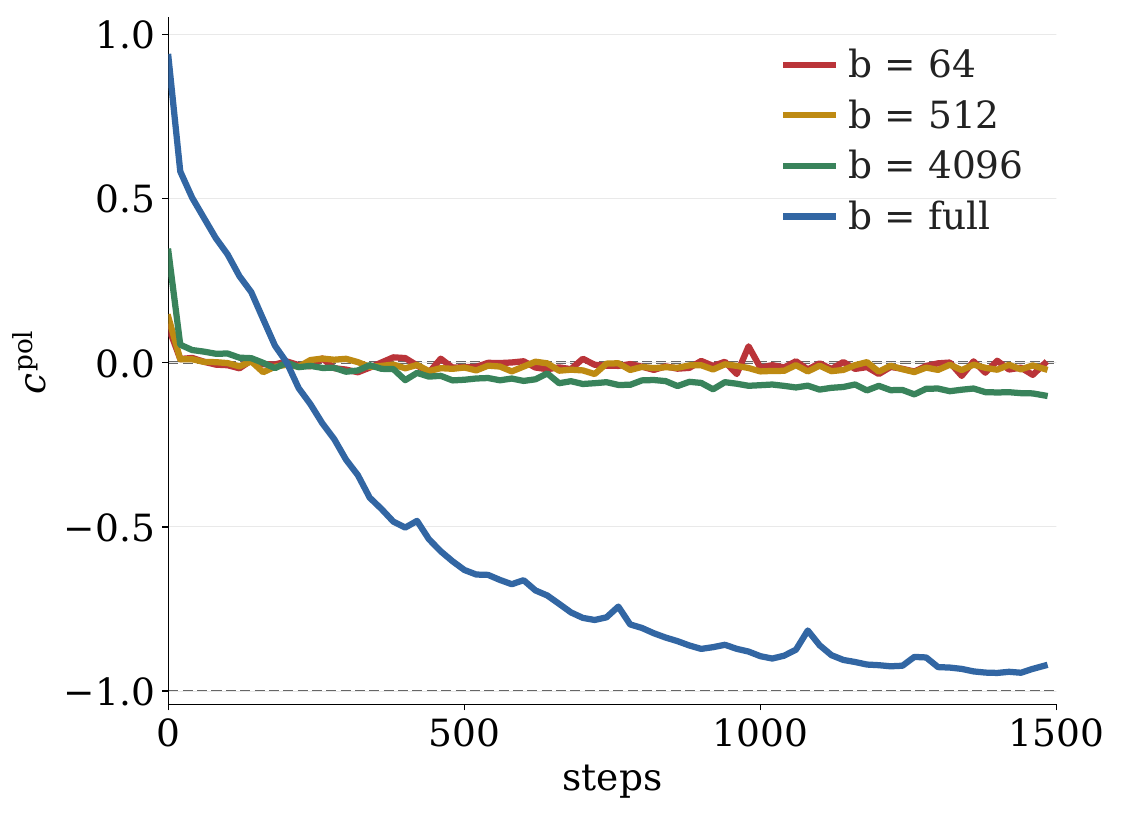}}
\par\smallskip
\subfloat[CNN: loss\label{fig:network-batch-grid-panel-d}]{\includegraphics[width=.32\linewidth]{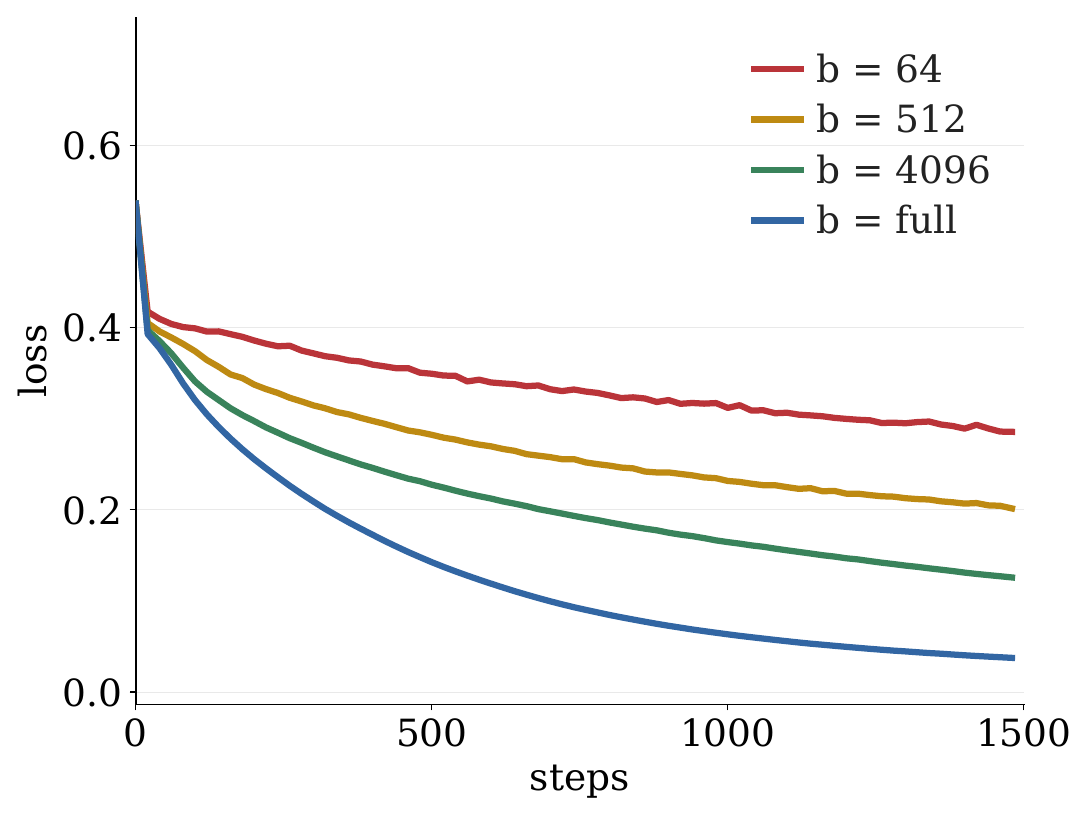}}\hfill
\subfloat[CNN: T1\label{fig:network-batch-grid-panel-e}]{\includegraphics[width=.32\linewidth]{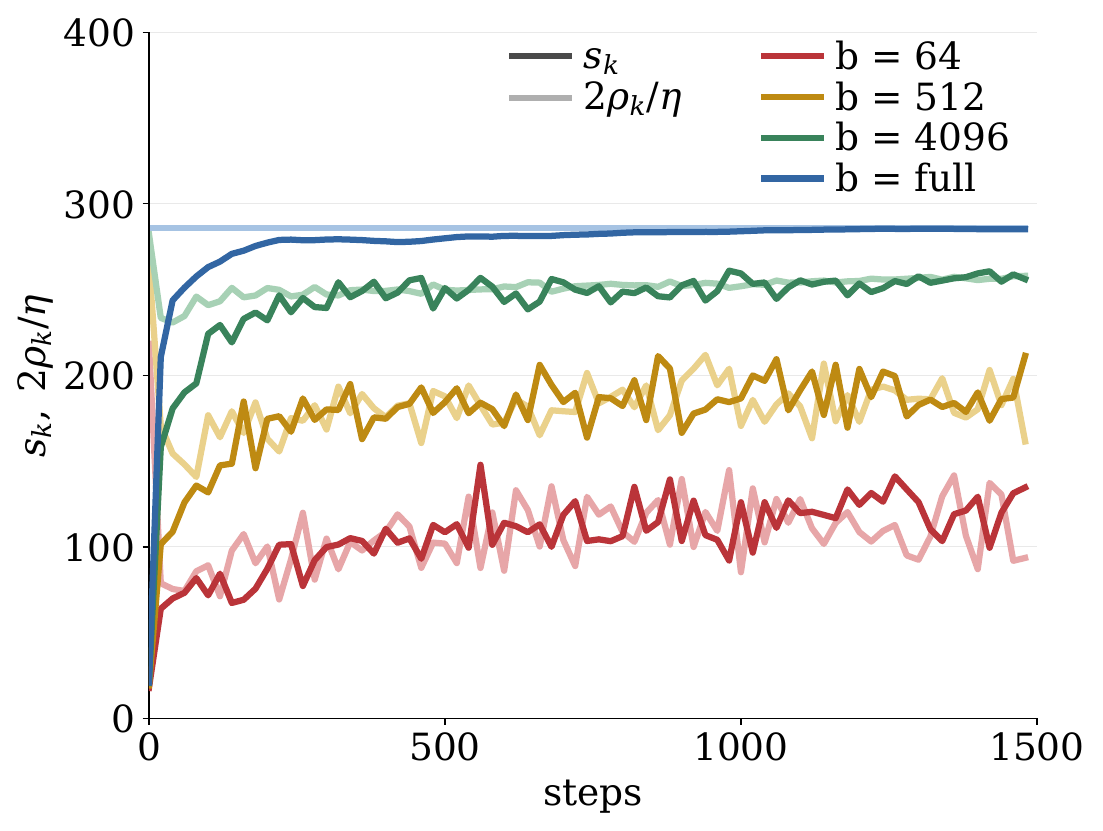}}\hfill
\subfloat[CNN: T2\label{fig:network-batch-grid-panel-f}]{\includegraphics[width=.32\linewidth]{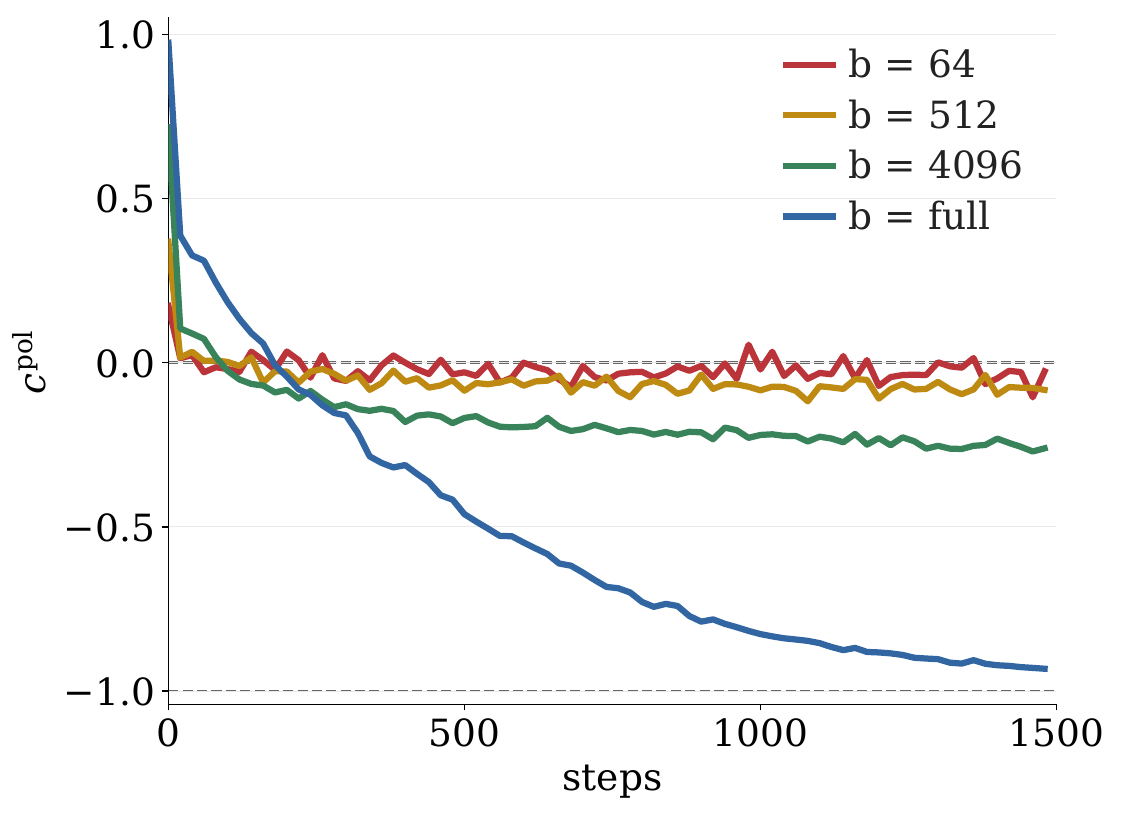}}
\par\smallskip
\subfloat[Transformer: loss\label{fig:network-batch-grid-panel-g}]{\includegraphics[width=.32\linewidth]{figure/neural_network/appendix/01_batch_comparison/transformer_loss.pdf}}\hfill
\subfloat[Transformer: T1\label{fig:network-batch-grid-panel-h}]{\includegraphics[width=.32\linewidth]{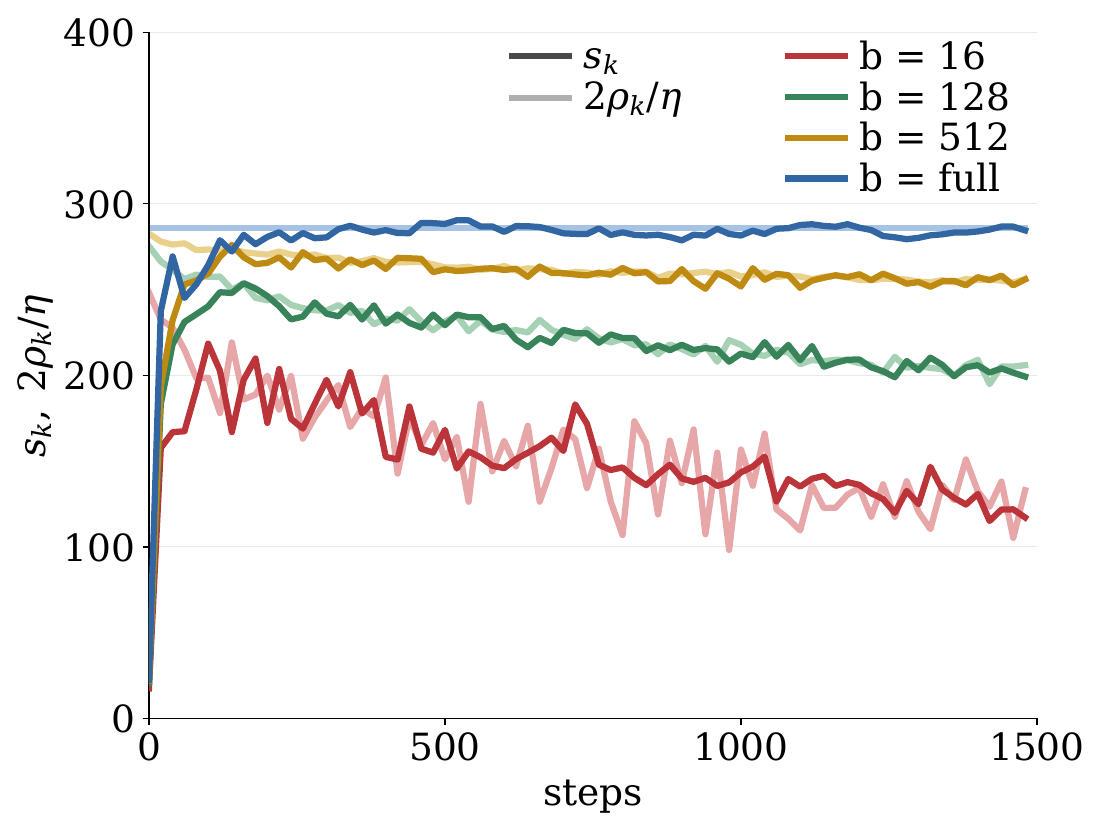}}\hfill
\subfloat[Transformer: T2\label{fig:network-batch-grid-panel-i}]{\includegraphics[width=.32\linewidth]{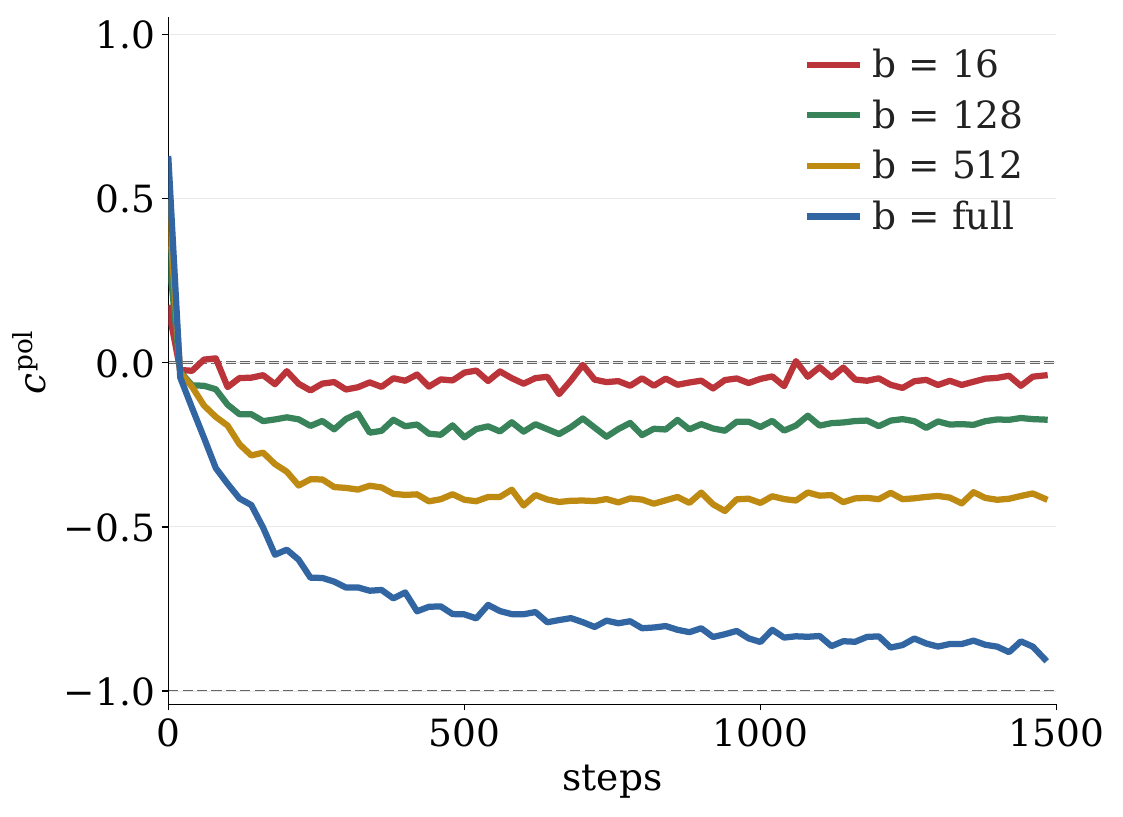}}
\caption{\textbf{Batch-dependent loss and direction diagnostics.} These experiments use the MLP, CNN, and Transformer at $\eta=0.007$; they measure full-objective loss, T1, and T2 across four batches. Rows correspond to MLP, CNN, and Transformer, while columns show loss, T1, and T2. T1 pairs dark curvature curves with light boundaries; curves connect saved observations.}
\label{fig:network-batch-grid}
\end{figure}

\noindent {\bf Full-batch layer and block directions.}
\cref{fig:network-layer-grid} separates layer/block trajectories from the global statistic. Only full-batch training is displayed, using curves throughout. The Transformer panels report the two saved block aggregates, rather than individual Q/K/V matrices.
\begin{figure}[!t]
\centering
\subfloat[MLP: layers/blocks\label{fig:network-layer-grid-panel-a}]{\includegraphics[width=.32\linewidth]{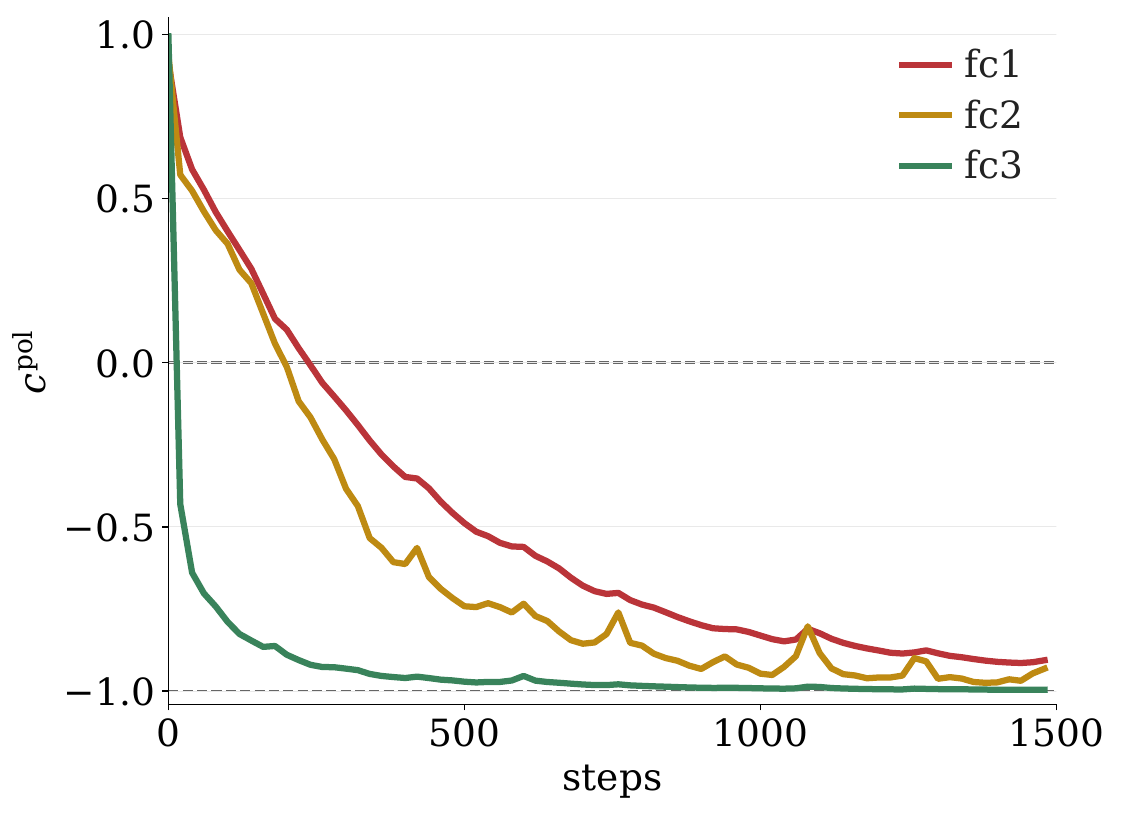}}\hfill
\subfloat[CNN: layers/blocks\label{fig:network-layer-grid-panel-b}]{\includegraphics[width=.32\linewidth]{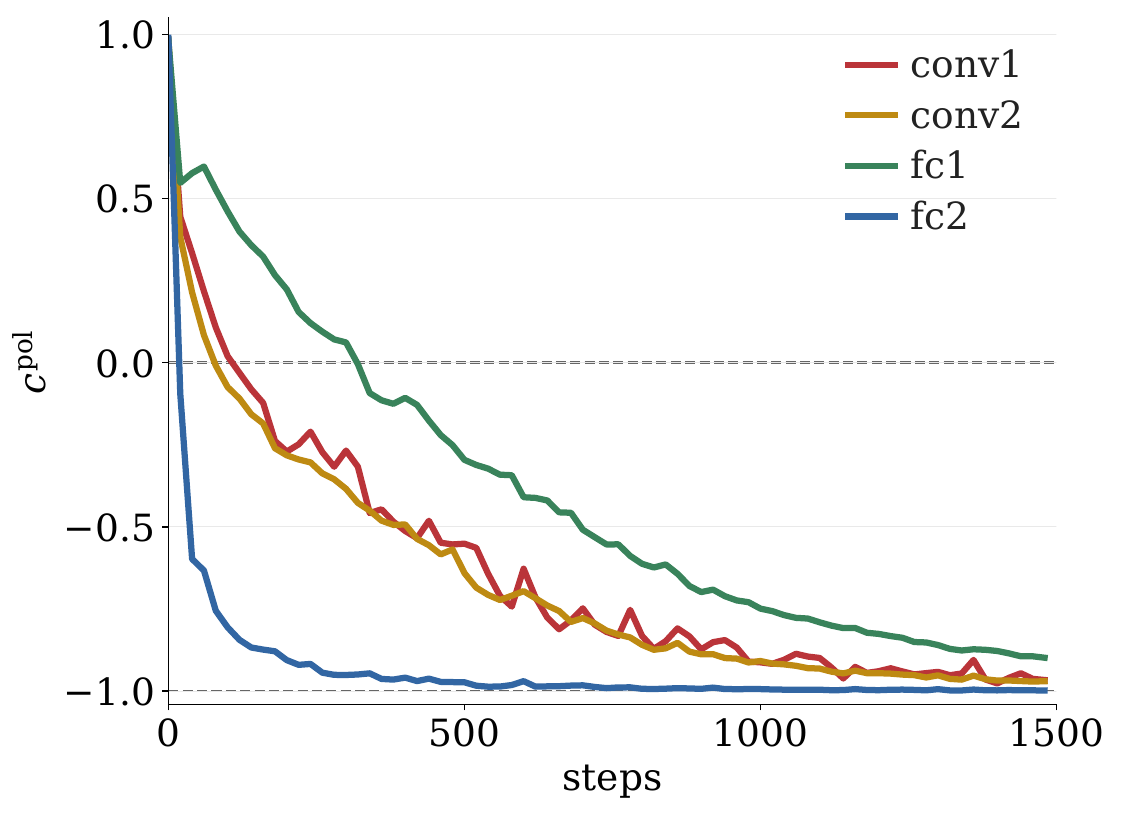}}\hfill
\subfloat[Transformer: layers/blocks\label{fig:network-layer-grid-panel-c}]{\includegraphics[width=.32\linewidth]{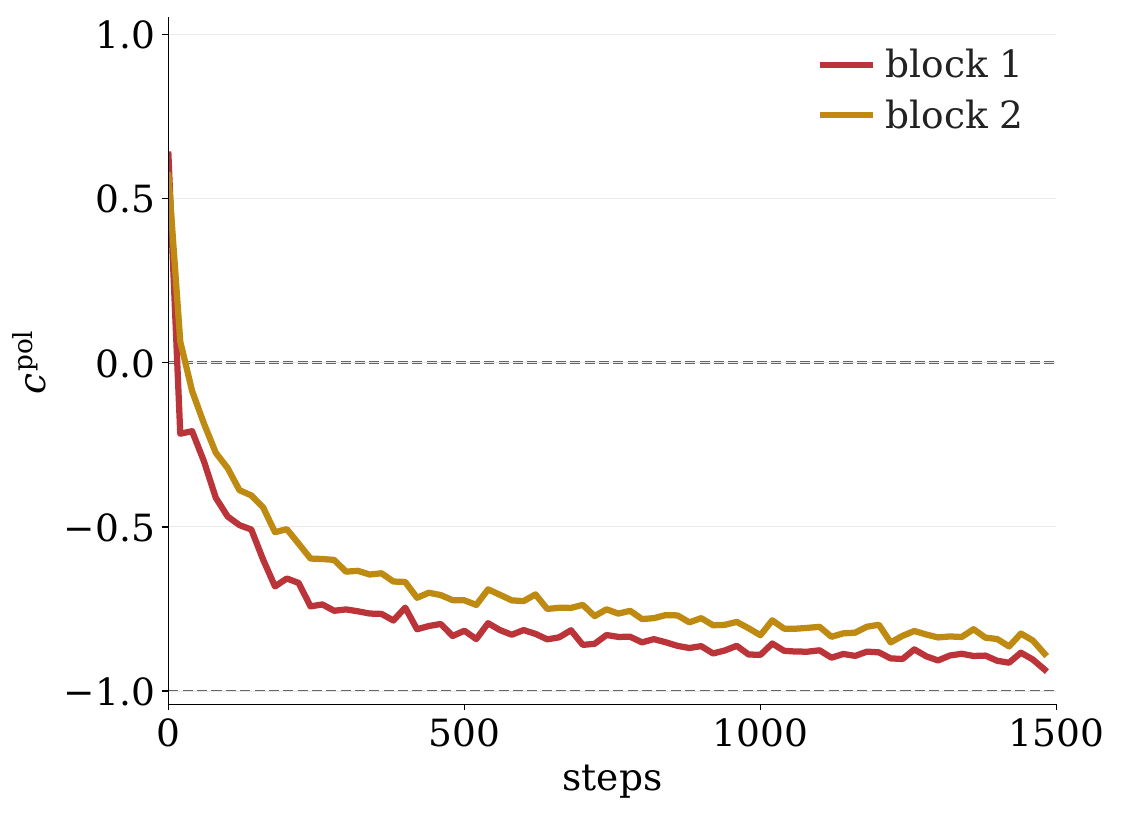}}
\par\smallskip
\subfloat[MLP: global\label{fig:network-layer-grid-panel-d}]{\includegraphics[width=.32\linewidth]{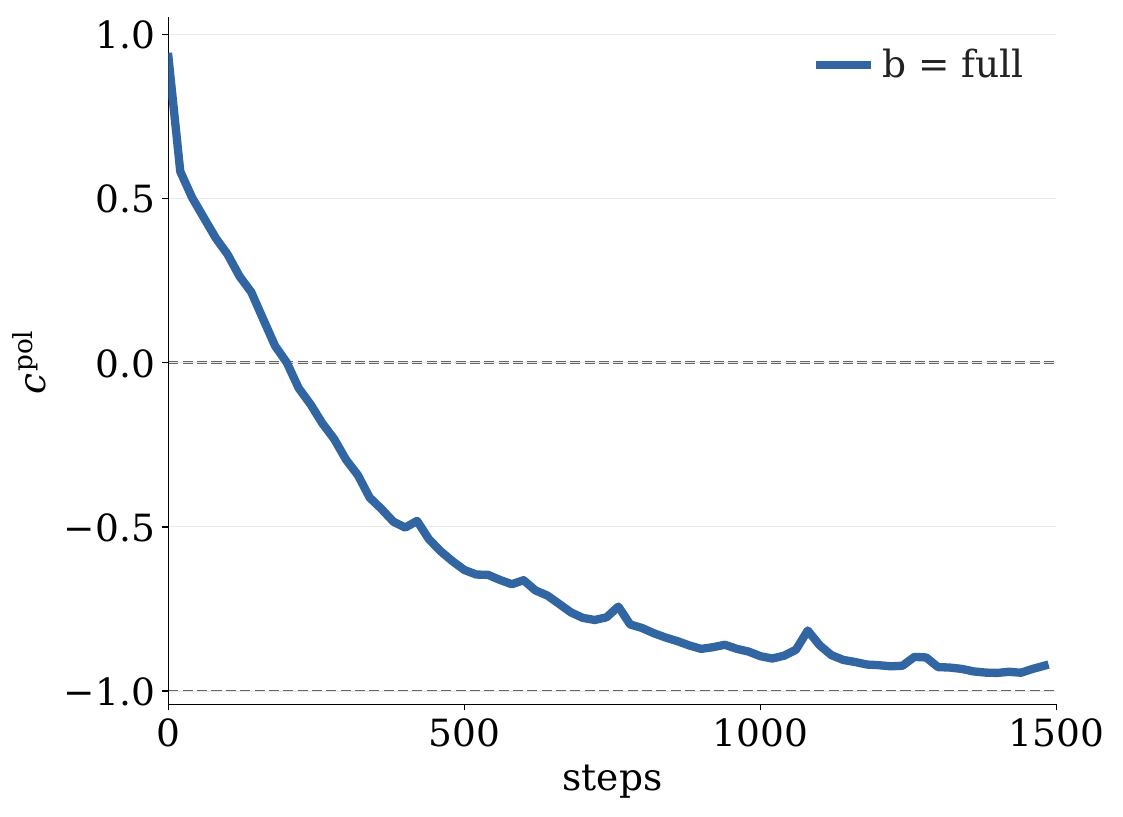}}\hfill
\subfloat[CNN: global\label{fig:network-layer-grid-panel-e}]{\includegraphics[width=.32\linewidth]{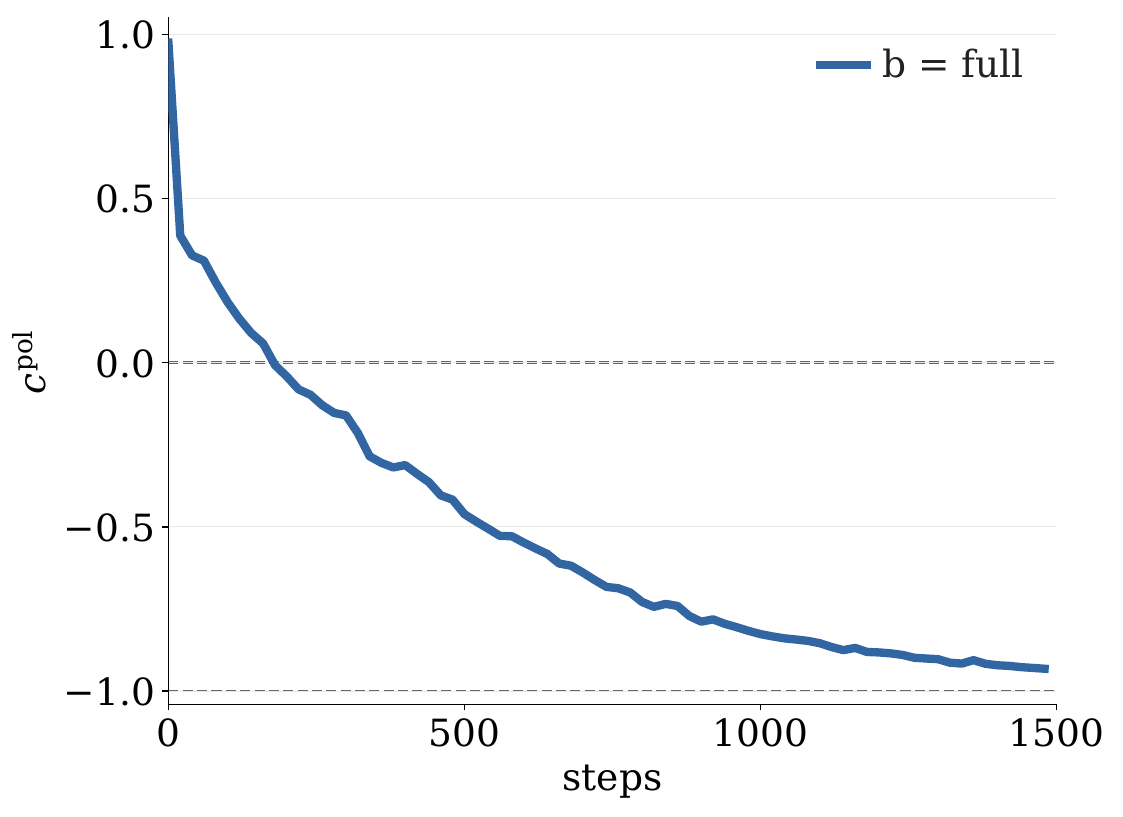}}\hfill
\subfloat[Transformer: global\label{fig:network-layer-grid-panel-f}]{\includegraphics[width=.32\linewidth]{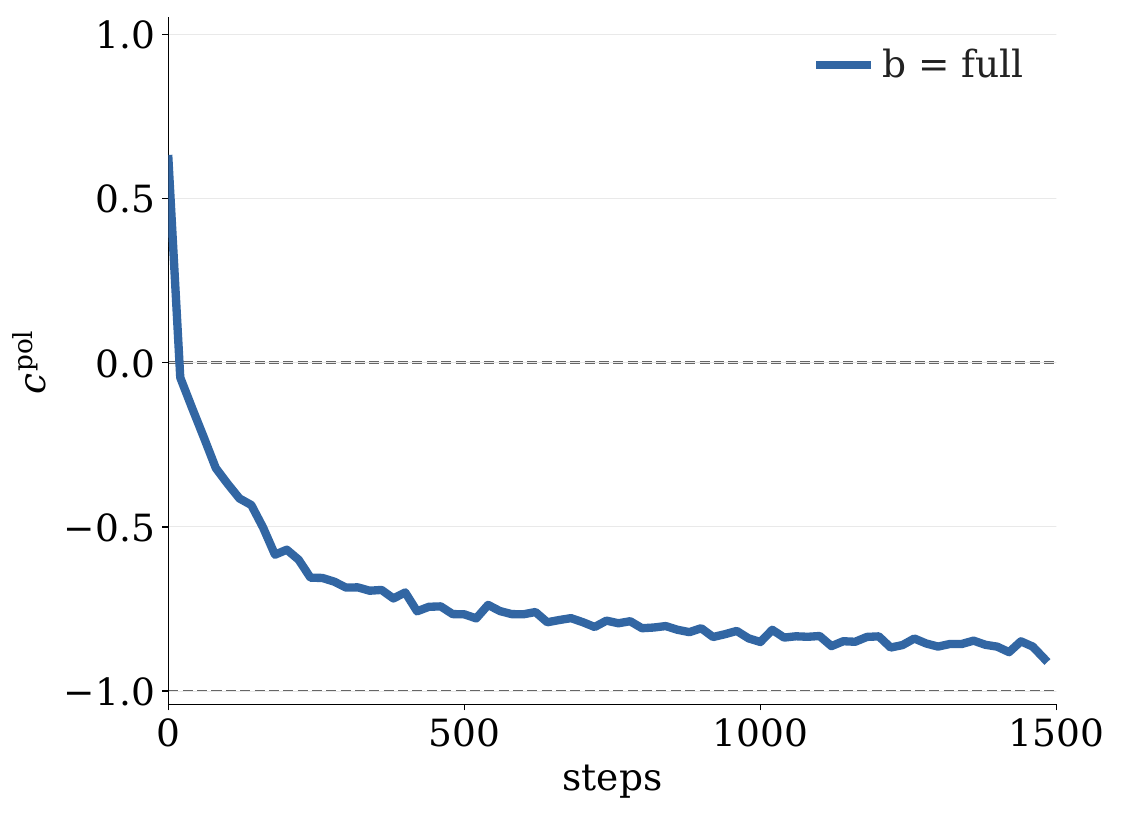}}
\caption{\textbf{Layerwise and global {update alignment}.} These experiments use full-batch MLP, CNN, and Transformer models at $\eta=0.007$; they measure block-level and global update-direction cosines. Columns show MLP, CNN, and Transformer. The top row contains layer or block trajectories, and the bottom row contains global cosines computed from aggregated inner products and norms.}
\label{fig:network-layer-grid}
\end{figure}

\subsection{Frozen-state and {update alignment}}
\label{app:fixedcp-protocol}
\label{subsec:fixedcp-noise}
In our experiment shown in \cref{fig:introduction-muon} and \cref{fig:controlled-diagnostics}, $c_k^{\mathrm{pol}}$ is quite close to zero.
However, unrelated directions in a high-dimensional space are already likely to be nearly orthogonal. A near-zero cosine alone therefore does not tell us whether the weak alignment comes from Muon's training dynamics or simply from minibatch randomness.

To distinguish these explanations, we compare two measurements. At selected
checkpoints spanning the warmup, stable, and decay stages of the WSD schedule,
we freeze the model state, draw independent minibatches, and compute the
corresponding Muon directions. Their pairwise cosines measure the variation
caused by minibatch sampling alone. We then compare them with the cosines
between consecutive directions along the actual training trajectory, where
each update changes the state used to compute the next direction.

In both seeds, the frozen-state mean cosines are small but positive. By
contrast, during the stable stage, the {update-direction cosine means} are approximately
$-0.031$ and $-0.032$, and more than $99\%$ of consecutive pairs are negative
(\cref{fig:fixedcp-noise}). Thus, minibatch dispersion can account for much of
the near-orthogonality, but it cannot explain the systematic negative bias
along training. The sign difference instead supports a feedback effect caused
by the evolving parameter state. Because the training trajectory jointly
updates the Muon and auxiliary AdamW parameter groups, this experiment does
not attribute the feedback to either optimizer alone.
The protocol follows below.

\begin{figure}[!t]
\centering
\subfloat[Seed 1337\label{fig:fixedcp-1337}]{
\includegraphics[width=0.49\linewidth]{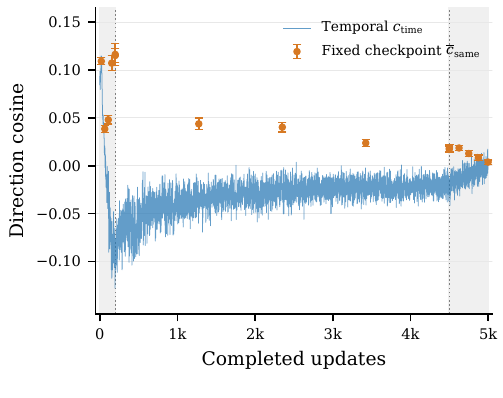}}\hfill
\subfloat[Seed 2026\label{fig:fixedcp-2026}]{
\includegraphics[width=0.49\linewidth]{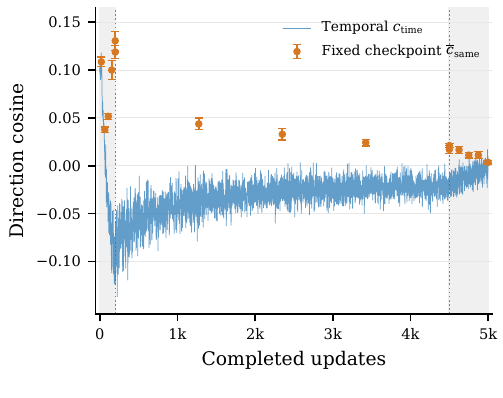}}
\caption{
Frozen-state and {update alignment} in a Tiny Transformer
trained with no-momentum NS-5 Muon at batch size eight.
Orange markers estimate $c_k^{\mathrm{same}}$ from all 120 unordered
pairs of 16 independent probe batches while holding the model state fixed;
bars show 95\% delete-one-batch jackknife intervals.
Blue curves show the realized consecutive-step statistic
$c_k^{\mathrm{time}}$. Both statistics use only the directions on the Muon parameter subspace. Dashed and dotted vertical lines mark zero and the
warmup--stable--decay boundaries, respectively.
}
\label{fig:fixedcp-noise}
\end{figure}
\noindent {\bf Direction statistics.}
Let $\widetilde{\bP}_k(B)$ collect the NS-5 Muon directions from batch $B$ at state $(\bW_k,\bZ_k)$, where $\bZ_k$ contains auxiliary parameters. Normalize them jointly as $\bU_k(B)=\widetilde{\bP}_k(B)/\|\widetilde{\bP}_k(B)\|_F$. For independent batches $B,B'$,
\begin{equation}
\begin{aligned}
c_k^{\mathrm{same}}
&:=\E_{B,B'}\!\left[\ipF{\bU_k(B)}{\bU_k(B')}\mid\bW_k,\bZ_k\right]\\
&=\left\|\E_B[\bU_k(B)\mid\bW_k,\bZ_k]\right\|_F^2\geq0.
\end{aligned}
\label{eq:fixed-state-null}
\end{equation}
Independent-batch dispersion can therefore produce near-orthogonality, but not a negative population-mean cosine at a fixed state. Finite-sample estimates can be negative. During training, we instead measure
\begin{equation}
c_k^{\mathrm{time}}:=\ipF{\bU_{k-1}(B_{k-1})}{\bU_k(B_k)}.
\label{eq:temporal-cosine}
\end{equation}
The first update changes the state used to compute the next direction.

\noindent {\bf Model and data.}
The decoder-only Transformer has four layers, width 96, four heads, FFN width 288, RMSNorm, rotary embeddings, a gated SiLU MLP, and tied token/output embeddings. It has 28 Muon matrices, vocabulary size 50,304, and context length 256.

The GPT-2-tokenized FineWeb-Edu cache~\citep{penedo2024fineweb}, \texttt{fineweb\_edu\_gpt2\_6b}, contains 6,000,004,097 training tokens and disjoint validation/test regions of 5,000,000 tokens each. Training windows are sampled with replacement. Each seed (1337 and 2026) runs 5,000 updates with batch eight and microbatch two: 2,048 tokens per update and 10,240,000 tokens in total.

\noindent {\bf Optimization.}
Muon uses NS-5, zero momentum, zero weight decay, and unit update scale. Gradients accumulate before NS-5. Training uses bfloat16; NS-5 uses float16; updates use float32 directions. Muon and auxiliary AdamW peak rates are 0.034 and $3\times10^{-4}$. AdamW uses betas $(0.9,0.95)$, $\epsilon=10^{-8}$, and weight decay 0.1 on matrices and zero on normalization vectors.

Both schedules use 200 warmup updates, a constant stage over $200\leq k<4500$, and 500 linear-decay updates ending at 10\% of peak rate. The environment is Python 3.12.3, PyTorch 2.7.0+cu128, and an RTX 4090.

\noindent {\bf Frozen-state probes.}
The index $k$ counts completed updates. Checkpoints are $20,65,110,154,199$ in warmup; $200,1275,2350,3424,4499$ in the constant stage; and $4500,4625,4750,4874,4999$ in decay. At each checkpoint, 16 independent batches of eight sequences give 120 direction pairs. All model parameters, including auxiliary parameters, remain fixed. Parameter/buffer digests are checked before and after probing; gradients are cleared and training random states restored.

We use the statistics defined above on the Muon subspace. Global and group cosines aggregate inner products and squared norms before normalization. Saved per-matrix Gram matrices and window starts allow reconstruction of pairwise statistics.

\noindent {\bf Uncertainty and indexing.}
Pairs share batches. We therefore use a delete-one-batch jackknife. Let $\hat\mu$ be the pair mean, $\hat\mu_{(-i)}$ exclude batch $i$, and $\overline\mu_{(-)}$ average these estimates. Then
\[
\widehat{\mathrm{se}}_{\rm J}^{\,2}
=\frac{15}{16}\sum_{i=1}^{16}(\hat\mu_{(-i)}-\overline\mu_{(-)})^2,
\qquad
\hat\mu\pm t_{0.975,15}\widehat{\mathrm{se}}_{\rm J}.
\]
These approximate pointwise intervals can be unreliable for a nearly degenerate pair statistic. Raw outputs also retain bounded-differences intervals. Neither interval measures across-seed uncertainty.

{Update-direction cosines} are logged at every update and indexed by the ending state, as in \cref{eq:temporal-cosine}. The initial incoming cosine is undefined. Constant-stage summaries use 4,300 pairs per seed over $200\leq k<4500$. \cref{fig:fixedcp-noise} shows unsmoothed values, separately by seed, without temporal-window intervals.

\noindent {\bf Results.}
Frozen-state global means range from $0.0172$ to $0.1153$ for seed 1337 and from $0.0170$ to $0.1189$ for seed 2026. {Update-direction cosine means} are $-0.0314$ and $-0.0324$; $99.91\%$ and $99.70\%$ of pairs are negative. All 28 matrices have negative {update-direction cosine means} in both seeds.

\label{app:fixedcp-null}
The fixed-state population mean is nonnegative by \cref{eq:fixed-state-null}, although finite estimates may be negative. The observed sign difference supports update feedback beyond independent-batch dispersion. Joint Muon and AdamW updates prevent assigning this feedback to either optimizer alone.

\section{Perturbation experiments}
\label{app:perturbations}
\subsection{22M setup and probes}
\label{app:perturb22m-protocol}
\noindent {\bf Model and data.}
The 22,518,016-parameter Transformer has 12 layers, width 256, eight heads, and FFN width 704. Muon updates 9,633,792 parameters; AdamW updates 12,884,224. It uses GPT-2-tokenized FineWeb \texttt{sample-10BT}, sequence length 4,096, microbatch four, and a fixed reference set of 32 sequences. Precision and NS-5 conventions match Appendix~\ref{app:130m-protocol}. Training, data, diagnostic, and evaluation seeds are 1337, 2026, 314159, and 271828.

\noindent {\bf Branches and schedule.}
All branches start at checkpoint 4,000, with reference loss $4.3578633$. The control uses $(\eta_{\rm peak},b)=(0.04,16)$; branches use $(0.02,16)$, $(0.08,16)$, $(0.04,8)$, and $(0.04,32)$. Auxiliary AdamW stays at peak rate $0.001$, betas $(0.8,0.999)$, weight decay $0.1$, and $\epsilon=10^{-8}$.

Training lasts 8,000 updates, with 400 warmup updates and 1,600 linear-decay updates ending at 10\% of peak rate. Branch rates are constant over $4000\leq k<6400$. A record indexed by $k$ describes the candidate update from state $k$; continuation records span 4000--7999.

\noindent {\bf Measurements.}
Conditional probes use eight draws for each batch $1,8,16,32$ every 500 updates from the branch point. The displayed curvature panels use probe batch size 16 for both pre-branch history and all branches, independently of training batch size. Probe estimators and pointwise Student-$t$ intervals are defined in Appendix~\ref{app:llm-estimators}. Cosines compare adjacent NS-5 directions and are logged every five updates. Fixed-reference loss and realized-update diagnostics are logged every 100 updates.

The last conditional, reference-loss, and cosine records occur at 7,500, 7,900, and 7,995. Curves connect recorded samples without smoothing or filling missing values. Validation uses 256 sequences every 500 updates and 4,096 sequences at completion.

\noindent {\bf Token budgets.}
Checkpoint 4,000 contains 262,144,000 processed tokens. After 4,000 further updates, batches 8, 16, and 32 reach 393,216,000, 524,288,000, and 786,432,000 tokens. Later batch-size comparisons therefore use different token budgets. These are single-seed branches from one checkpoint.

\subsection{{130M setup and probes}}
\label{app:perturb130m-protocol}
The 130M branches use the architecture in Appendix~\ref{app:130m-protocol}, no-momentum NS-5 Muon, and training batch size 32. At step 7,000, they change the Muon rate from $0.04$ to $0.02$, $0.04$, or $0.08$. The auxiliary AdamW schedule is unchanged across branches.

Each branch continues for 5,000 updates: 3,000 at constant rate, then 2,000 with linear decay to 10\% of peak rate. Validation uses 64 sequences every 1,000 updates. Both the 22M and 130M perturbation panels report global consecutive-direction cosines on the Muon parameter subspace. 

The 130M panels, \cref{fig:perturb130m-lr-panel-a,fig:perturb130m-lr-panel-b}, start at step 6,500. The pre-branch loss segment connects validations at steps 6,000 and 7,000; there is no validation measurement at step 6,500. These checkpoint branches are separate from the 130M runs with 40 tokens per parameter in Appendix~\ref{app:130m-protocol}.

\noindent \textbf{130M conditional probes.}
{Conditional probes use a fixed reference set of 128 sequences shared across branches. At each measured state, 32 independent batches of 32 sequences generate virtual updates.}

\begin{table}[t]
\centering
\caption{{Conditional loss responses at fixed checkpoints of the 22M and 130M models. }}
\label{tab:perturb22m-initial}
\begin{tabular}{@{}rrrrrl@{}}
\toprule
$\eta$ & $b$ & $\hat s^{\mathrm M}$ & $2\hat\rho/\eta$ & $\overline d$ & 95\% interval for $\overline d$\\
\midrule
\multicolumn{6}{@{}l}{{\textbf{22M: checkpoint 4,000}}}\\
0.02 & 16 & 17.182 & 34.453 & $-0.01967$ & $[-0.01999,-0.01936]$\\
0.04 & 16 & 17.169 & 17.227 & $-0.00026$ & $[-0.00138,\phantom{-}0.00086]$\\
0.08 & 16 & 16.964 & 8.613 & $\phantom{-}0.15219$ & $[\phantom{-}0.14650,\phantom{-}0.15789]$\\
0.04 & 8 & 13.830 & 14.575 & $-0.00340$ & $[-0.00512,-0.00167]$\\
0.04 & 32 & 20.321 & 19.406 & $\phantom{-}0.00417$ & $[\phantom{-}0.00349,\phantom{-}0.00485]$\\
\midrule
\multicolumn{6}{@{}l}{{\textbf{130M: checkpoint 7,000}}}\\
{0.02} & {32} & {14.990} & {30.172} & {$-0.01900$} & {$[-0.01908,-0.01892]$}\\
{0.04} & {32} & {14.664} & {15.086} & {$-0.00211$} & {$[-0.00248,-0.00175]$}\\
{0.08} & {32} & {14.207} & {7.543} & {$\phantom{-}0.13343$} & {$[\phantom{-}0.13102,\phantom{-}0.13583]$}\\
\bottomrule
\end{tabular}
\end{table}

\subsection{Loss and direction responses}
\label{app:perturb22m-evidence}

\noindent{\textbf{22M.}
Table~\ref{tab:perturb22m-initial} summarizes the fixed-checkpoint
conditional responses. Learning-rate changes reverse the conditional
loss sign while sampled update cosines remain negative
(\cref{fig:perturb22m-lr}); subsequent sparse probes indicate a return
toward loss balance without resolving its recovery time
(\cref{fig:perturb22m-lr-balance}).
Batch-size changes also affect loss balance and alignment differently
(\cref{fig:perturb22m-batch}): doubling the batch raises coherence by
$12.7\%$ but curvature by $18.4\%$, producing a positive conditional
increment despite the higher boundary.
These Muon-only one-step responses do not determine subsequent loss
under complete updates, and reference-loss samples every 100 updates
resolve transients rather than stepwise oscillations.\par}

\noindent{\textbf{130M.}
The original loss and cosine panels show the analogous alignment response and continued validation-loss improvement
(\cref{fig:perturb22m-lr}\subref{fig:perturb130m-lr-panel-a}--\subref{fig:perturb130m-lr-panel-b}).
The additional conditional probes confirm the fixed-checkpoint loss-sign change
(Table~\ref{tab:perturb22m-initial}) and recovery toward loss balance before decay
(\cref{fig:perturb130m-t1}), as summarized in Section~\ref{subsec:llm-perturbations}.
Complete Muon-plus-AdamW probes show the same sign change: mean increments are $-0.01832$ and $0.13732$ at rates $0.02$ and $0.08$, with 95\% intervals $[-0.01841,-0.01823]$ and $[0.13481,0.13983]$, conditional on the reference set and saved model and optimizer state.
For both models, later cosine changes also reflect learning-rate decay; the observed return does not establish convergence to a shared cosine value.\par}

\begin{figure}[!t]
\centering
\subfloat[Loss\label{fig:perturb22m-batch-loss}]{\includegraphics[width=.32\linewidth]{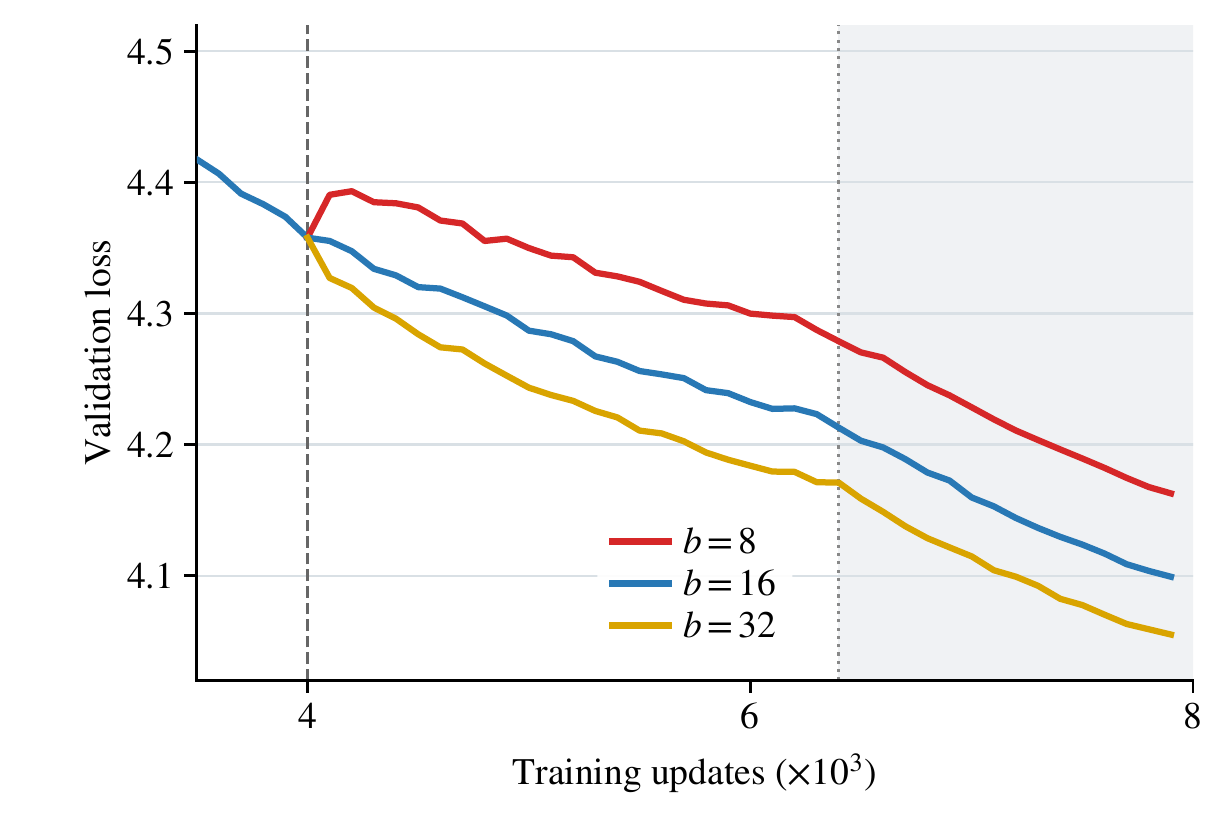}}
\hfill
\subfloat[Loss balance\label{fig:perturb22m-batch-balance}]{\includegraphics[width=.32\linewidth]{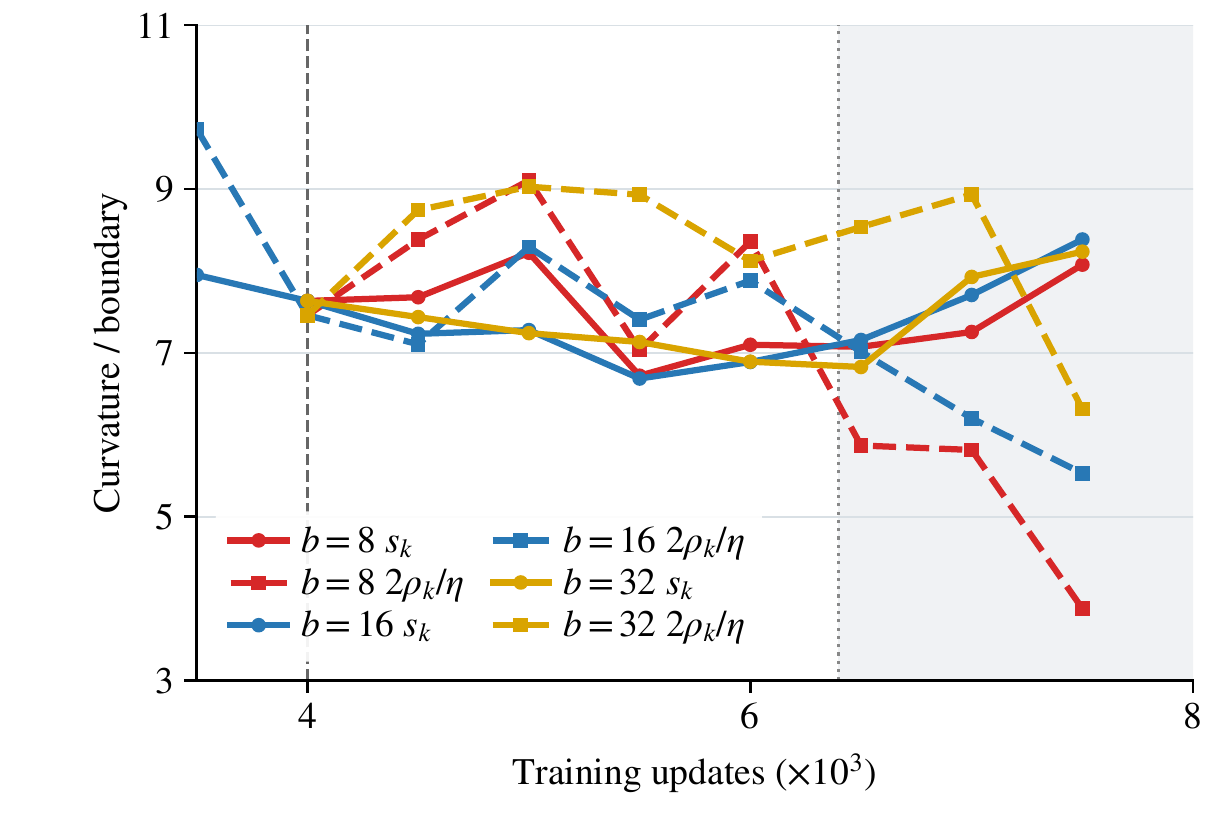}}
\hfill
\subfloat[{Update alignment}\label{fig:perturb22m-batch-cpol}]{\includegraphics[width=.32\linewidth]{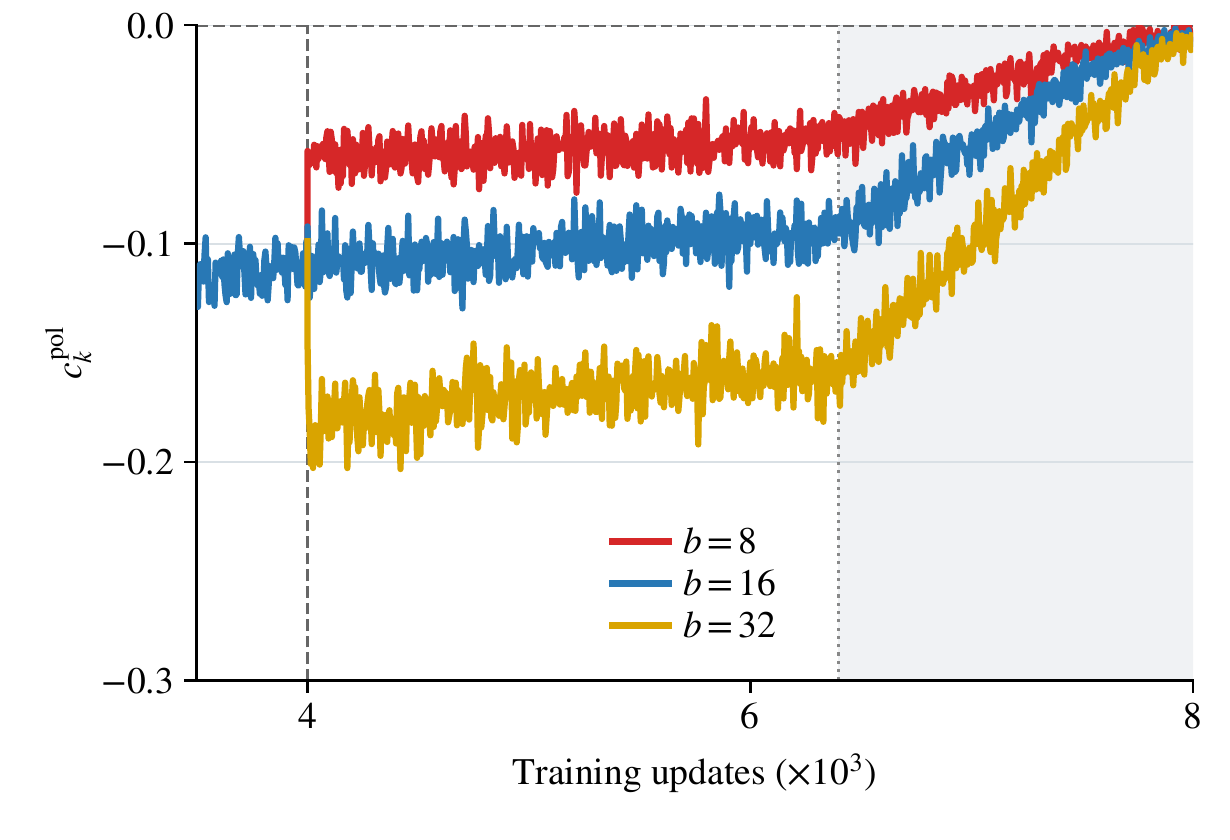}}
\caption{\textbf{Responses to checkpoint batch-size changes.} Panels show fixed-reference loss, conditional T1 curvature and boundary with probe batch size 16, and global consecutive-direction cosines. The displayed range is steps 3500--8000, with branching and schedule-change markers shown.}
\label{fig:perturb22m-batch}
\end{figure}

\begin{figure}[!t]
\centering
\subfloat[{Loss balance}\label{fig:perturb130m-t1-full}]{{\includegraphics[width=.25\linewidth]{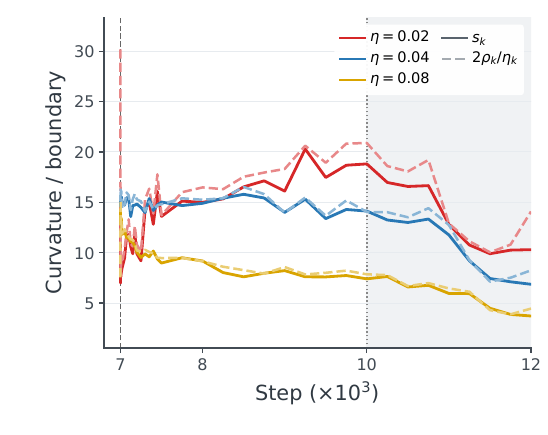}}}
\subfloat[{First 1,000 updates}\label{fig:perturb130m-t1-zoom}]{{\includegraphics[width=.25\linewidth]{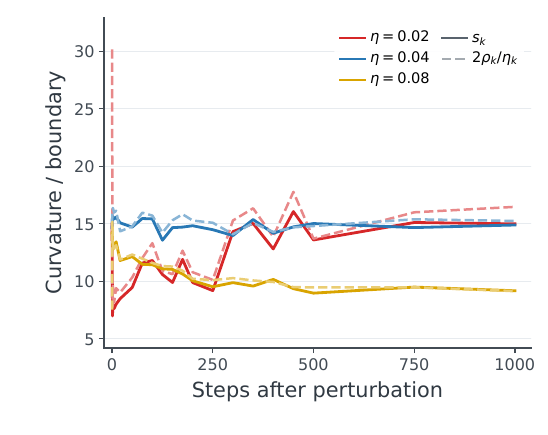}}}
\subfloat[{Relative discrepancy}\label{fig:perturb130m-r-full}]{{\includegraphics[width=.25\linewidth]{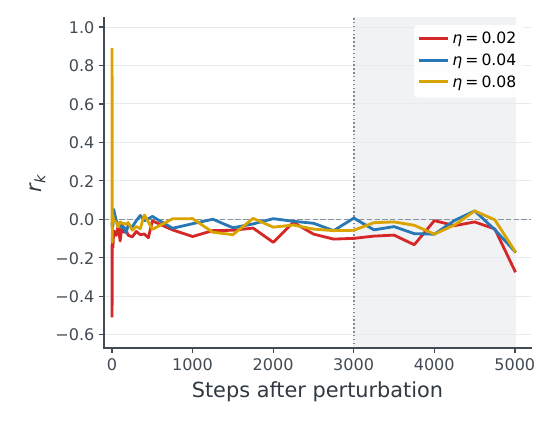}}}
\subfloat[{First 1,000 updates}\label{fig:perturb130m-r-zoom}]{{\includegraphics[width=.25\linewidth]{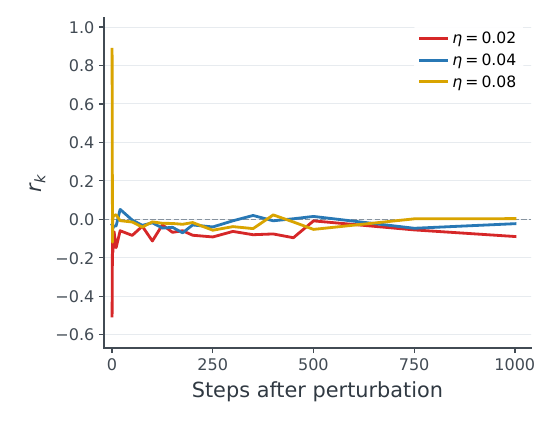}}}
\caption{{\textbf{Conditional loss balance after 130M learning-rate interventions.} (a) shows T1 after interventions; (c) shows $r_{k,b}$ (defined in \cref{eq:relative-boundary-discrepancy}) after interventions; (b) and (d) show the first 1000 updates.}}
\label{fig:perturb130m-t1}
\end{figure}

\section{Results on 130M and 1B pretraining}
\label{app:llm-settings}
\label{app:llm-figures}
\subsection{130M pre-training protocol}
\label{app:130m-protocol}
\noindent {\bf Model, data, and budget.}
The 130M Llama-like LLM has 12 layers, width 768, 12 heads, FFN width 2,304, and vocabulary size 50,304. Of 130,665,216 parameters, Muon updates 92,012,544 and AdamW updates 38,652,672.

GPT-2-tokenized FineWeb \texttt{sample-10BT} supplies six billion training tokens and 16,777,216 validation tokens from disjoint documents. Training windows are sampled with replacement. Batch size 32 and sequence length 4,096 give 131,072 tokens per update. The 39,876 updates process 5,226,627,072 tokens, about 40 per parameter. Microbatch size is two.

\begin{table}[t]
\centering
\caption{Recorded 130M pre-training configuration.}
\begin{tabular}{lp{0.60\linewidth}}
\toprule
Item & Setting\\
\midrule
Muon & NS-5; momentum 0; weight decay 0; update scale 1\\
Muon peak learning rates & 0.02, 0.04, 0.08, 0.12\\
Auxiliary optimizer & AdamW; peak LR 0.001; betas $(0.8,0.999)$; $\epsilon=10^{-8}$; weight decay 0.1\\
Schedule & Linear warmup fraction 0.05; constant stage; decay fraction 0.2; final LR fraction 0.1\\
Training / data seed & 1337 / 2026\\
Diagnostic / evaluation seed & 314159 / 271828\\
Training / probe precision & bfloat16 / float32; NS arithmetic float32\\
Reference objective & 32 sequences, 131,072 tokens\\
Conditional probes & Batch sizes 1, 8, 16, 32; eight draws; every 500 updates\\
Other diagnostics & Realized reference diagnostics every 100 updates; adjacent-direction cosine every 5 updates\\
Validation & Every 500 updates, 256 sequences; final evaluation 4096 sequences\\
Environment & Python 3.11.5; PyTorch 2.8.0+cu128; CUDA 12.8; NVIDIA A100-PCIE-40GB\\
\bottomrule
\end{tabular}
\end{table}

\subsection{Conditional probes and update scope}
\label{app:llm-estimators}
The 22M and 130M probes {described here} use no-momentum NS-5 with unit scale and zero Muon weight decay. Write $\bm W$ for Muon matrices and $\bm Z$ for auxiliary parameters. {Each Muon-only probe evaluates} $L_R(\bm W_k-\eta_k\widetilde{\bm P}_{k,b}^{(r)},\bm Z_k)$, holding $\bm Z_k$ fixed. Training instead updates both groups. {Additional probes compare conditional Muon-only and complete-update loss balance (\cref{fig:130m-reference-full}).}

\noindent {\bf Estimators.}
Write $\bm{G}_{R,j}=\nabla_{\bm{W}_j}L_R(\bm{W}_k,\bm{Z}_k)$ and $N_R=\sum_j\|\bm{G}_{R,j}\|_*$. For probe $r$, define
\[
 A_r=\sum_j\langle \bm{G}_{R,j},\widetilde{\bm{P}}_{k,b,j}^{(r)}\rangle_{\mathrm{F}},
 \qquad d_r=L_R(\bm{W}_k-\eta_k\widetilde{\bm{P}}_{k,b}^{(r)},\bm{Z}_k)-L_R(\bm{W}_k,\bm{Z}_k).
\]
The plotted conditional quantities are
\[
 \hat\rho_{k,b}=\frac{\overline A}{N_R},\qquad
 \hat s^{\mathrm M}_{k,b}=\frac{2(\overline d+\eta_k\overline A)}{\eta_k^2N_R}.
\]
The difference between the two T1 curves has the sign of $\overline d$.
This loss identity holds for the implemented directions. Exact-polar
properties, such as unit operator norm and full-batch coherence one,
are not assumed for NS-5.
The plotted $c_k^{\mathrm{pol}}$ is the Frobenius cosine between consecutive
implemented directions on the Muon parameter subspace. Their displacements
obey \cref{eq:two-step-displacement} , but cosine $-1$ alone need not give
an exact two-step return under NS-5.

The maximum recorded duality discrepancies are 20.8\%, 22.6\%, 27.9\%,
and 18.8\% in increasing learning-rate order. They measure first-order
pairing error relative to the nuclear norm, not direction error.
The \texttt{reference} probe uses a finite sample and NS-5, so it need
not have coherence one. Monte Carlo intervals exclude reference-sample
uncertainty.

\noindent {\bf Sampling and uncertainty.}
{
The original pretraining probes use eight draws per probe batch size; the additional 130M intervention and reference-set checks use 32.
Conditional loss increments have pointwise 95\% Student-$t$ intervals.
We quantify the relative discrepancy between curvature and boundary by
}
\begin{equation}
{
\hat r_{k,b}
=
\frac{\hat s^{\mathrm M}_{k,b}-2\hat\rho_{k,b}/\eta_k}
     {2\hat\rho_{k,b}/\eta_k}
=
\frac{\overline d}{\eta_k\overline A}.
}
\label{eq:relative-boundary-discrepancy}
\end{equation}
{
For positive mean alignment, a small $|\hat r_{k,b}|$ indicates
a small net loss change relative to the predicted first-order decrease.
Since curvature and boundary share $\overline A$, we compute
paired Fieller confidence sets:
}
\begin{equation}
{
\mathcal C_{k,b}^{0.95}
=
\left\{
q\in\mathbb R:
(\overline d-q\eta_k\overline A)^2
\leq
\frac{t_{0.975,m-1}^{\,2}}{m}
\left(v_d+q^2\eta_k^2v_A-2q\eta_k c_{dA}\right)
\right\}.
}
\label{eq:paired-fieller}
\end{equation}
{
Here $m=8$ for the original probes and $m=32$ for the additional 130M intervention and reference-set checks; $v_d$, $v_A$, and $c_{dA}$ are the sample
variances and covariance of the paired draws $(d_r,A_r)$,
computed with divisor $m-1$.
The covariance term accounts for the shared alignment estimate.
These approximate, pointwise 95\% sets assume bivariate-normal
paired draws and quantify probe uncertainty conditional on the
checkpoint and reference set. These pointwise conditional sets do not provide simultaneous coverage along the training trajectory or establish robustness across independent training seeds.
}

{Cosines compare consecutive Muon directions and are logged every five updates. The pair has lag one; five is the logging interval. Unrecorded pairs do not determine a first-crossing time.}

\subsection{{Normalized generalized sharpness}}
\label{app:generalized-sharpness}
\begingroup

\paragraph{Definition and reference objective.}
Let $\bm W=(W_j)_j$ denote the Muon matrix parameters and
hold the auxiliary parameters $\bm Z_k$ fixed at checkpoint $k$.
For a reference set $R$ of $n_R$ sequences of length $T$,
we use the mean next-token cross-entropy
\begin{equation}
L_R(\bm W,\bm Z)
=
\frac{1}{n_RT}
\sum_{u=1}^{n_R}\sum_{t=1}^{T}
\ell_{u,t}(\bm W,\bm Z).
\end{equation}
Writing
$H_R=\nabla_{\bm W}^2L_R(\bm W_k,\bm Z_k)$ and
$N_R=\sum_j\|\nabla_{W_j}L_R(\bm W_k,\bm Z_k)\|_{\mathrm{nuc}}$,
the multi-matrix extension of \cref{eq:intro-diagnostics} is
\begin{equation}
S_R(\bm W_k)
:=
\frac{
\displaystyle
\max_{\max_j\|D_j\|_{\mathrm{op}}=1}
\sum_j\langle D_j,[H_R[D]]_j\rangle_{\mathrm F}
}{
N_R
}.
\label{eq:reference-generalized-sharpness}
\end{equation}
The Hessian acts jointly on all Muon matrices, retaining
cross-matrix Hessian terms; it is not a block-diagonal
approximation. This diagnostic concerns the Muon parameter
subspace.

Reference sequences are sampled by drawing starting positions
uniformly with replacement from the training-token cache.
The sampled positions and token sequences are held fixed
across checkpoints. The displayed results use $n_R=2048$ sequences. To assess reference sensitivity, we additionally
evaluate three independently sampled reference sets at each size
$n_R\in\{1024,2048,4096\}$, holding each set fixed across checkpoints.
The observed generalized-sharpness pattern persists across
the tested reference sizes and replicates.

\paragraph{Numerical maximization.}
We estimate the numerator of
\cref{eq:reference-generalized-sharpness} using automatic-
differentiation Hessian--vector products and a multi-start
Frank--Wolfe-type search.
Define
\[
\|D\|_{\mathrm{blk}}=\max_j\|D_j\|_{\mathrm{op}},
\qquad
Q_R(D)=\sum_j\langle D_j,[H_R[D]]_j\rangle_{\mathrm F}.
\]
Each start initializes the matrix entries independently from
a standard Gaussian distribution and normalizes the resulting
direction to $\|D^{(0)}\|_{\mathrm{blk}}=1$.
At iteration $t$, we compute
\begin{equation}
Z^{(t)}=H_R[D^{(t)}],
\qquad
V_j^{(t)}
=
\operatorname{polar}_{\mathrm{SVD}}(Z_j^{(t)}),
\qquad
D^{(t+1)}
=
(1-\gamma_t)D^{(t)}+\gamma_tV^{(t)},
\quad
\gamma_t=\frac{2}{t+2}.
\end{equation}
Here, if $Z_j=U_j\Sigma_jV_j^\top$ is a thin SVD,
the oracle returns $U_jV_j^\top$, which maximizes the
linear objective over the matrix operator-norm unit ball.
The search iterates therefore remain in
$\|D\|_{\mathrm{blk}}\leq1$.

At recorded iterations, we evaluate the boundary-normalized
candidate
\begin{equation}
q_R(D)=\frac{Q_R(D)}{\|D\|_{\mathrm{blk}}^2}.
\end{equation}
The reported estimate $\widehat S_R$ is the largest recorded
candidate across starts, divided by $N_R$.
The standard budget is three starts and 50 iterations per
start, with random seed $20261113+r$ for start
$r=0,1,2$. Candidates are recorded every ten iterations,
at the midpoint, and at the initial and final iterates.
Termination uses the prescribed iteration budget, rather
than a convergence certificate.

The SVD-polar oracle optimizes a reference-Hessian quadratic
form and is distinct from the NS-5 directions used for
training and conditional T1 probes.

\paragraph{Numerical sensitivity and uncertainty.}
We also increase the search budget to 5 starts × 100 iterations, which changed the estimates by less than 1 percent at the checked checkpoints. This numerical stability supports the robustness of the estimates to the tested search budgets, but does not certify attainment of the global maximum.

\endgroup

\subsection{{130M loss balance: reference sets and complete updates}}
\label{app:toc-7}

\noindent {{\bf Relative boundary discrepancy.}}

{Table~\ref{tab:130m-relative-discrepancy} summarizes
$\hat r_{k,32}$ at 60 constant-learning-rate checkpoints
($k=2000,2500,\ldots,31500$), using eight draws per checkpoint.
Median curvature is 3.7--6.2\% below the boundary.
The entire paired 95\% interval lies within a 10\% relative
tolerance at 44--56 of the 60 checkpoints in each run.
All intervals are bounded and all mean alignments are
significantly positive at the pointwise 95\% level.
}

\begin{table}[t]
\centering
\caption{{
Relative boundary discrepancy in 130M pretraining.
Quantiles describe variation across checkpoints.
The last two columns count point estimates and entire paired
95\% intervals within $\pm10\%$, respectively.
}}
\label{tab:130m-relative-discrepancy}
\begingroup
\small
\setlength{\tabcolsep}{5pt}
\begin{tabular}{@{}ccccc@{}}
\toprule
Peak LR & Median $\hat r$ & 10--90\% quantiles
& $|\hat r|\leq0.1$ & CI within $\pm0.1$ \\
\midrule
0.02 & $-0.0366$ & $[-0.0829,\,0.0015]$ & $58/60$ & $56/60$ \\
0.04 & $-0.0428$ & $[-0.0874,\,0.0018]$ & $58/60$ & $53/60$ \\
0.08 & $-0.0619$ & $[-0.1069,-0.0076]$  & $53/60$ & $44/60$ \\
0.12 & $-0.0614$ & $[-0.1043,-0.0030]$  & $51/60$ & $44/60$ \\
\bottomrule
\end{tabular}
\endgroup
\end{table}

\noindent{\textbf{Reference-set robustness.}}
{An additional run at peak Muon rate $0.02$ uses batch 32, sequence length 4,096, and 60 probe checkpoints, including 42 in the constant-rate stage ($k=1993$--$31900$). We independently sample three fixed reference sets at each size 32, 64, and 128; these sets are not nested. Each checkpoint uses 32 draws. Across the 42 constant-rate checkpoints and three sets per size, median Muon-only $\hat r$ is $-0.0458$, $-0.0415$, and $-0.0427$, respectively (\cref{fig:130m-reference-sets}). Thus the relative loss-balance observation persists across the tested reference sizes and choices.}

\noindent{\textbf{Conditional complete updates.}}
{At each fixed parameter and optimizer state, let $\Delta\bm\theta_i$ be a virtual Muon-plus-AdamW displacement, and define}
\[
{d_i=L_R(\bm\theta+\Delta\bm\theta_i)-L_R(\bm\theta),\qquad
 a_i=-\langle\nabla L_R(\bm\theta),\Delta\bm\theta_i\rangle,\qquad
 \hat r^{\rm full}=\overline d/\overline a.}
\]
{Using the Muon reference-gradient nuclear-norm sum $N_R$, the complete-update curves are}
\[
{\widetilde s^{\rm full}=\frac{2(\overline d+\overline a)}{\eta_k^2N_R},\qquad
\frac{2\widetilde\rho^{\rm full}}{\eta_k}=\frac{2\overline a}{\eta_k^2N_R}.}
\]
{These are conditional means on the Muon normalization scale; $\widetilde\rho^{\rm full}$ is not polar coherence. The corresponding median $\hat r^{\rm full}$ values are $-0.0059$, $-0.0029$, and $-0.0036$ for reference sizes 32, 64, and 128 (\cref{fig:130m-reference-full}). Complete updates therefore also exhibit near-zero relative conditional increments during the constant-rate stage. }

\noindent{\textbf{Draw-count sensitivity.}}
{We compare the first 8, 16, and all 32 draws from the same samples (\cref{fig:130m-reference-draws}). Across constant-rate checkpoints and all nine reference sets, median absolute changes in Muon-only $\hat r$ relative to 32 draws are $0.0036$ and $0.0019$ for 8 and 16 draws. Median 95\% interval widths decrease from $0.0246$ to $0.0158$ and $0.0113$. We report 32 draws as primary; paired intervals follow Appendix~\ref{app:llm-estimators} and exclude reference-set and training-seed uncertainty.}

\begin{figure}[!t]
\centering
\subfloat[$\eta_{\rm peak}=0.02$\label{fig:llm-lr-loss-appendix-panel-a}]{\includegraphics[width=.24\linewidth]{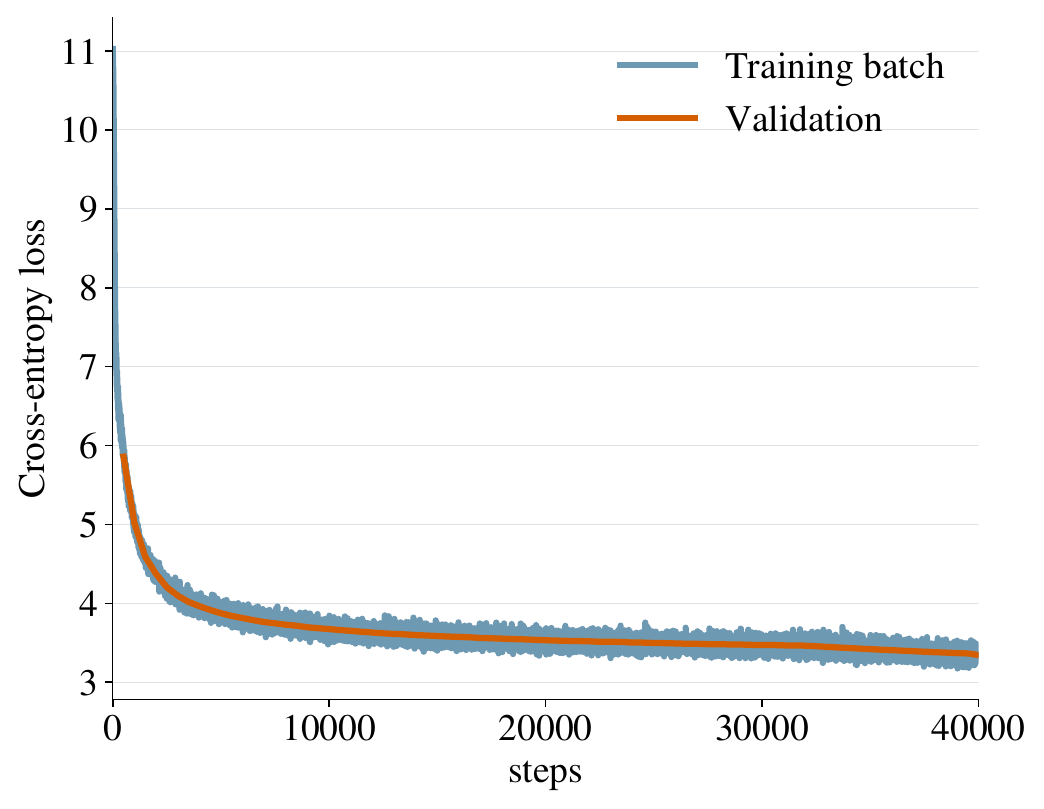}}\hfill
\subfloat[$\eta_{\rm peak}=0.04$\label{fig:llm-lr-loss-appendix-panel-b}]{\includegraphics[width=.24\linewidth]{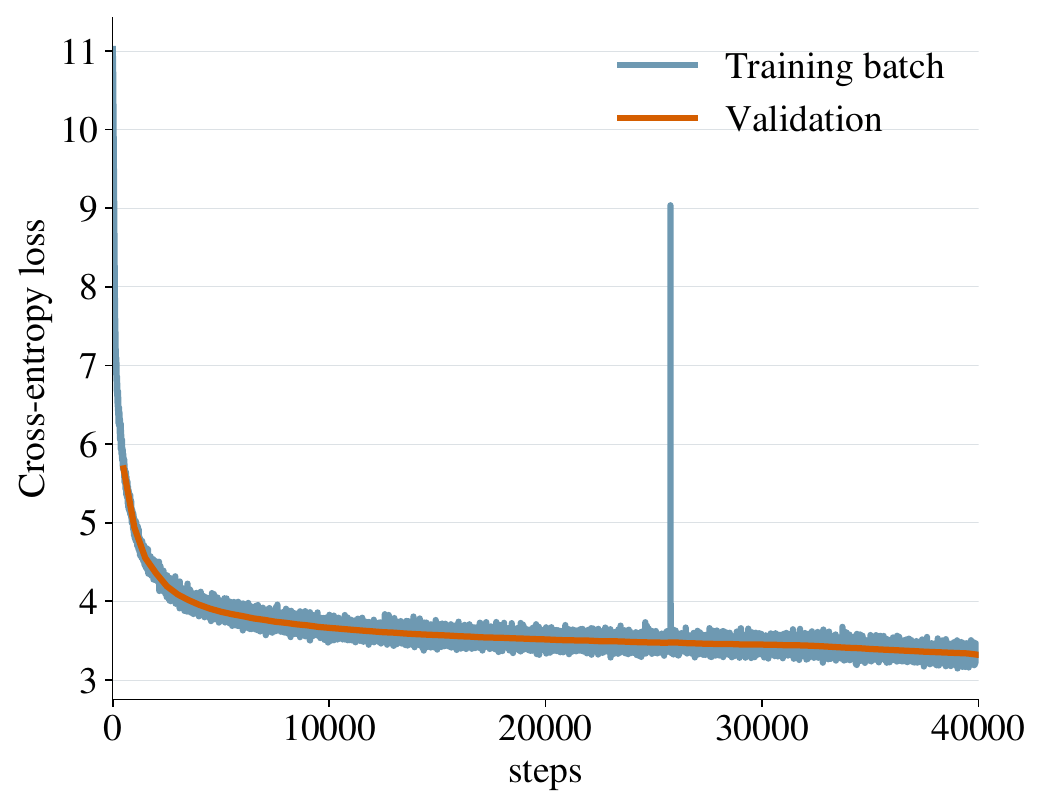}}\hfill
\subfloat[$\eta_{\rm peak}=0.08$\label{fig:llm-lr-loss-appendix-panel-c}]{\includegraphics[width=.24\linewidth]{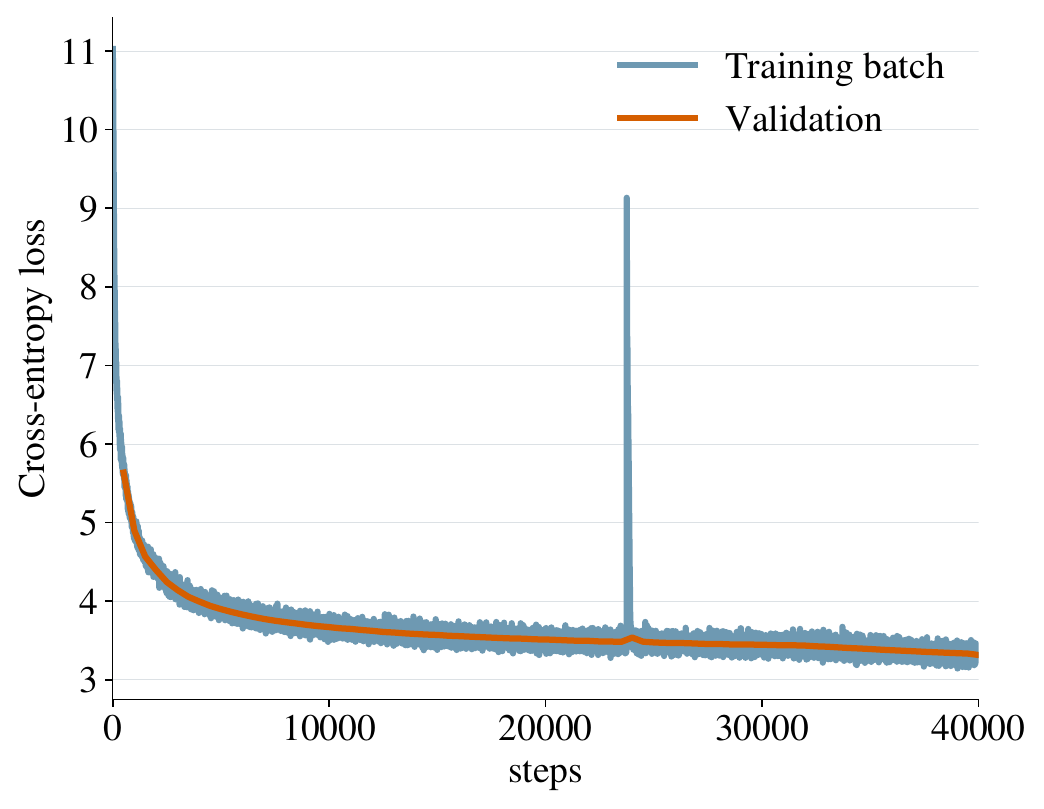}}\hfill
\subfloat[$\eta_{\rm peak}=0.12$\label{fig:llm-lr-loss-appendix-panel-d}]{\includegraphics[width=.24\linewidth]{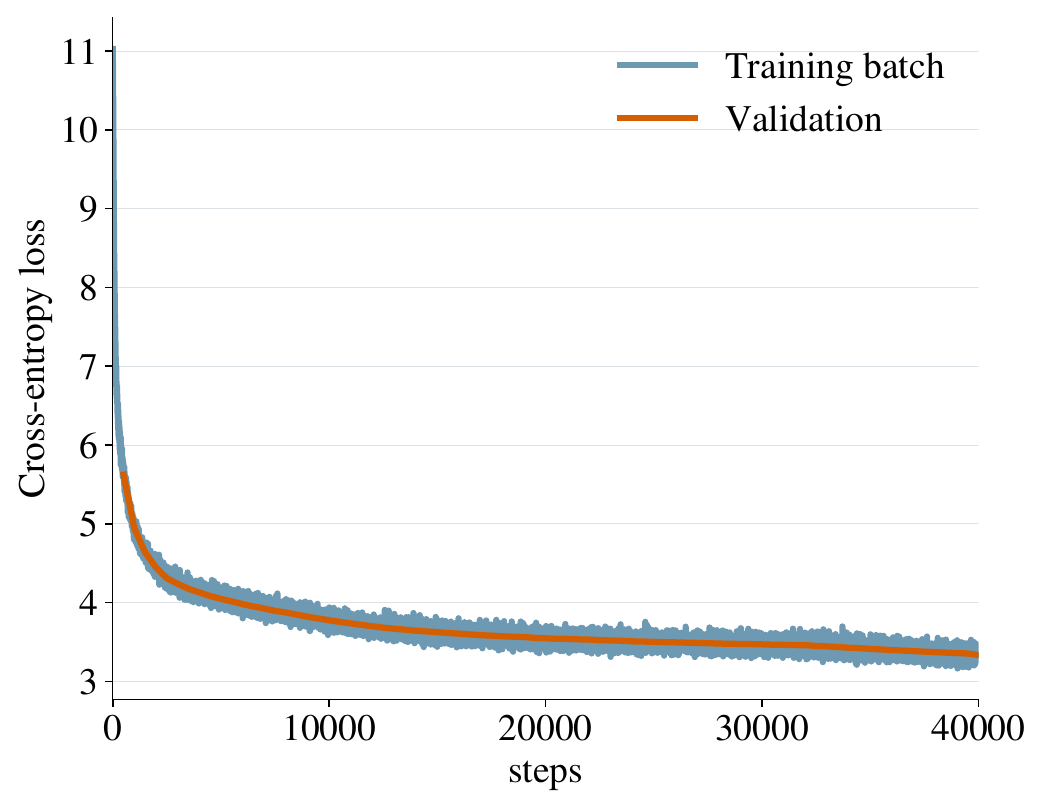}}
\caption{\textbf{Training and validation loss across learning rates.} This experiment uses the 130M Llama-like LLM with no-momentum NS-5 Muon at four peak learning rates; it measures training and validation loss. (a)--(d) Curves for peak rates $0.02$, $0.04$, $0.08$, and $0.12$.}
\label{fig:llm-lr-loss-appendix}
\end{figure}

\begin{figure}[!t]
\centering
\subfloat[{32 reference sequences}]{{\includegraphics[width=.32\linewidth]{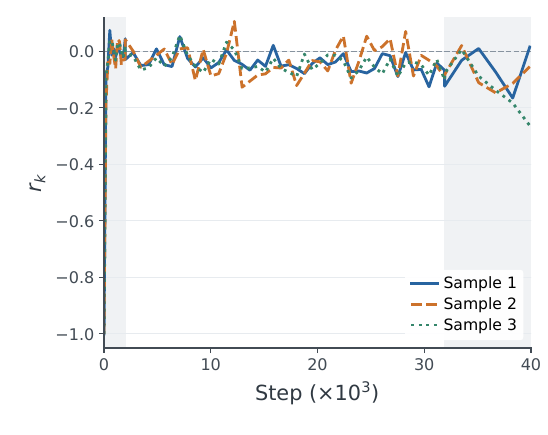}}}\hfill
\subfloat[{64 reference sequences}]{{\includegraphics[width=.32\linewidth]{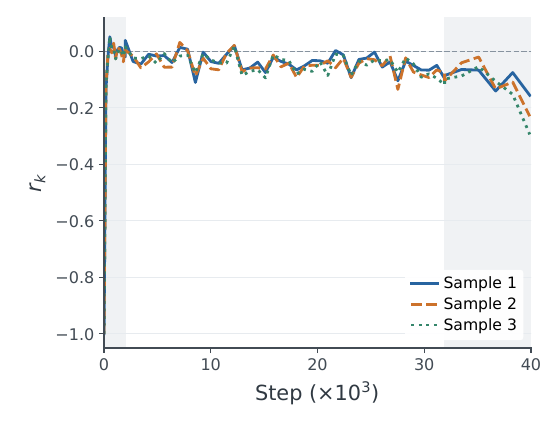}}}\hfill
\subfloat[{128 reference sequences}]{{\includegraphics[width=.32\linewidth]{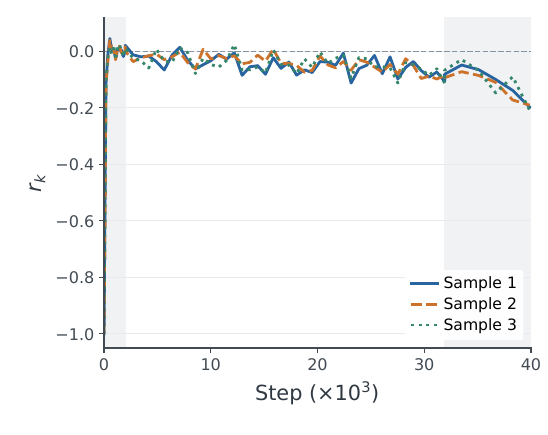}}}
\caption{{\textbf{Reference-set sensitivity at peak rate $0.02$.} Muon-only relative discrepancy $\hat r$ for three independently sampled reference sets per size, using 32 draws at each of 60 checkpoints. Zero denotes conditional loss balance.}}
\label{fig:130m-reference-sets}
\end{figure}

\begin{figure}[!t]
\centering
\subfloat[{Muon-only and complete updates}]{{\includegraphics[width=.485\linewidth]{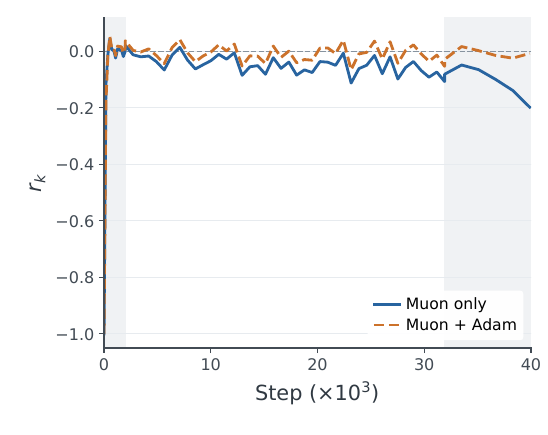}}}\hfill
\subfloat[{Complete-update loss balance}]{{\includegraphics[width=.485\linewidth]{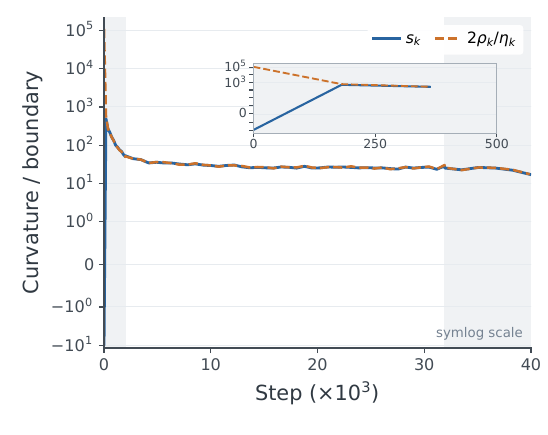}}}
\caption{{\textbf{Conditional complete-update checks at peak rate $0.02$.} (a) Relative discrepancy for Muon-only and complete Muon-plus-AdamW updates. (b) Complete-update effective curvature and loss boundary.}}
\label{fig:130m-reference-full}
\end{figure}

\begin{figure}[!t]
\centering
\subfloat[{Relative discrepancy}]{{\includegraphics[width=.485\linewidth]{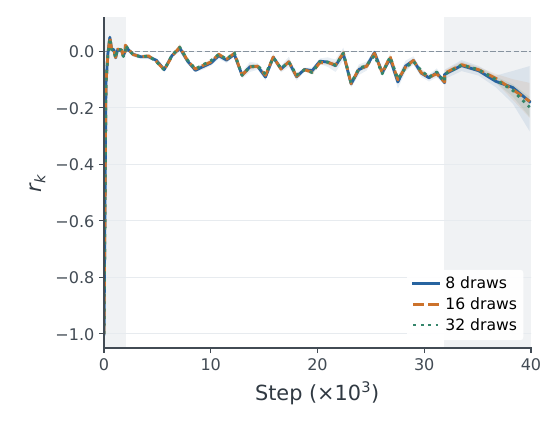}}}\hfill
\subfloat[{95\% interval width}]{{\includegraphics[width=.485\linewidth]{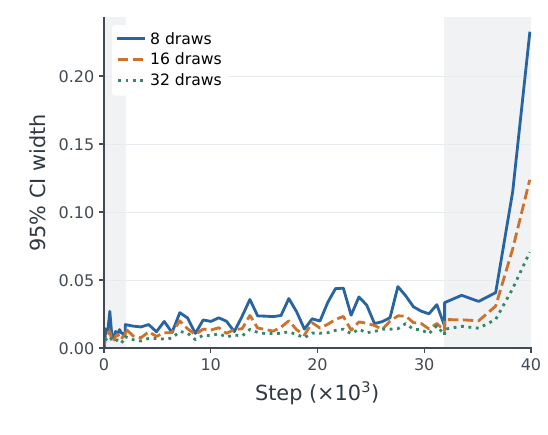}}}
\caption{{\textbf{Sensitivity to Monte Carlo draw count.} Muon-only estimates from the first 8, 16, and all 32 draws of the same samples. Intervals quantify pointwise conditional Monte Carlo uncertainty, excluding reference-set and training-seed uncertainty.}}
\label{fig:130m-reference-draws}
\end{figure}

\subsection{{130M Muon with momentum}}
\label{app:momentum}
{We additionally train the 130M model with NS-5 Muon at momentum $0.95$ and peak learning rate $0.02$, using batch size 32, sequence length 4,096, and a WSD schedule over 40 tokens per parameter. Conditional Muon-only probes use 128 reference sequences and 32 draws per checkpoint. Each draw starts from the same parameters and optimizer state, updates only the Muon matrices, and holds auxiliary parameters fixed. The estimators in Appendix~\ref{app:llm-estimators} use the implemented momentum-dependent directions; update alignment is measured between consecutive actual updates on the same Muon matrices.}

\begin{figure}[!t]
\centering
\subfloat[$\eta_{\rm peak}=0.02$\label{fig:llm-batch-panel-a}]{\includegraphics[width=.24\linewidth]{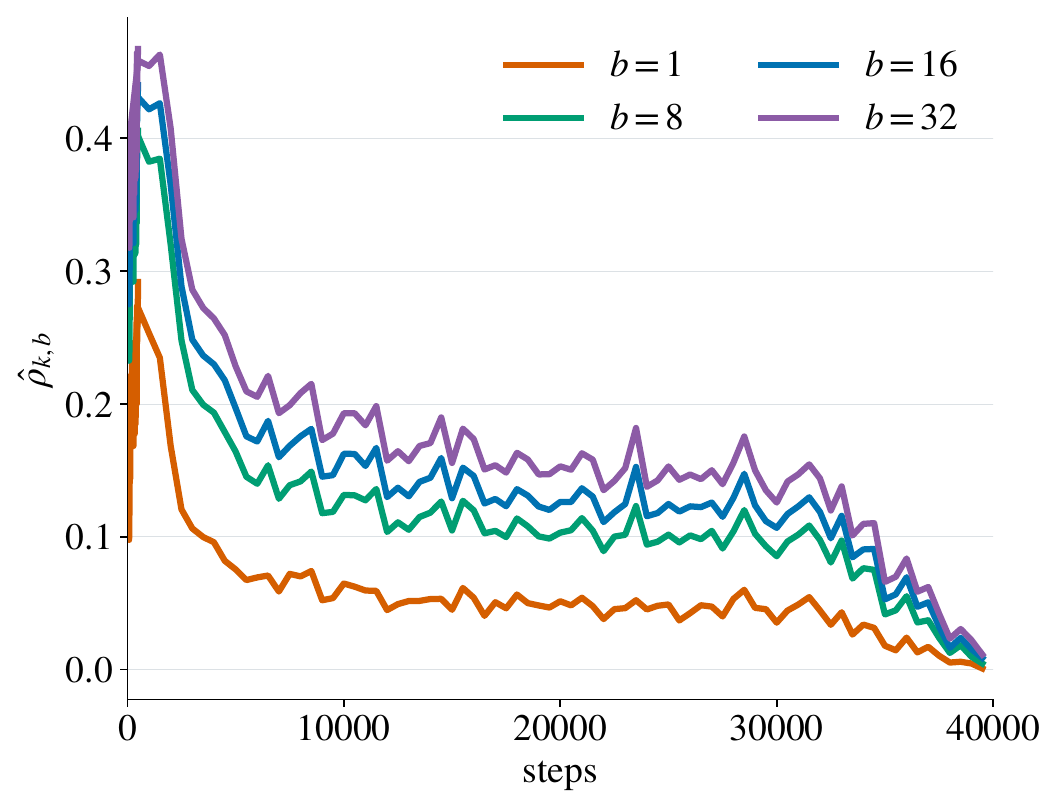}}\hfill
\subfloat[$\eta_{\rm peak}=0.04$\label{fig:llm-batch-panel-b}]{\includegraphics[width=.24\linewidth]{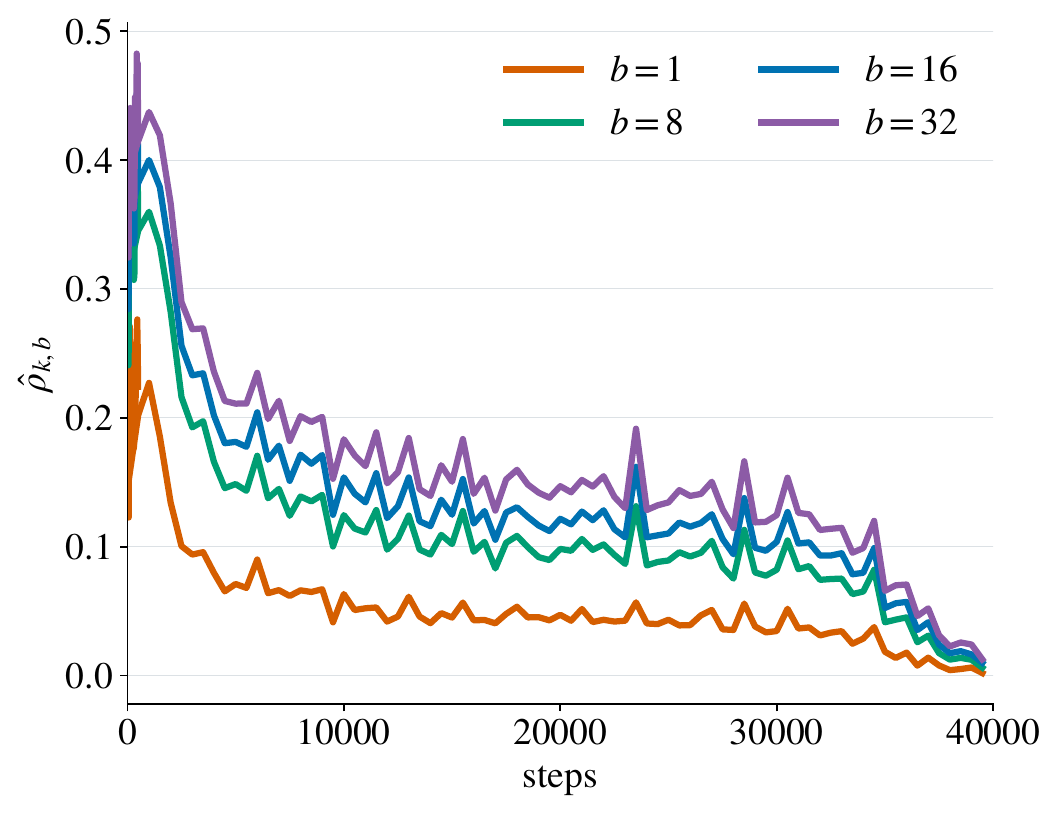}}
\hfill
\subfloat[$\eta_{\rm peak}=0.08$\label{fig:llm-batch-panel-c}]{\includegraphics[width=.24\linewidth]{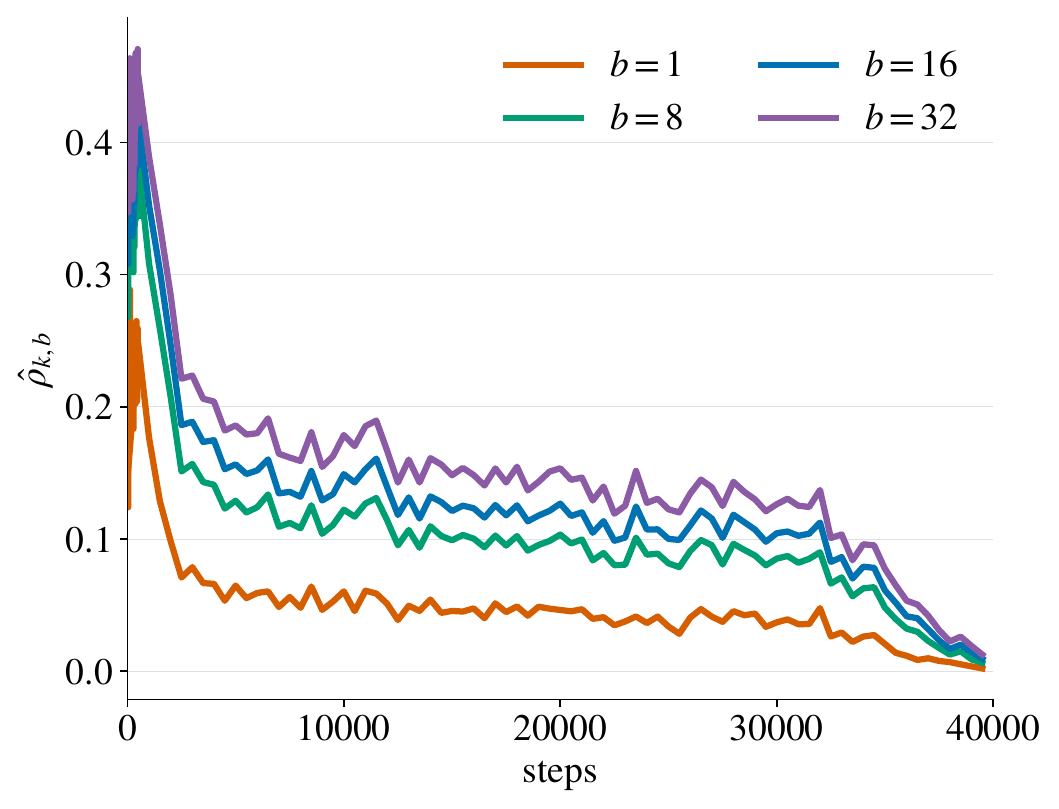}}\hfill
\subfloat[$\eta_{\rm peak}=0.12$\label{fig:llm-batch-panel-d}]{\includegraphics[width=.24\linewidth]{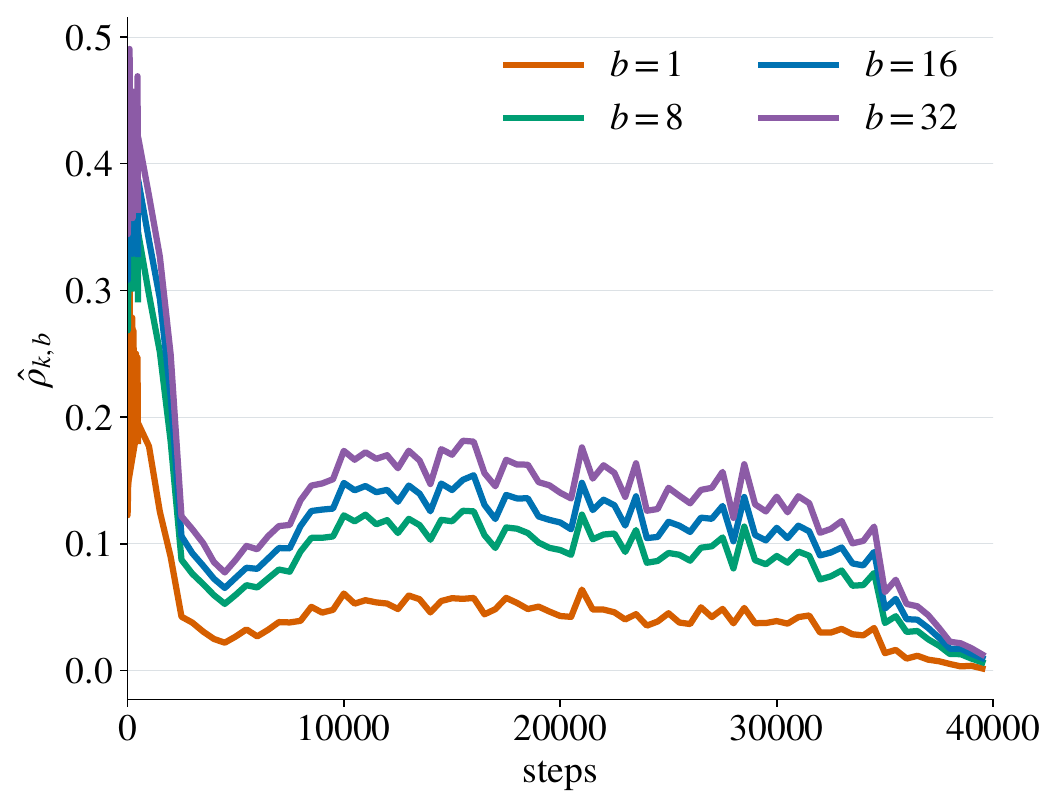}}
\caption{\textbf{Probe-batch effects on T1 quantities.} This experiment uses the 130M Llama-like LLM with training batch size 32; it measures reference-gradient coherence for probe batches $b=1,8,16,32$ at fixed checkpoints. (a)--(d) Results for peak Muon learning rates $0.02$, $0.04$, $0.08$, and $0.12$, respectively.}
\label{fig:llm-batch}
\end{figure}

\subsection{130M layer structure and batch-dependent coherence}
\label{app:llm-structure}
Figures~\ref{fig:llm-layer} and~\ref{fig:llm-batch} support the layer and probe-batch analyses in Section~\ref{subsec:llm-alignment}. They report matrixwise heterogeneity and fixed-checkpoint probe-batch effects, respectively; the latter do not vary the training batch.

\begin{figure}[!t]
\centering
\subfloat[Block trajectories\label{fig:llm-layer-panel-a}]{\includegraphics[width=.485\linewidth]{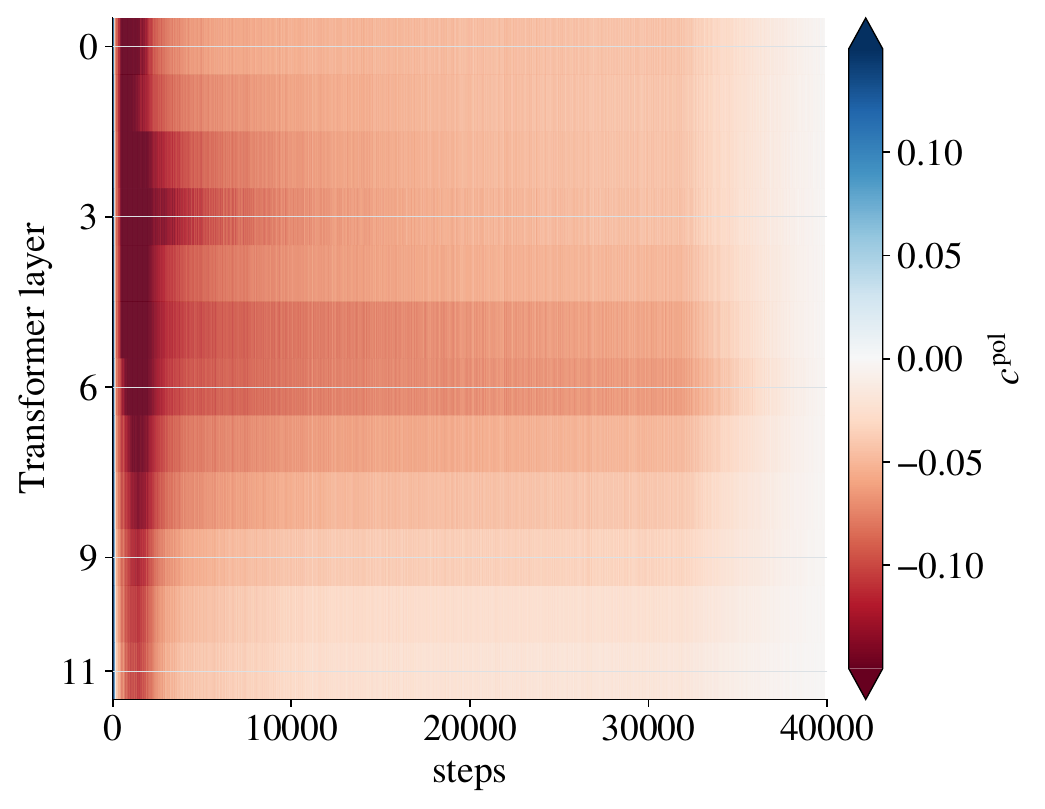}}\hfill
\subfloat[Global cosine\label{fig:llm-layer-panel-b}]{\includegraphics[width=.485\linewidth]{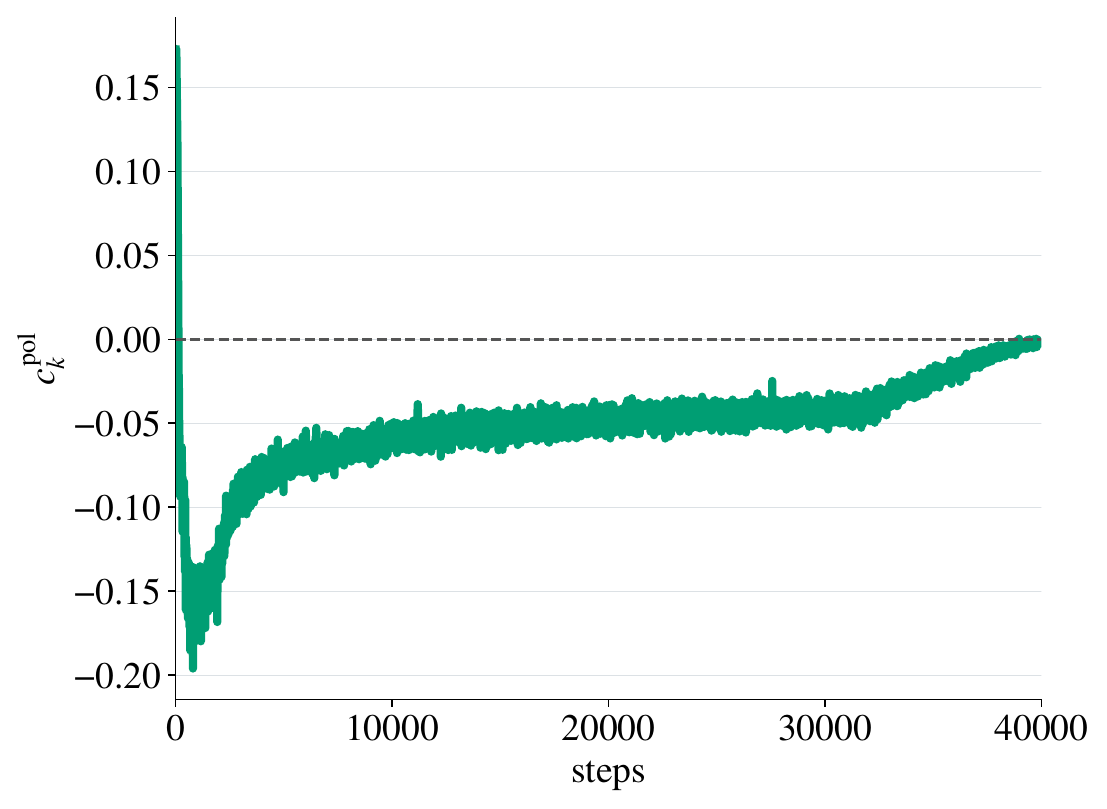}}
\caption{\textbf{Layer and module variation in {update alignment}.} This experiment uses the 130M Llama-like LLM at peak Muon learning rate $0.02$. (a) shows block trajectories; (b) shows the global cosine. \cref{fig:llm-module-medians} shows module medians.}
\label{fig:llm-layer}
\end{figure}

\subsection{1B LLM pretraining}
\label{app:1b-protocol}
The 1B Llama-like model has 16 layers, model width 2,048, 16 attention heads, and FFN width 6,400. Table~\ref{tab:1b-config} lists its architecture, optimization settings, token budget, and recorded diagnostics. {Each of the three runs processes approximately 40 tokens per parameter.}

\begin{table}[!t]
\centering
\small
\caption{1B pretraining configuration and recorded diagnostics.}
\label{tab:1b-config}
\begin{tabularx}{\linewidth}{@{}lX@{}}
\toprule
Item & Setting\\
\midrule
Model & Llama-like; approximately 1B parameters; 16 layers\\
Model width / heads & 2,048 / 16; head dimension 128\\
FFN & SwiGLU; intermediate width 6,400\\
Normalization / positions & RMSNorm / RoPE\\
Vocabulary / output head & 50,304; tied input and output embeddings\\
Data & GPT-2-tokenized FineWeb \texttt{sample-100BT}\\
Prepared token pools & 50,000,000,000 training tokens and 16,777,216 validation tokens from disjoint documents\\
Training sampling & Random contiguous windows sampled with replacement\\
Training batch / microbatch & {Global batches 128 / 256 / 512; training microbatches 16 / 32 / 64 sequences per device}\\
Sequence length & 4,096 tokens\\
Tokens per update & {524,288 / 1,048,576 / 2,097,152}\\
Training updates & {76,346 / 38,173 / 19,087}\\
Total training tokens & {40,027,291,648 / 40,027,291,648 / 40,028,340,224}\\
Token budget & {Approximately 40 tokens per parameter}\\
\midrule
Muon parameter scope & Attention and FFN weight matrices\\
Muon & NS-5; momentum 0; weight decay 0; update scale 1\\
Muon peak LR & 0.04\\
NS-5 settings & Five Newton--Schulz iterations; coefficients
$(3.4445,-4.7750,2.0315)$; Frobenius normalization with
$\epsilon=10^{-7}$; float32 arithmetic\\
Auxiliary parameter scope & Token embeddings and normalization parameters\\
Auxiliary optimizer & AdamW; peak LR $0.001$; betas $(0.8,0.999)$;
$\epsilon=10^{-8}$; weight decay $0.1$\\
WSD schedule & Linear warmup fraction $0.05$; constant stage;
cosine-decay fraction $0.20$; final LR fraction $0.10$;
the same schedule multiplier is applied to Muon and auxiliary AdamW\\
Training precision & bfloat16 autocast with float32 parameters\\
Memory settings & Activation checkpointing; loss chunks of 128 tokens\\
Gradient clipping & None\\
Training seed & 1337\\
Configured data seed & 2026\\
\midrule
Evaluation protocol & {Every 500 updates, 256 sequences; final evaluation 4096 sequences}\\
Recorded training / validation losses & {76,346 / 153; 38,173 / 77; 19,087 / 39 values}\\
Recorded global direction cosines & {15,269 / 7,635 / 3,818 samples; every 5 updates}\\
Direction scope & Global and layer-level cosines of consecutive Muon directions\\
Conditional T1 probes & {Reference set: 32 sequences; 8 draws per checkpoint; probe batch equals training batch; FP32 probes, microbatch 32}\\
{T1 cadence / checkpoints} & {Every 5 updates in the first 5\%, then every 100; 1,489 / 744 / 372 checkpoints}\\
\bottomrule
\end{tabularx}
\end{table}

{Figure~\ref{fig:llm-1b} compares all three batches; Figures~\ref{fig:1b-b128} and~\ref{fig:1b-b256} show the individual batch-128 and batch-256 runs. }

\begin{figure}[!t]
\centering
\subfloat[{Loss}\label{fig:1b-b128-loss}]{\includegraphics[width=.315\linewidth]{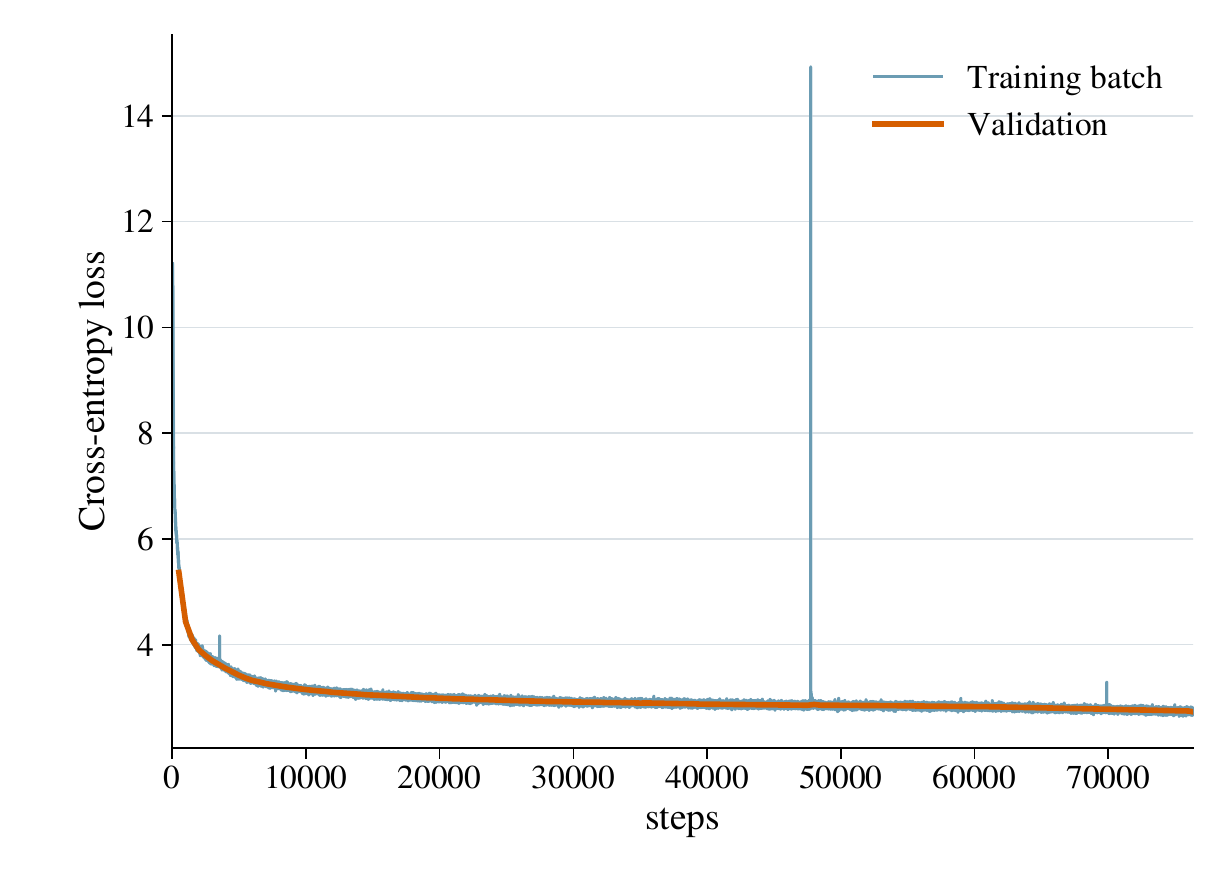}}
\hfill
\subfloat[{Loss balance}\label{fig:1b-b128-balance}]{\includegraphics[width=.315\linewidth]{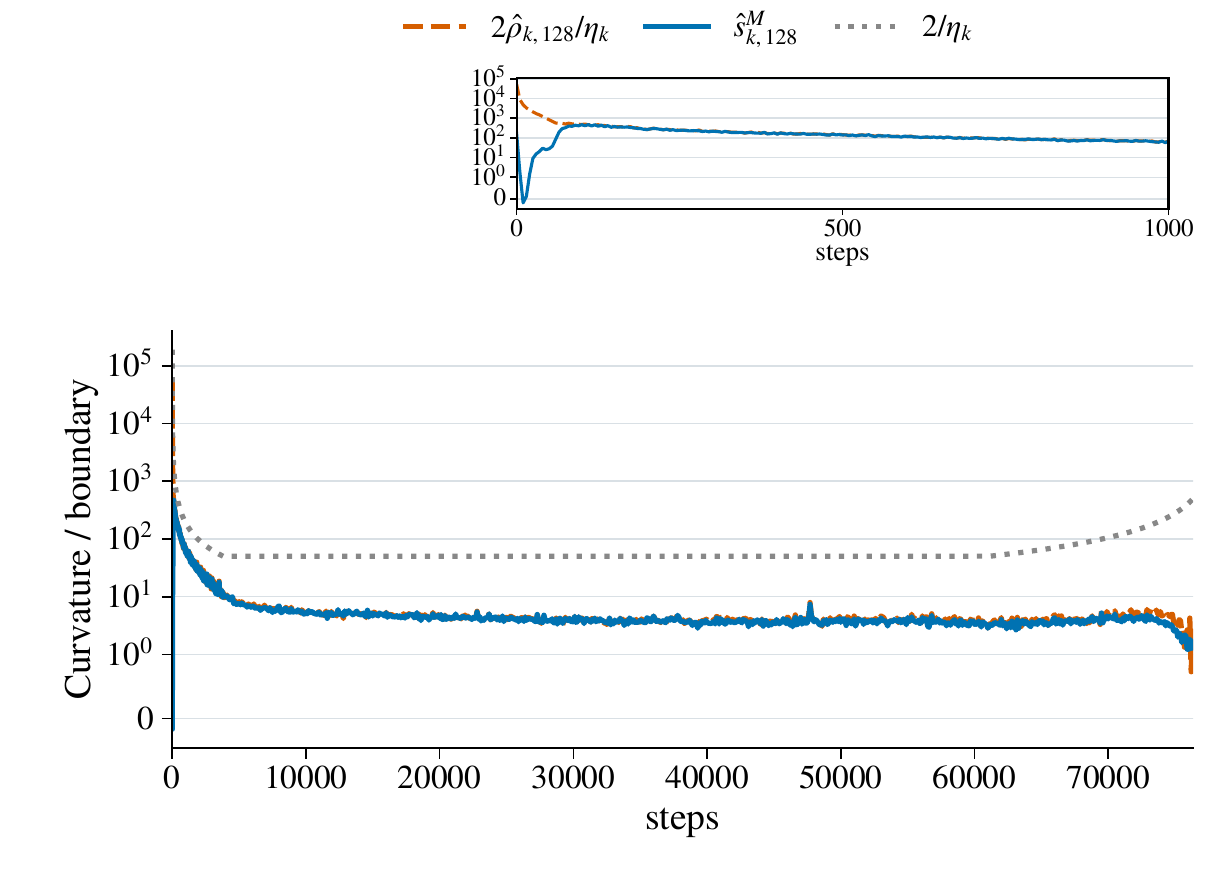}}
\hfill
\subfloat[{Global $c^{\mathrm{pol}}$}\label{fig:1b-b128-global}]{\includegraphics[width=.315\linewidth]{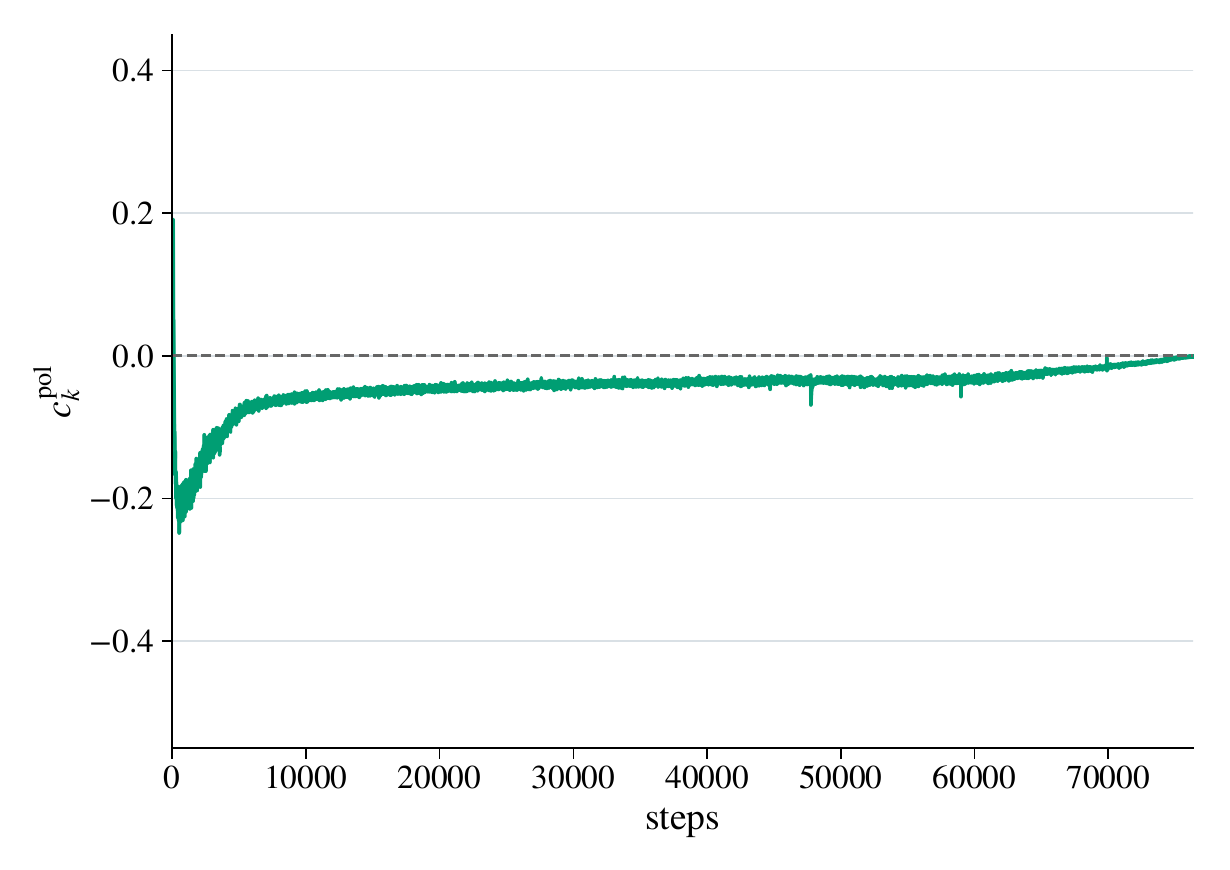}}
\caption{{\textbf{Loss, conditional loss balance, and update alignment in 1B pretraining ($b=128$).} Each run uses approximately 40 tokens per parameter. (a) Training and validation loss. (b) Conditional effective curvature $\hat s^M_{k,b}$ and boundary $2\hat\rho_{k,b}/\eta_k$; the grey dotted line is $2/\eta_k$. The separate upper inset enlarges steps 0--1,000 without covering the full-run curves. The curvature axes use a symmetric-log scale. (c) Global cosine of consecutive Muon directions. T1 uses a fixed reference set of 32 sequences and eight draws per checkpoint.}}
\label{fig:1b-b128}
\end{figure}

\begin{figure}[!t]
\centering
\subfloat[{Loss}\label{fig:1b-b256-loss}]{\includegraphics[width=.315\linewidth]{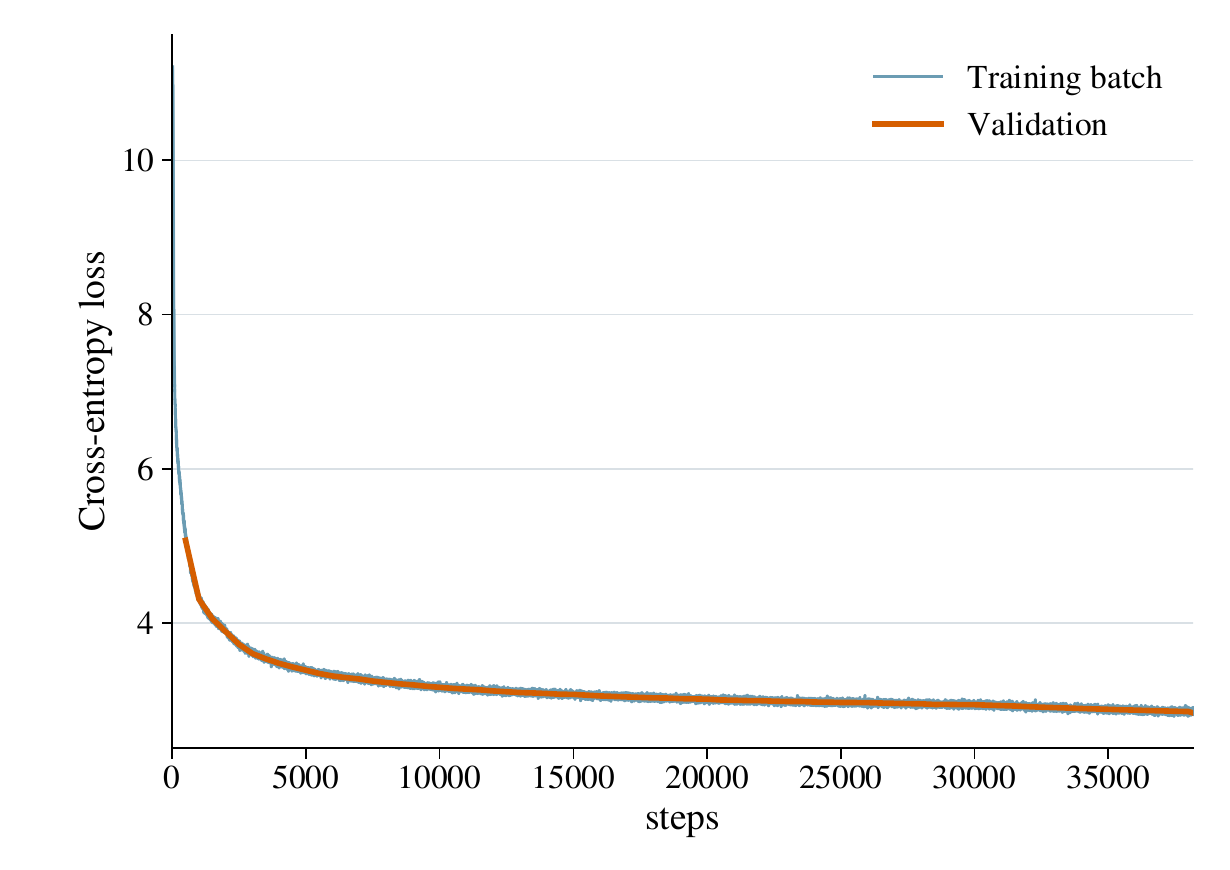}}
\hfill
\subfloat[{Loss balance}\label{fig:1b-b256-balance}]{\includegraphics[width=.315\linewidth]{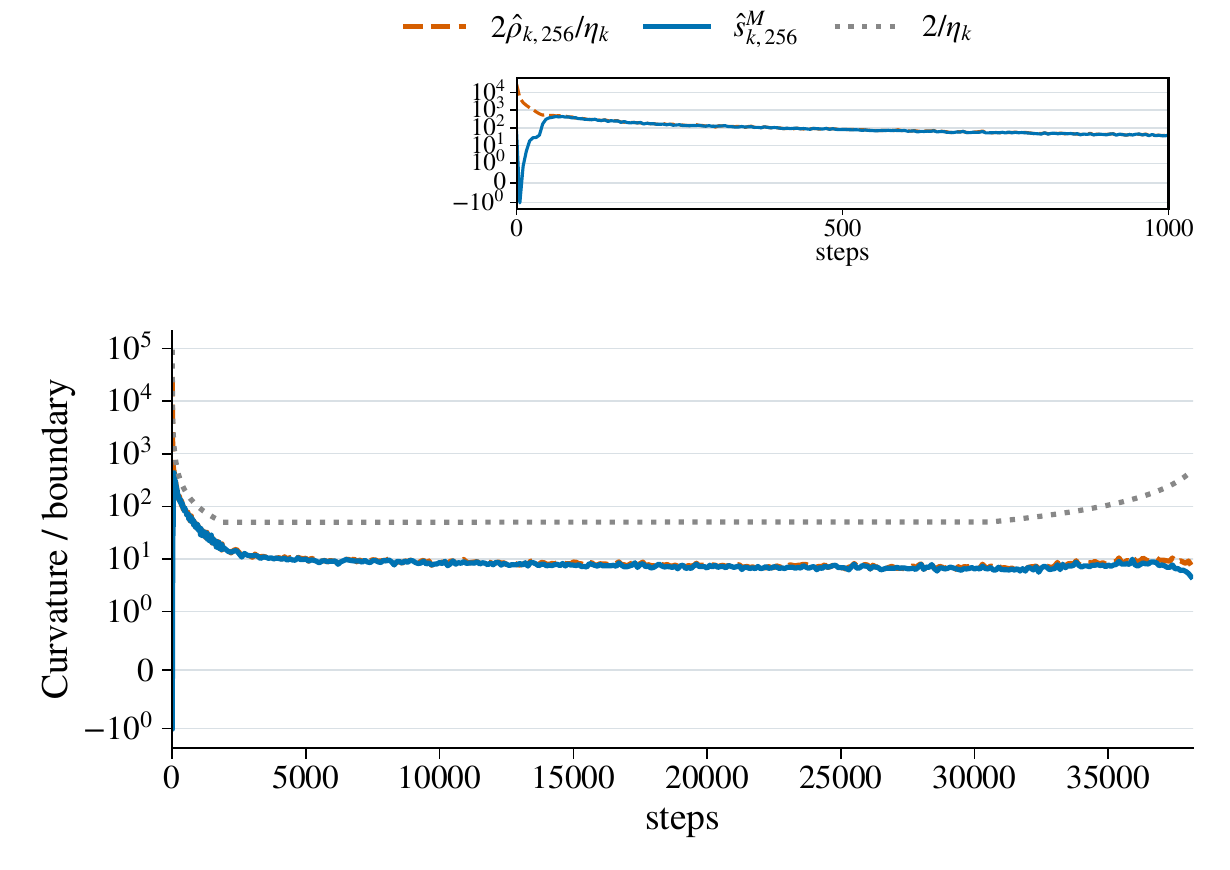}}
\hfill
\subfloat[{Global $c^{\mathrm{pol}}$}\label{fig:1b-b256-global}]{\includegraphics[width=.315\linewidth]{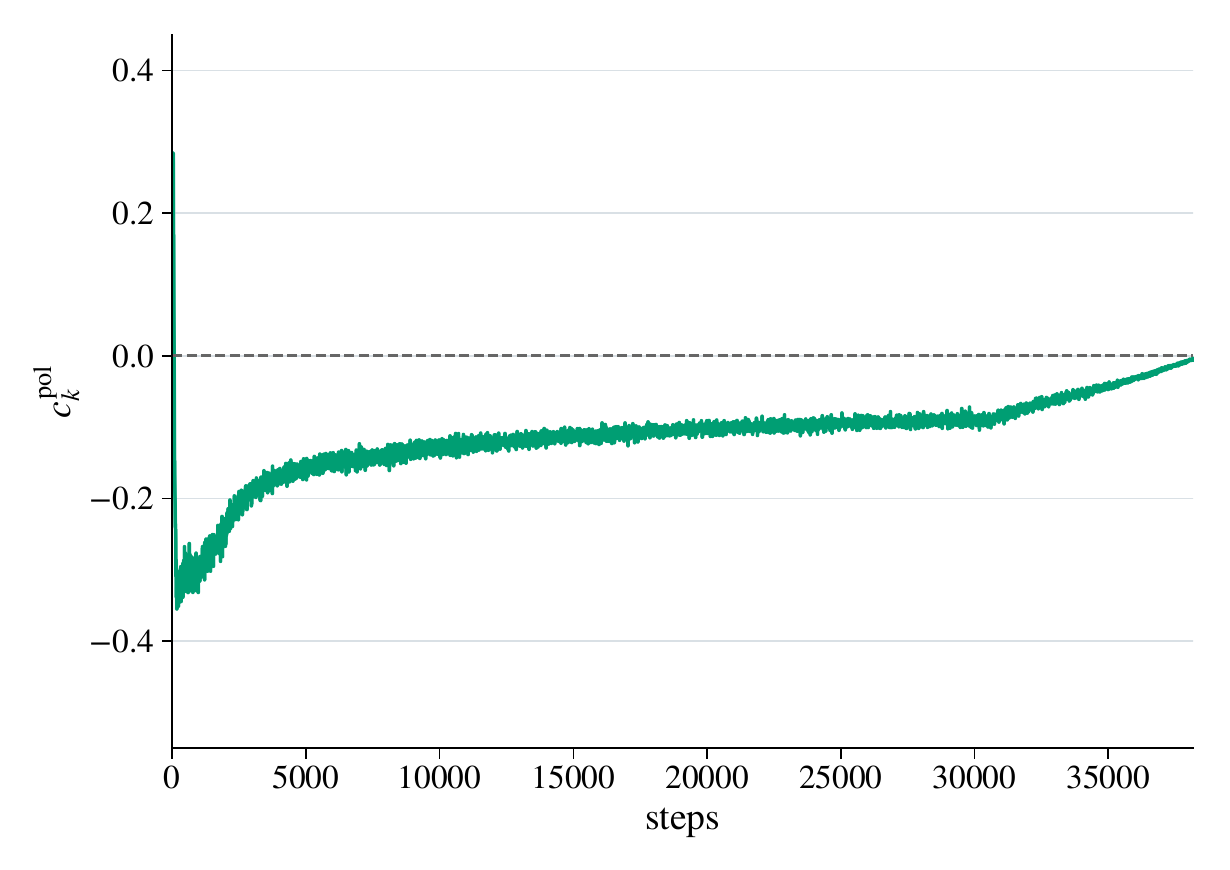}}
\caption{{\textbf{Loss, conditional loss balance, and update alignment in 1B pretraining ($b=256$).} Each run uses approximately 40 tokens per parameter. (a) Training and validation loss. (b) Conditional effective curvature $\hat s^M_{k,b}$ and boundary $2\hat\rho_{k,b}/\eta_k$; the grey dotted line is $2/\eta_k$. The separate upper inset enlarges steps 0--1,000 without covering the full-run curves. The curvature axes use a symmetric-log scale. (c) Global cosine of consecutive Muon directions. T1 uses a fixed reference set of 32 sequences and eight draws per checkpoint.}}
\label{fig:1b-b256}
\end{figure}

{
\section{SGD comparison and spectral weighting}
\label{app:sgd-comparison}
SGD on 130M model also exhibit proximity to the {loss-balance boundary (T1)} and negative alignment of consecutive updates. Their update alignment is less stable than in the displayed Muon runs. We first report these observations, then examine how singular-value weighting affects the cosine.


\subsection{GD beyond the scalar quadratic}
\label{app:gd-general-signatures}
For full-batch GD with fixed $\eta>0$, let $d_k=-\eta\nabla L(w_k)\ne0$ and $L\in C^2$.
The three Hessian quantities used in Table~\ref{tab:gd-general-signatures} are
\begin{align*}
\bar H_k&=\int_0^1\nabla^2L(w_k+t d_k)\,\mathrm dt,
 H_k=2\int_0^1(1-t)\nabla^2L(w_k+t d_k)\,\mathrm dt, H_*=\nabla^2L(w).
\end{align*}
\begin{table}[!ht]
\centering
\caption{Loss balance and update dynamics for multidimensional full-batch GD. The first two conditions concern the current step; the third concerns the linearized map at a stationary point.}
\label{tab:gd-general-signatures}
\small
\setlength{\tabcolsep}{5pt}
\renewcommand{\arraystretch}{1.3}
\begin{tabularx}{\linewidth}{@{}>{\raggedright\arraybackslash}p{.45\linewidth}>{\centering\arraybackslash}X@{}}
\toprule
Signature & GD condition \\
\midrule
Loss balance & $d_k^\top\widetilde H_kd_k/\|d_k\|^2=2/\eta$ \\
Equal-magnitude reversal / exact two-step return & $\bar H_kd_k=(2/\eta)d_k$ \\
Modal marginal linear stability & $\eta\lambda_i(H_*)=2$ for a positive-curvature mode \\
\bottomrule
\end{tabularx}
\end{table}
Indeed, $d_{k+1}=(I-\eta\bar H_k)d_k$, the exact loss change is
$-\|d_k\|^2/\eta+d_k^\top\widetilde H_kd_k/2$, and the linearized GD map at $w_*$ is $I-\eta H_*$.
For $L(w)=w^\top Hw/2$ with $H\succ0$, all three Hessians equal $H$, but loss balance fixes a Rayleigh quotient, exact reversal requires an eigendirection, and the stability boundary of the full map is $\eta\lambda_{\max}(H)=2$.
Thus their scalar coincidence does not imply a general equivalence in multiple dimensions.

\subsection{SGD in 130M pretraining}
\label{app:sgd-130m}
We use the 130M architecture and data protocol in Appendix~\ref{app:130m-protocol}, replacing Muon on the attention and FFN matrices with no-momentum SGD while retaining auxiliary AdamW. The two runs use peak SGD learning rates $0.01$ and $0.03$, batch 32, sequence length 4,096, and the WSD schedule over 40 tokens per parameter (39,876 updates). Conditional probes use 128 reference sequences and 32 draws at each of 97 checkpoints. They update only the SGD matrices, holding auxiliary parameters fixed, and use the reference normalization of Appendix~\ref{app:llm-estimators} with SGD directions. 

For $g=\nabla F(\theta_k)$, minibatch gradient $g_B$, and $N=\|g\|^2$, frozen-checkpoint probes give
\begin{align}
 A_{k,B}&=\eta\langle g,g_B\rangle,
 &\Delta_{k,B}&=F(\theta_k-\eta g_B)-F(\theta_k),\notag\\
 \hat\rho_{k,b}&=\frac{\overline A_{k,b}}{\eta N},
 &\hat s^{\mathrm{SGD}}_{k,b}&=\frac{2(\overline\Delta_{k,b}+\overline A_{k,b})}{\eta^2N}.
 \label{eq:sgd-probe-definitions}
\end{align}
Thus
\begin{equation}
 \overline\Delta_{k,b}
 =\frac{\eta^2N}{2}\left(\hat s^{\mathrm{SGD}}_{k,b}
 -\frac{2\hat\rho_{k,b}}{\eta}\right).
 \label{eq:sgd-loss-identity}
\end{equation}

{Update alignment} is
\begin{equation}
 c_k^{\mathrm{SGD}}=
 \frac{\langle g_{B_k}(\theta_k),g_{B_{k+1}}(\theta_{k+1})\rangle}
 {\|g_{B_k}(\theta_k)\|\,\|g_{B_{k+1}}(\theta_{k+1})\|}.
 \label{eq:sgd-temporal-cosine}
\end{equation}

\begin{figure}[!t]
\centering
\subfloat[{Loss, $\eta_{\max}=0.01$}]{\includegraphics[width=.32\linewidth]{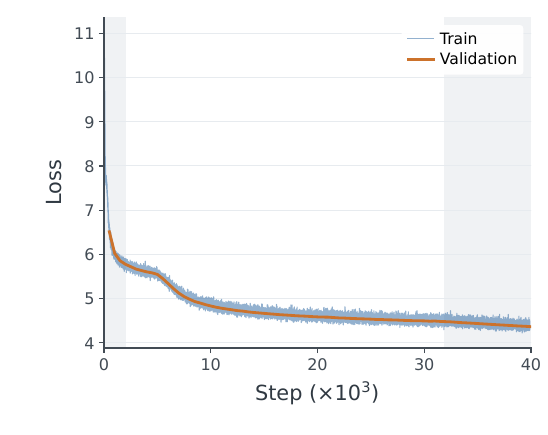}}\hfill
\subfloat[{Conditional loss balance}]{\includegraphics[width=.32\linewidth]{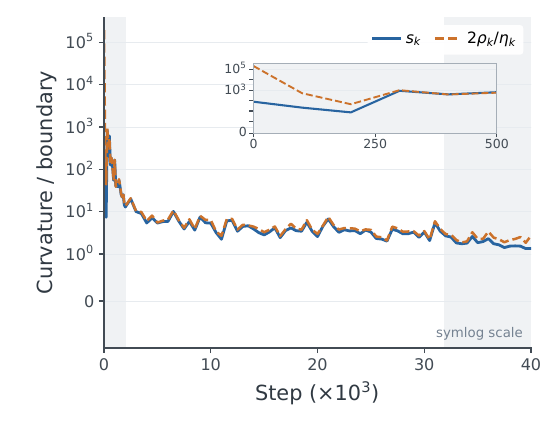}}\hfill
\subfloat[{{Update alignment}}]{\includegraphics[width=.32\linewidth]{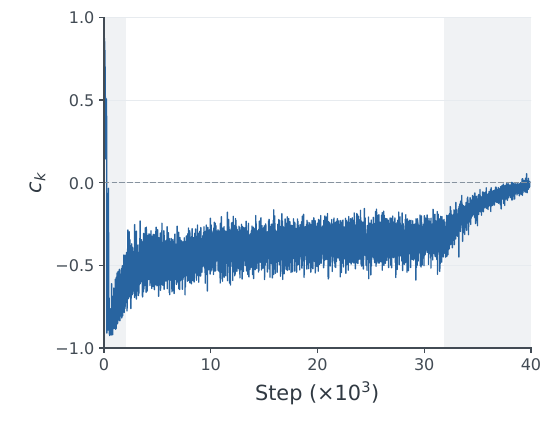}}\\[.6em]
\subfloat[{Loss, $\eta_{\max}=0.03$}]{\includegraphics[width=.32\linewidth]{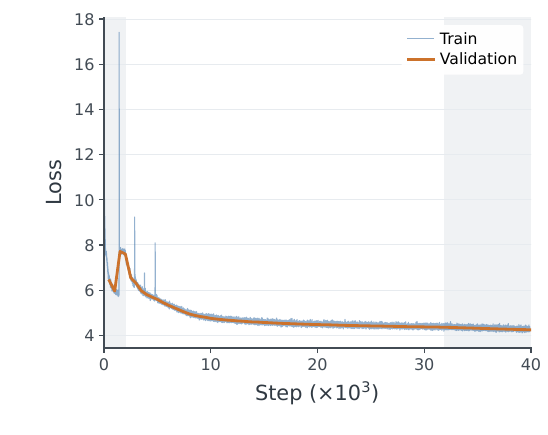}}\hfill
\subfloat[{Conditional loss balance}]{\includegraphics[width=.32\linewidth]{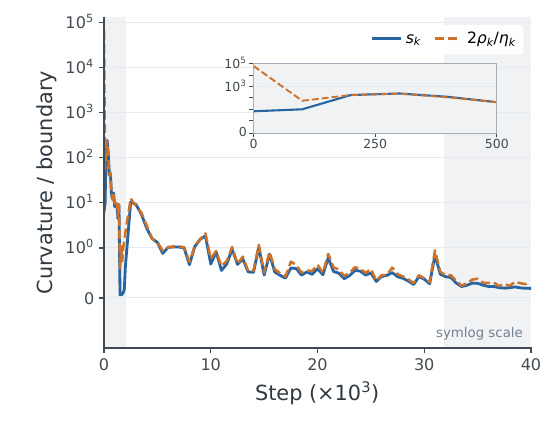}}\hfill
\subfloat[{{Update alignment}}]{\includegraphics[width=.32\linewidth]{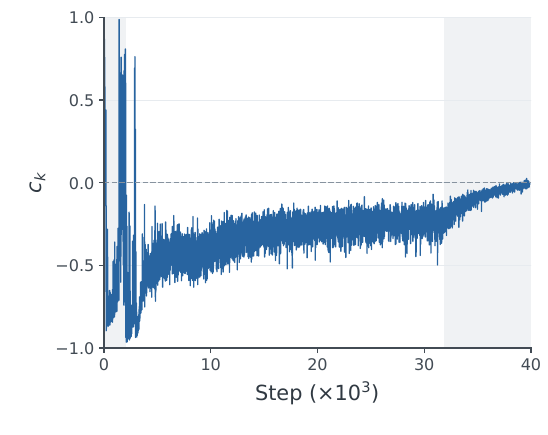}}
\caption{{\textbf{SGD in 130M pretraining.} Rows use peak learning rates $0.01$ and $0.03$. Columns show loss, conditional effective curvature and the {loss-balance boundary} $2\hat\rho_{k,b}/\eta_k$, and consecutive-update cosine. Loss-balance probes and cosines concern the SGD matrix blocks; auxiliary parameters are trained with AdamW.}}
\label{fig:130m-sgd}
\end{figure}

{Figure~\ref{fig:130m-sgd} shows the SGD comparison to LLM pretraining. After the initial transient, effective curvature approaches and fluctuates near, generally below, the {loss-balance boundary} during the main training phase, while consecutive updates are negatively aligned. Compared with the Muon runs in \cref{fig:introduction_muon_130M:c,fig:llm-lr-cosine-004,fig:llm-lr-cosine-008,fig:llm-lr-cosine-012}, SGD shows stronger negative alignment and visibly less stable update alignment. The $0.03$ run also exhibits pronounced early loss spikes and abrupt alignment changes. Thus boundary proximity and negative alignment are shared observations, while their magnitude and fluctuations depend on the optimizer and training configuration.}

\subsection{How singular values enter the cosine}
\label{app:sgd-spectral-cosine}
To interpret the alignment differences, consider the same pair of nonzero gradient matrices with compact SVDs
\[
 G=\sum_{i=1}^{r}\sigma_i u_i v_i^\top,\qquad
 G'=\sum_{j=1}^{r'}\sigma'_j u'_j{v'_j}^\top,\qquad
 M_{ij}=(u_i^\top u'_j)(v_i^\top v'_j).
\]
Expanding Frobenius inner products gives
\begin{align}
 c^{\mathrm{grad}}(G,G')
 &=\frac{\sum_{i,j}\sigma_i\sigma'_jM_{ij}}
 {\sqrt{\sum_i\sigma_i^2}\sqrt{\sum_j{\sigma'_j}^2}},
 \label{eq:sgd-singular-weighted}\\
 c^{\mathrm{pol}}(G,G')
 &=\frac{\sum_{i,j}M_{ij}}{\sqrt{rr'}}.
 \label{eq:sgd-singular-unweighted}
\end{align}
Gradient cosine weights mode pairs by singular-value products; exact polar cosine removes these amplitude weights. For example, $G=\operatorname{diag}(M,\epsilon)$ and $G'=\operatorname{diag}(-M,\epsilon)$ give $c^{\mathrm{grad}}\to-1$ as $M/\epsilon\to\infty$, but $c^{\mathrm{pol}}=0$. Strong negative alignment of large-amplitude modes can therefore dominate gradient cosine without dominating polar cosine.

The identities explain a possible source of alignment differences.

\subsection{GD on the same matrix quadratic as Muon}
\label{app:sgd-quadratic}
For the quadratic in \cref{sec:noncommuting-quadratic},
\[
L(\bW)=\tfrac12\|\bW\bm X-\bm Y\|_F^2,
\qquad
\bm A=\bm X\bm X^\top\succ0,
\qquad
\bG=(\bW\bm X-\bm Y)\bm X^\top,
\]
full-batch GD gives
\begin{equation}
\bG^+=\bG(\bm I-\eta\bm A).
\label{eq:gd-matrix-recursion}
\end{equation}
Its exact loss change and effective curvature are
\begin{align}
L(\bW-\eta\bG)-L(\bW)
&=-\eta\|\bG\|_F^2
+\tfrac{\eta^2}{2}\operatorname{tr}(\bG^\top\bG\bm A),
\notag\\
s^{\mathrm{GD}}
&=\frac{\operatorname{tr}(\bG^\top\bG\bm A)}
{\|\bG\|_F^2},
\qquad \rho^{\mathrm{GD}}=1.
\label{eq:gd-matrix-balance}
\end{align}
Writing $\bm A q_i=\lambda_iq_i$ in an orthonormal eigenbasis
and $e_i=\|\bG q_i\|^2$, for nonzero $\bG,\bG^+$,
\begin{equation}
c^{\mathrm{GD}}
=\frac{\sum_i(1-\eta\lambda_i)e_i}
{\sqrt{\sum_i e_i}
 \sqrt{\sum_i(1-\eta\lambda_i)^2e_i}}.
\label{eq:gd-matrix-cosine}
\end{equation}
Since $s^{\mathrm{GD}}=\sum_i\lambda_ie_i/\sum_i e_i$,
negative alignment begins at $\eta s^{\mathrm{GD}}>1$,
whereas loss balance occurs at $\eta s^{\mathrm{GD}}=2$.

Compared to Muon in Appendix~\ref{app:matrix-quadratic},
GD alignment is weighted by gradient energies, whereas
Muon alignment depends on the inertia of $\bm B-\eta\bm A$.
When $\bm B$ and $\bm A$ commute, $e_i=\sigma_i^2$:
GD reverses mode $i$ when $\eta\lambda_i>1$, and Muon when
$\sigma_i<\eta\lambda_i$.
These formulas describe each optimizer's own successor gradient.
}

\end{document}